\documentclass{article}

\PassOptionsToPackage{numbers, compress}{natbib}
\usepackage[preprint]{neurips_2026}

\usepackage[utf8]{inputenc} 
\usepackage[T1]{fontenc}    
\usepackage{hyperref}       
\usepackage{url}            
\usepackage{booktabs}       
\usepackage{amsfonts}       
\usepackage{nicefrac}       
\usepackage{array}
\usepackage{microtype}      
\usepackage{graphicx}
\usepackage{subcaption}
\usepackage{multirow}
\usepackage{makecell}
\usepackage{pifont}
\usepackage{tcolorbox}
\usepackage{enumitem}
\usepackage{amsmath}
\usepackage{amssymb}
\usepackage{mathtools}
\usepackage{amsthm}
\usepackage{float}
\usepackage{algorithm}
\usepackage{algpseudocode}
\usepackage{wrapfig}

\usepackage[capitalize,noabbrev]{cleveref}
\crefname{appendix}{Appendix}{Appendices}
\Crefname{appendix}{Appendix}{Appendices}

\theoremstyle{plain}
\newtheorem{theorem}{Theorem}[section]
\newtheorem{proposition}[theorem]{Proposition}

\theoremstyle{definition}

\theoremstyle{remark}

\newcommand{\ourmodel}{{\textsc{RE-Tab}{}}}
\newcommand{\ourmetric}{{\textsc{TabROUGE}{}}}
\newcommand{\tick}{\ding{51}}
\newcommand{\cross}{\ding{55}}
\newcommand{\inlinehl}[2]{{\setlength{\fboxsep}{1pt}\colorbox{#1}{#2}}}
\usepackage{colortbl}
\usepackage[table]{xcolor}

\definecolor{Best}{RGB}{210,240,220}
\definecolor{Second}{RGB}{245,245,200}
\newcommand{\hilite}[1]{\cellcolor{Best}#1}
\newcommand{\hilight}[1]{\cellcolor{Second}#1}

\title{Enhancing Table Reasoning with Deterministic Table-State Rewards}

\author{%
  Tung Sum Thomas Kwok\thanks{Equal contribution.} \\
  UCLA\\
  Los Angeles, CA 90095 \\
  \texttt{tk1018@ucla.edu} \\
  \And
  Xinyu Wang\footnotemark[1] \\
  McGill University \\
  Montreal, Canada \\
  \And
  Hengzhi He\\
  UCLA\\
  Los Angeles, CA 90095 \\
  \And
  Xiaofeng Lin\\
  UCLA\\
  Los Angeles, CA 90095 \\
  \And
  Peng Lu\\
  University of Montreal\\
  Montreal, Canada\\
  \And
  Liheng Ma\\
  McGill University \\
  Montreal, Canada \\
  \And
  Chunhe Wang\\
  UCLA\\
  Los Angeles, CA 90095 \\
  \And
  Yingnian Wu\\
  UCLA\\
  Los Angeles, CA 90095 \\
  \And
  Lei Ding\\
  University of Manitoba \\
  Winnipeg, Canada\\
  \And
  Guang Cheng\\
  UCLA\\
  Los Angeles, CA 90095 \\
}

\author{\mdseries
Tung Sum Thomas Kwok$^{1*}$, 
Xinyu Wang$^{2*}$, 
Hengzhi He$^{1}$,
Xiaofeng Lin$^{1}$, \\
Peng Lu$^{3}$,
Liheng Ma$^{2}$,
Chunhe Wang$^{1}$,
Chun Ho Mak$^{4}$,\\
Yuyu Luo$^{5}$,
Yingnian Wu$^{1}$, 
Lei Ding$^{6}$,
Guang Cheng$^{1}$\\
$^{1}$University of California, Los Angeles, 
$^{2}$McGill University,\\
$^{3}$Université de Montréal,
$^{4}$University College London, \\
$^{5}$The Hong Kong University of Science and Technology (Guangzhou), \\
$^{6}$University of Manitoba,\\
\texttt{\href{mailto:tk1018@ucla.edu}{tk1018@ucla.edu},
\href{mailto:guangcheng@ucla.edu}{guangcheng@ucla.edu}}
}

\begin{document}

\maketitle

\begin{abstract}
    Large Language Models (LLMs) struggle with multi-step reasoning over structured tables. The primary reason is the lack of explicit supervision for intermediate reasoning states. Existing learned reward models or executor-based verifiers are either unscalable or rely on answer-checking environments unavailable for many tabular tasks. This leaves no signal that is scalable and grounded in the query. 
    To address this, we introduce \ourmetric{}, a training-free and deterministic state reward. By adapting the Longest Common Subsequence (LCS) metric from text summarization to evaluate tabular states, \ourmetric{} assesses the lexical coverage and structural integrity of intermediate tables against the query without requiring learned models or external executors. Built upon this metric, we propose \ourmodel{}, a plug-and-play, training-free framework. \ourmodel{} reframes table reasoning as deterministic control over intermediate states, utilizing \ourmetric{} for stepwise feedback and trajectory-level test-time scaling (TTS) signals. Across six backbones and three benchmarks, \ourmodel{} improves accuracy by an average of 26.7~pp over no-reward baselines. It also reduces TTS samples by up to 33\%. Preliminary GRPO experiments further indicate \ourmetric{}'s viability as a scalable post-training reward, increasing gains by 8.34~pp. We further analyze failure modes of \ourmetric{}, including paraphrase under-rewarding and echo-column hacking, and identify when structure-aware lexical rewards remain reliable. 
\end{abstract}

\section{Introduction}

\textbf{Table reasoning agents reason by manipulating tables.} Tool-using table reasoning agents answer queries through operations (filter, project, sort, aggregate, join) that reshape the input table across intermediate states~\cite{wang2024chainoftable,ji2025treeoftable,cheng2023binding,zhang-etal-2024-syntqa,yang2025cit,sui-etal-2024-tap4llm,chen-2023-large}. Each intermediate table is the agent's \emph{working state}, conditioning the next action and grounding downstream reasoning. Treating these states explicitly aligns table reasoning with the world-model view used in language and vision agents~\cite{xing2025critiquesworldmodels,wu2025rlvrworld}, making \emph{state quality} the central object of supervision.

\textbf{Intermediate state errors compound silently.} Without a state-quality signal, the agent cannot detect when a transformation drops query-relevant evidence, retains irrelevant content, or applies one operation too many. These errors accumulate, leaving the final table without the evidence needed to recover. \cref{fig:motivation}(a) shows a Chain-of-Table trajectory that retrieves the right rows but, lacking a stop signal, over-applies operations and outputs a wrong answer. \cref{fig:motivation}(b) shows a VLM agent failing on a visually corrupted state despite a correct reasoning procedure. Both stem from one root cause: no \emph{state-level} feedback ties the table to the query.

\begin{figure}[t]
    \centering
    \includegraphics[width=0.98\linewidth]{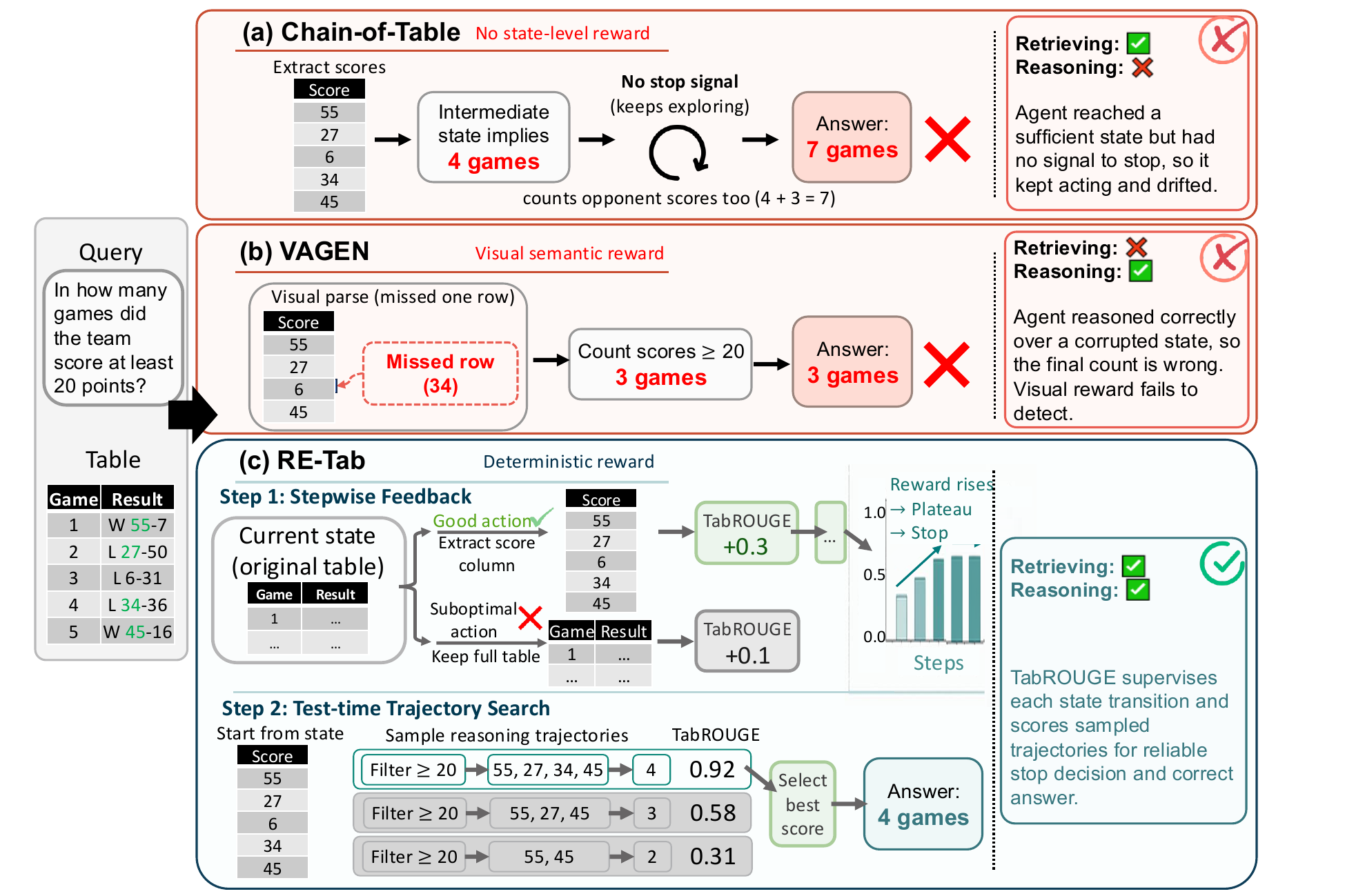}
    \caption{Without deterministic state-level reward, two prior agent paradigms drift in complementary ways: (a) Chain-of-Table retrieves correctly but, lacking a stop signal, over-applies operations and answers 7. (b) VAGEN reasons over a visually corrupted state and answers 3. (c) \ourmodel{} attaches \ourmetric{} to each state transition (Step 1) and to sampled trajectories (Step 2), producing the correct answer 4. Reward values shown are illustrative, with actual ranges reported in \cref{subsec:reward-validation}.}
    \label{fig:motivation}
    \vspace{-5mm}
\end{figure}

\textbf{Existing reward signals do not solve this cleanly.} Each existing reward family falls short on at least one axis. \emph{Final-answer rewards} are too sparse to guide intermediate transformations, leaving long-trajectory credit assignment open~\cite{liu-etal-2024-rethinking,nguyen2025interpretable}. \emph{Executor-based rewards} require answer-checking environments unavailable for tasks such as multi-table joins, schema understanding, or open-ended analytics~\cite{xing2025mmtu}. \emph{Learned reward models} need preference data that is costly to collect at table scale and brittle across domains~\cite{10.5555/3600270.3602281,yang2025tablegptr1advancingtabularreasoning}. \emph{Semantic-similarity rewards} smooth over exact numeric and symbolic strings~\cite{hessel-etal-2021-clipscore,10.1145/3589335.3651526}, and their context limits truncate long serialized tables before scoring, destroying the state information being evaluated.

\textbf{Useful table states are compactly query-aligned.} A useful intermediate table keeps query-relevant headers and values while discarding irrelevant content. \emph{Lexical coverage} measures retained evidence, and \emph{compactness} measures whether the state is narrow enough for the next reasoning step. Both are deterministic functions of the query-table pair, requiring no learned model or executor, and match how tool-using agents progress.

\textbf{\ourmetric{}: a deterministic query-to-state reward built for training-free table reasoning.} \ourmetric{} adapts ROUGE-L's Longest Common Subsequence (LCS)~\cite{rouge} to table states. It serializes each intermediate table into a deterministic ``\texttt{[header] is [value]}'' string and scores it against the query through a single LCS match normalized by the encoded table length. The numerator measures coverage of query-relevant evidence, and the denominator penalizes bloat. \textbf{Three properties make this the right design for training-free table reasoning.} First, lexical matching has no encoding phase and no length ceiling, preserving exact numeric and symbolic strings that embedding rewards lose to truncation. Second, order-preserving LCS over ``\texttt{[header] is [value]}'' rewards the row-level header-value coupling that bag-of-words rewards such as BM25 cannot. Third, the single-LCS denominator makes echo-column hacking unprofitable: duplicating query tokens inflates numerator and denominator at the same rate. We position \ourmetric{} as a \emph{relative process signal} for ranking states within a trajectory rather than an absolute correctness oracle, with scope characterized in \cref{subsec:state-reward}.

\textbf{\ourmodel{}: stepwise feedback and trajectory ranking from a single signal.} \ourmodel{} uses this scalar at two levels. Within a trajectory, \ourmetric{} is appended as a feedback token after each action and as a rolling-variance stopping criterion, so the next action conditions on cumulative state quality rather than on the model's token confidence. Across sampled trajectories, the aggregated trajectory reward serves as the test-time scaling (TTS) selection signal, replacing action-confidence voting and uniform majority. One deterministic reward thus supervises step and trajectory levels without any learned scorer or executor.

\textbf{Contributions and results.}
Across six backbones and three benchmarks, from small open-weight models to GPT-5.4~\cite{openai2026gpt54}, \ourmodel{} improves accuracy by 26.7~pp over the no-reward baseline within the same pipeline and reduces samples needed to reach the 95\% accuracy band by roughly 33\% versus action-confidence voting. Frontier-model results show that the targeted intermediate-state failures persist beyond small backbones, while the same training-free reward remains effective under a state-of-the-art reasoner. As a process reward, \ourmetric{} matches or exceeds trained Qwen2.5-Math-PRM-72B and TaTToo scorers on 4/5 evaluation sets despite having no learned parameters. Failure-mode analysis identifies where lexical supervision is unreliable and shows the headline gains survive each regime. A preliminary GRPO study further suggests \ourmetric{} scales as a post-training reward, adding 8.34~pp over the training-free framework. Our contributions are:
\begin{itemize}
    \item \textbf{A deterministic query-to-state reward for table reasoning.} \ourmetric{} adapts ROUGE-style LCS matching to serialized intermediate tables, with scope (\cref{subsec:motivating-study,subsec:state-reward}) delimiting when lexical supervision is reliable.
    \item \textbf{A plug-and-play, training-free framework for tool-using table reasoning.} \ourmodel{} uses \ourmetric{} for both step-level feedback and trajectory selection, improving accuracy and cutting token cost without any trained scorer (\cref{sub:two-phase}).
    \item \textbf{Empirical validation across models and benchmarks.} Consistent gains over no-reward and PRM baselines on six backbones, ranging from small open-weight models to a state-of-the-art frontier reasoner, and three benchmarks, with a preliminary GRPO study showing viability as a post-training reward (\cref{sub:initial-results,sub:ablation,sub:grpo-ablation,subsec:failure-modes-main}).
\end{itemize}

\section{Related Work}
\textbf{Table Reasoning.}
Table reasoning requires understanding structured tables to infer answers~\cite{10.1609/aaai.v37i11.26536, cheng-etal-2024-call}. Early approaches relied on semantic parsing and pre-training~\cite{jiang2022omnitab, lee-etal-2025-dcg}, with execution-grounded systems such as TAPAS~\cite{herzig-etal-2020-tapas}, TaBERT~\cite{yin2020tabertpretrainingjointunderstanding}, TAPEX~\cite{liu2022tapex}, and RAT-SQL~\cite{wang-etal-2020-rat} obtaining intermediate supervision through trained encoders or runnable SQL executors. \ourmodel{} is complementary, offering a training-free, executor-free state signal that applies wherever a serialized table can be produced. Recently, advanced LLMs and VLMs have enabled complex reasoning via symbolic and multi-step reasoning~\cite{pal-etal-2023-multitabqa, lee2024learninge, zhu-etal-2025-statschartmwp}, integrating executable programs with QA generation~\cite{cheng2023binding, zhang-etal-2024-syntqa} and retrieval-augmented strategies~\cite{lin-etal-2023-inner, wu-etal-2023-tacr}. Sequential Planning frameworks iteratively transform tables using predefined tools~\cite{10.14778/3659437.3659452, wang2024chainoftable}, whereas Compositional Planning decomposes queries into subproblems~\cite{wang-etal-2023-know, 10.1145/3539618.3591708, jiang-etal-2023-structgpt}, with additional enhancements in decomposition and causal analysis~\cite{ji2025treeoftable,yang2025triples,yang2025cit}. Search-based methods employ Monte Carlo Tree Search~\cite{10.1007/s10462-022-10228-y}, process-supervised rewards~\cite{li2025alphasql, zhang2025rewardsqlboostingtexttosqlstepwise,zou2026tattoo}, and progressive verification for multi-step reasoning~\cite{hong-etal-2025-data,chegini-etal-2025-repanda}. Other recent table reasoning systems push in complementary directions through trained trajectory pruning and verifier-guided reinforcement learning~\cite{yang2025tablegptr1advancingtabularreasoning,guo2026rethinkingtablepruningtableqa}. We include a more detailed comparison against concurrent work in \cref{sec:additional-related-work}. Nevertheless, these approaches rely on internal reasoning traces and task-specific executors, limiting the transferability of the supervision signal across pipelines. Our work addresses this gap by studying a lightweight deterministic reward for table-state transitions.


\textbf{Trajectory optimization in table reasoning.}
State-based reasoning is widely studied in visual and multimodal domains that focus on perception, representation learning and the flow of information across reasoning steps~\cite{10.5555/3737916.3740687, liu2024llavanext, basu2024understanding, liu2024vmamba, sarch2025grounded,wang2025vagen,kwok2026tabqaworldoptimizingmultimodalreasoning}. Similar ideas have also been explored in code generation to improve reasoning over intermediate program states~\cite{faircodegenteam2025cwmopenweightsllmresearch}. These developments motivate applying world model perspectives to structured reasoning over tables. In table reasoning, interactive methods have started to incorporate related components, such as tabular grounding~\cite{wang2024chainoftable, yang2025triples, hollmann2025tabular, qu2025tabicl}, causal tracing and interpretability~\cite{yang2025cit}, and self-correction in multi-turn interactions~\cite{nguyen2025interpretable, yang2025tablegptr1advancingtabularreasoning}. Techniques like Mix-SC~\cite{liu-etal-2024-rethinking} further improve robustness by aggregating multiple reasoning trajectories through voting and TTS. Recent RL frameworks further optimize trajectory~\cite{wang2025ragenunderstandingselfevolutionllm,zhang2026starpostabilityaugmentedreinforcementpolicy} by incorporating correctness annotations in table reasoning~\cite{zou2026tattoo}. Despite these advances, current approaches lack stepwise quantitative evaluation of exact table states, leaving quality estimation over table reasoning trajectories largely unexplored. We address this gap with \ourmodel{} by evaluating intermediate table states at both local transition and complete trajectory level. 
\section{Deterministic State-Level Reward for Table Reasoning} \label{sec:methodology}
\cref{sub:problem-formulation} formalizes tool-based table reasoning as a POMDP and test-time trajectory search. We then develop the deterministic state reward in three steps: \cref{subsec:motivating-study} motivates a ROUGE-style design via training-free reward comparisons, \cref{subsec:state-reward} defines \ourmetric{} with its scope and failure modes, and \cref{sub:two-phase} integrates it into the two-phase \ourmodel{} pipeline.

\subsection{Problem Formulation}  \label{sub:problem-formulation}
\textbf{Tool-based Table Reasoning State Transitions.}
We formalize tool-based table reasoning~\cite{wang2024chainoftable,ji2025treeoftable} as a partially observable Markov decision process (POMDP)~\cite{ASTROM1965174}. Each instance consists of a table $T$, a natural-language query $Q$, and a ground-truth answer $\alpha$, with the goal of producing $\hat{\alpha}$ that maximizes an evaluation function $G(\hat{\alpha},\alpha)$. Since LLMs and VLMs struggle to attend to entire tables token by token~\cite{wang2025needleinatable}, agents reason over partial table snapshots. At step $t$, the agent observes $o_t\in O$, a partial view of the latent state $T_t$, and selects an atomic table operation $a_t\in A$~\cite{wang2024chainoftable} with arguments $\theta_{a_t}$. We adopt a curated tool inventory to promote controlled execution and more consistent outcome distributions under current LLMs (\cref{tab:main_results,tab:lower-variance}). Each step uses a single-pass transition: sample one tool call, execute it, and reward the resulting post-action state. Conditioned on $o_t$, the agent generates $\theta_{a_t}$, executes $a_t$, and induces $T_{t+1}\sim P(\cdot\mid T_t,a_t)$, $r_t=r(a_t,T_{t+1},Q)$, and $o_{t+1}\sim\Omega(\cdot\mid T_{t+1})$. The trajectory $\tau$ aggregates $\{r_t\}_{t=0}^{S}$ into a discounted return $R(\tau)$.

\textbf{Optimal Test-time Trajectory Search.}
At test time, the agent samples multiple independent rollouts under the same prompt and tool interface, then selects the highest-reward trajectory. We use sample-and-rank rather than breadth-first enumeration to better match OOD inference. With $R(\tau)=\sum_{l=0}^{S}\gamma_l\,r(a_l,T_{l+1},Q)$ and $\gamma_l\in[0,1]$, the agent returns $\hat{\alpha}_{\tau^*}=a_S\in\tau^*$, where $\tau^*=\arg\max_{\tau\in\mathcal{T}}R(\tau)$. The maximization is over complete sampled trajectories, not alternative actions within a rollout. \ourmodel{} computes step-wise rewards from observations and aggregates them into $R(\tau)$ for selection.

\subsection{Pilot Study: Comparing Reward Families} \label{subsec:motivating-study}
\ourmodel{} introduces a training-free \emph{deterministic state-supervision signal} for tool-based table reasoning, scoring whether each action moves the intermediate table state toward the query. Before formalizing \ourmetric{}, we evaluate existing training-free reward families on three complex table reasoning datasets, WikiTableQuestions (WTQ)~\cite{pasupat-liang-2015-compositional}, MMQA~\cite{wu2025mmqa} and MMTU~\cite{xing2025mmtu} using GPT-4.1-nano~\cite{openai_gpt41_2025} as base reasoning model. We include semantic-based BERTScore~\cite{Zhang*2020BERTScore:}, and lexical-based BM25~\cite{10.1561/1500000019}, and ROUGE LCS matching~\cite{rouge}. We operate on the deterministic ``\texttt{[header] is [value]}'' table-state serialization in~\cite{borisov2023language}, further specified in \cref{sec:textual-encoding}. An additional image-rendered VLM agent baseline serves as an orthogonal modality axis, analyzed in \cref{subsec:case-study-vlm-appendix}. 

\begin{table*}[t]
    \centering
    \caption{VLM-based rewards improve on reasoning-focused datasets but degrade on exact lookup in large tables, indicating the need for tabular-specific tuning. Embedding-based similarity rewards are also unstable due to truncation that disrupts table structure. The ROUGE-L column reports the standard F-measure on the ``\texttt{[header] is [value]}'' encoding. \ourmetric{} (\cref{subsec:state-reward}) refines this with a precision-leaning $|\operatorname{Enc}(T)|$ denominator. Results report accuracy (mean $\pm$ std).}
    \label{tab:vs-vagen}
    \vspace{1mm}
    \resizebox{0.82\textwidth}{!}{
    \begin{tabular}{c|c|c|c|c|c}
    \toprule
    Dataset &
    \makecell{Tabular \\ no reward} &
    \makecell{VLM \\ + rewards} &
    \makecell{Tabular \\ + BERTScore} &
    \makecell{Tabular \\ + BM25} &
    \hilite{\textbf{\makecell{Tabular \\ + ROUGE-L}}} \\
    \midrule
    WTQ & 59.18 $\pm$ 3.48 & 61.22 $\pm$ 3.45 & 51.02 $\pm$ 3.53 & 57.10 $\pm$ 3.50 & \textbf{67.35 $\pm$ 3.34}\\
    MMQA & 32.91 $\pm$ 3.32 & 45.57 $\pm$ 3.52 & 48.10 $\pm$ 3.53 & 51.90 $\pm$ 3.53 & \textbf{60.76 $\pm$ 3.45}\\
    MMTU & 53.26 $\pm$ 3.53 & 28.26 $\pm$ 3.18 & 52.17 $\pm$ 3.53 & 54.30 $\pm$ 3.52 & \textbf{69.57 $\pm$ 3.25} \\
    \midrule
    Average & 48.45 $\pm$ 3.53 & 45.02 $\pm$ 3.52 & 50.43 $\pm$ 3.54 & 54.43 $\pm$ 3.52 & \textbf{65.89 $\pm$ 3.35}\\
    \bottomrule
    \end{tabular}
    }
\end{table*}
\cref{tab:vs-vagen} surfaces three observations. First, reward feedback matters: all three text-based variants beat the no-reward baseline on average. Second, lexical rewards (BM25, ROUGE) outperform semantic BERTScore consistently. Beyond prior findings that embeddings fail to capture numeric differences~\cite{10.1145/3589335.3651526}, embedding context limits (e.g., 512-token DeBERTa~\cite{he2021debertadecodingenhancedbertdisentangled}, 4096-token QWEN3-embedding~\cite{zhang2025qwen3embeddingadvancingtext}) force truncation of long table serializations and destroy state information being scored. Lexical matching has no such limit (further length analysis in \cref{sec:analysis-size}). Third, ROUGE outperforms BM25 because BM25 ignores position~\cite{10.1561/1500000019} and cannot reward row-level header-value coupling, whereas LCS rewards order-preserving matches under the ``\texttt{[header] is [value]}'' encoding. These observations motivate \ourmetric{}, which we formalize next.

\subsection{\ourmetric{}: Adapting ROUGE to Tabular State Supervision} \label{subsec:state-reward}
\begin{figure}[t]
    \centering
    \includegraphics[width=0.95\linewidth]{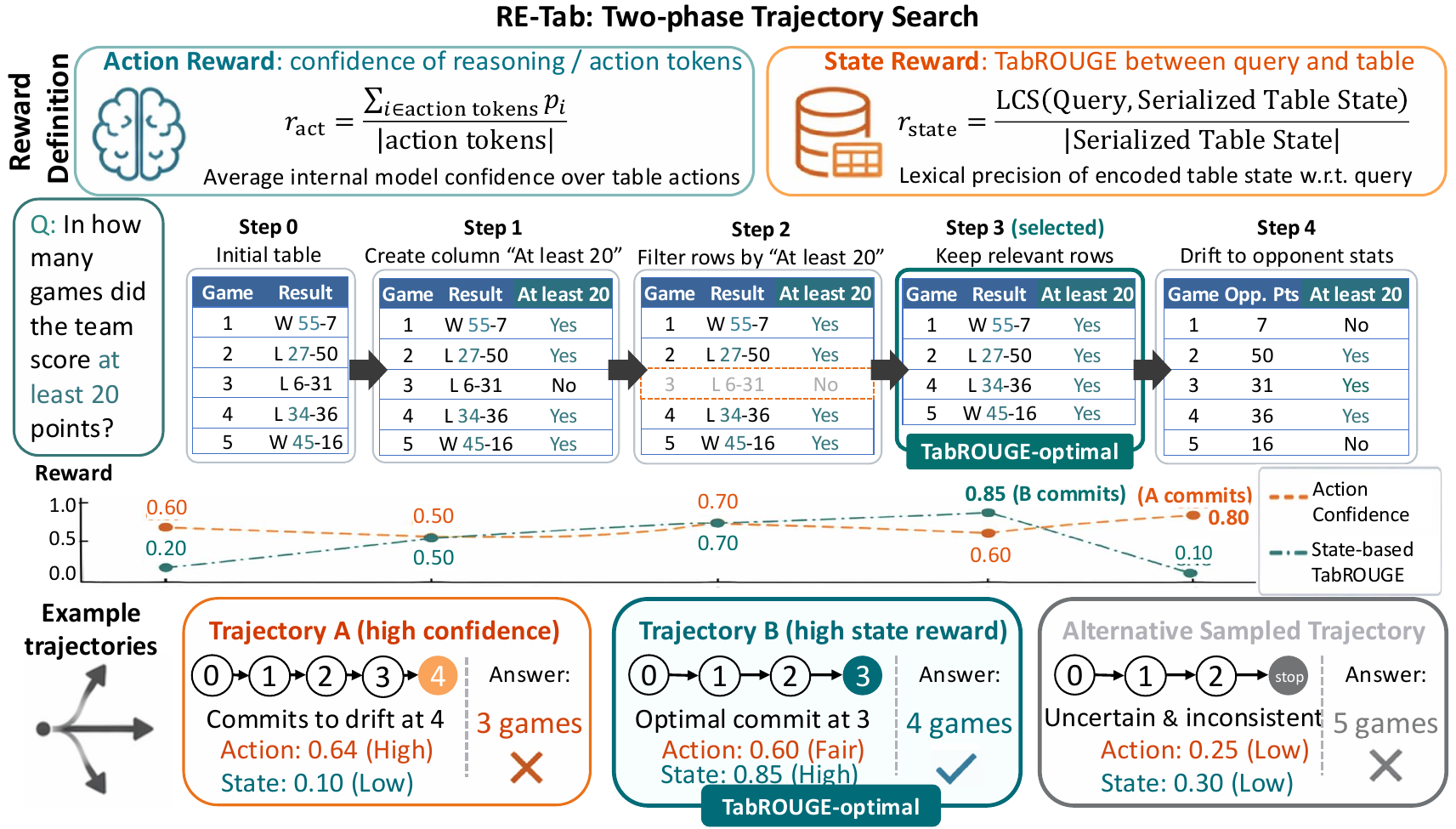}

    \caption{\textbf{State reward is more reliable than action-token confidence for trajectory selection.} \textcolor{orange}{Action reward} averages model confidence over action tokens, while \textcolor{teal}{state reward}~\ourmetric{} scores the post-action table against the query. \textcolor{orange}{Action reward} favors Step~4, where the trajectory drifts to opponent statistics and yields a confidently incorrect answer. \textcolor{teal}{State reward} peaks at Step~3, where the table retains the relevant evidence and produces the correct answer. Reward values follow the same illustrative convention as in \cref{fig:motivation}.}
\label{fig:main-workflow}
\vspace{-3mm}
\end{figure}
We extend ROUGE-L~\cite{rouge} to formally define \ourmetric{}. After each step $l$, the updated table state $T_{l+1}$ is converted into the textual sequence $\operatorname{Enc}(T_{l+1})$ via the deterministic ``\texttt{[header] is [value]}'' encoding introduced above. We include detailed design rationale and encoding sensitivity in \cref{sec:tabrouge-design,sub:delimiter-sensitivity}. The reward compares the encoded table against the query $Q$: \begin{align} \label{eq:tabrouge}
r_{\ourmetric}(a_{l}, T_{l+1}, Q)
&= \frac{\operatorname{LCS}(Q, \operatorname{Enc}(T_{l+1}))}{|\operatorname{Enc}(T_{l+1})|}.
\end{align}

The numerator counts query-relevant content preserved in the state, the denominator penalizes bloat. This rewards lexical coverage and precision without embedding-driven truncation: in \cref{fig:main-workflow}, \ourmetric{} peaks at Step 3 where the relevant ``at least 20'' column overlaps the query with irrelevant rows filtered out. Duplicating query tokens does not produce proportional gains, mitigating reward hacking. The trajectory reward $R_{\text{state}}(\tau) = \sum_{l=0}^{S} \gamma_{l}\,r_{\ourmetric}(a_l, T_{l+1}, Q)$ uses $\gamma_l{=}1$ throughout. Rolling-variance early stopping terminates near the \ourmetric{} peak (e.g., Step~3), so the unweighted sum concentrates on peak contributions and trajectory-level argmax selects the best peak-aligned rollout. Formal analysis, discount sensitivity, and validation tests are in \cref{sub:theoretical-support-rouge,sec:implementation-details,sec:tabrouge-design}.

\paragraph{Design rationale.} The LCS-over-``\texttt{[header] is [value]}'' design gives \ourmetric{} three properties that address the failure modes of training-free alternatives in \cref{tab:vs-vagen}. Lexical matching has no encoding phase and no length ceiling, preserving the exact numeric and symbolic strings that embedding rewards lose to truncation. Order-preserving LCS rewards the row-level header-value coupling that bag-of-words BM25 cannot capture. The single-LCS denominator $|\operatorname{Enc}(T_{l+1})|$ resists token stuffing because duplicating query tokens inflates numerator and denominator at the same rate. Together these handle \emph{lexical coverage, structural integrity, exact numeric and symbolic match, and bloat penalization without a learned model or executor.} \vspace{-1mm}

\paragraph{Scope and known failure modes.} Three failure regimes follow from the lexical formulation: paraphrase and derived-evidence under-rewarding (F1), echo-column reward hacking (F2), and renamed/distractor-column collisions (F3). Because \ourmetric{} is consumed as a \emph{relative} Best-of-$K$ ranking signal rather than an absolute correctness verifier, F1 attenuates correct-trajectory scores symmetrically without flipping argmax, and F2/F3 are bounded by the denominator. \cref{subsec:failure-modes-main} quantifies each regime on real trajectories. Full correlation, calibration, and stress-test results are in \cref{subsec:reward-validation,subsec:stress-tests}. \vspace{-1mm}

\subsection{Two-Phase Implementation of Deterministic State Supervision} \label{sub:two-phase} \vspace{-1mm}
\paragraph{Two-Phase Search.} \ourmodel{} operationalizes deterministic state supervision at the step and trajectory levels, exploring and scoring without backtracking on low-reward actions:
\begin{enumerate}[noitemsep,topsep=2pt,parsep=0pt]
    \item \textbf{Intermediate state reasoning}: at each step the agent scores the post-action state via $r(a_l, T_{l+1}, Q)$ and appends the scalar as a feedback token (e.g., \texttt{[reward: 0.72]}). \ourmetric{} plays two roles: (i) \emph{reward in prompt}, conditioning the next action on cumulative progress, and (ii) \emph{reward for stopping}, triggering rolling-variance early stopping once threshold $v$ is reached (hyperparameters in \cref{tab:all_hyperparameters}).
    \item \textbf{TTS trajectory search}: the agent ranks $\mathcal{T}$ by the aggregated state-supervision signal, outputs $\tau^{*}=\arg\max_{\tau\in\mathcal{T}} R(\tau)$, and returns $\hat{\alpha}_{\tau^{*}}=a_{S}\in \tau^{*}$.
\end{enumerate}

\paragraph{Optimization Objective.} \label{par:simulative-objective} The agent maximizes the expected state-supervision signal $J(\mathcal{T}) = \mathbf{E}_{\tau \sim \mathcal{T}} \big[ R_{\text{state}}(\tau) \big]$, biasing toward trajectories whose intermediate states stay query-aligned across turns, favoring preserved evidence over plausible language traces. \vspace{-2mm}
\section{Experiments} \label{sec:experiments} \vspace{-1mm}
We evaluate \ourmodel{} along three axes. \cref{sub:initial-results} reports state-of-the-art comparisons on two pipelines and head-to-head against trained PRM scorers. \cref{sub:ablation} ablates accuracy and trajectory efficiency under six modern backbones spanning small open-weight models to a state-of-the-art frontier model. \cref{sub:grpo-ablation,subsec:failure-modes-main} examine post-training viability and the empirical magnitude of the failure modes characterized in \cref{subsec:state-reward}.

\textbf{Baseline.} We evaluate \ourmodel{} as a framework incorporated into two major table reasoning paradigms: sequential planning, represented by Chain-of-Table~\cite{wang2024chainoftable} (example illustrated in \cref{fig:motivation,fig:main-workflow}), and compositional planning, represented by Tree-of-Table~\cite{ji2025treeoftable} that decomposes queries into atomic sub-queries. We reproduce each pipeline to isolate the effect of \ourmodel{} as a controlled comparison and use GPT-3.5 as the base model for parity with prior systems. We further isolate \ourmetric{} as a process reward model (PRM) by inserting it into the same pipeline as the recent PRM baselines~\cite{zou2026tattoo}. All PRM-baseline comparisons share the DeepSeek-R1-Distill-Qwen-14B backbone~\cite{deepseek2025r1}, Best-of-4 selection, prompt template, and step budget, with only the scorer swapped.

\textbf{Datasets.} Following standard table-reasoning protocols~\cite{yang2025cit,zou2026tattoo}, we report accuracy on WTQ, TableBench (TB: numerical reasoning (NR), fact checking (FC), data analysis (DA))~\cite{10.1609/aaai.v39i24.34739}, TabFact~\cite{Chen2020TabFact}, and MMQA~\cite{wu2025mmqa}, using the unified extract-normalize-compare protocol in \cref{subsec:evaluation-protocol}. We sample $n=200$ per dataset with random seed 1018, except for the full TB-FC dataset where $n=96$. The corresponding binomial 95\% Wilson half-widths are $\pm 6.86$ pp at $n=200$ and $\pm 9.81$ pp at $n=96$. Cross-condition gaps below these resolutions should be read as ties. \vspace{-1mm}
\subsection{Main Benchmark Results} \label{sub:initial-results} \vspace{-1mm}
\begin{table}[t!]
    \centering
    \small
    \caption{Main benchmark comparison across datasets. GPT-3.5-based methods report TabFact while PRM baselines (DeepSeek-R1-Distill-Qwen-14B, Best-of-4) report MMQA instead. ``---'' indicates the result was not reported by the original work under a comparable evaluation setting.}
    \label{tab:initial-result}
    \setlength{\tabcolsep}{5pt}
    \renewcommand{\arraystretch}{1.1}
    \resizebox{0.9\textwidth}{!}{%
    \begin{tabular}{l c c c c c c}
        \toprule
        \textbf{Method} & \textbf{WTQ} & \textbf{TB-NR} & \textbf{TB-FC} & \textbf{TB-DA} & \textbf{TabFact} & \textbf{MMQA} \\
        \midrule
        \rowcolor{gray!11} \multicolumn{7}{l}{\textit{GPT-3.5-based Table Reasoning Agent}} \\
        Binder & 54.8 & 66.8 & 67.7 & 26.8 & 83.3 & --- \\
        Chain-of-Table & 59.9 & 68.5 & 78.1 & 30.3 & 80.2 & --- \\
        \textit{Chain-of-Table Repl.} & 62.2 & 66.3 & 70.7 & 32.6 & 82.6 & --- \\
        Tree-of-Table & 61.1 & --- & --- & --- & 81.9 & --- \\
        \textit{Tree-of-Table Repl.} & 60.2 & 48.9 & 56.5 & 31.5 & 81.5 & --- \\
        Dater & 65.8 & 65.0 & 69.8 & 28.6 & 83.6 & --- \\
        TabSQLify & 68.7 & 65.2 & 76.0 & 28.0 & 78.3 & --- \\
        TALON & 70.7 & 67.3 & 77.1 & 28.9 & 87.6 & --- \\
        Table-Critic & 72.6 & \underline{73.0} & \underline{81.3} & 33.8 & 90.6 & --- \\
        CIT-DP & \underline{76.4} & --- & --- & --- & \underline{91.3} & --- \\
        TIDE & 75.0 & --- & --- & --- & 89.8 & --- \\
        \textbf{\ourmodel{} Chain-of-Table (GPT-3.5)} & {76.5} & \textbf{76.1} & \textbf{81.5} & \textbf{41.3} & \textbf{92.4} & --- \\
        \textbf{\ourmodel{} Tree-of-Table (GPT-3.5)} & {74.5} & 65.2 & 73.9 & \underline{38.0} & \underline{91.3} & --- \\
        \midrule
        \rowcolor{gray!11} \multicolumn{7}{l}{\textit{Process Reward Model baselines on DeepSeek-R1-Distilled-Qwen-14B base model (Best-of-4)}} \\
        Majority Vote & 64.7 & 65.5 & 76.2 & 23.5 & --- & 18.4 \\
        LLM-as-a-judge & 65.2 & 66.7 & 77.2 & 23.5 & --- & 19.6 \\
        Skywork-PRM-7B & 65.9 & 66.1 & 76.8 & 24.1 & --- & 21.4 \\
        Math-Shepherd-PRM-7B & 66.8 & 67.2 & 76.2 & 22.7 & --- & 22.0 \\
        Qwen2.5-Math-PRM-7B & 65.2 & 66.9 & 75.4 & 23.2 & --- & 23.5 \\
        ThinkPRM & 64.3 & 69.2 & 75.8 & 21.6 & --- & 22.4 \\
        GenPRM & 69.8 & \textbf{71.5} & 76.3 & 25.3 & --- & 23.8 \\
        Qwen2.5-Math-PRM-72B & 69.2 & 70.4 & 77.8 & 25.5 & --- & 24.4 \\
        {TaTToo} & {69.8} & {71.2} & {77.4} & {27.7} & --- & {25.1} \\
        \textbf{\ourmodel{} (DS-R1-Distill-Qwen-14B)} & \textbf{71.4} & 71.1 & \textbf{79.4} & \textbf{35.5} & --- & \textbf{46.7} \\
        \bottomrule
    \end{tabular}}
    \vspace{-2mm}
\end{table}
\textbf{State-of-the-art Table Reasoning Agent.} \cref{tab:initial-result} shows that \ourmodel{} consistently improves reasoning accuracy over both pipelines and matches or exceeds prior systems. In Chain-of-Table, \ourmodel{} exceeds prior CIT-DP, while Tree-of-Table integration matches within 2~pp. Because CIT-DP and TIDE codebases are unreleased, we report only published WTQ and TabFact numbers, and we reproduce~\cite{wang2024chainoftable,ji2025treeoftable} within 0.4~pp to ground the improvements.

\textbf{PRM head-to-head.} On reward-model robustness, the training-free \ourmetric{} outperforms the larger Qwen2.5-Math-PRM-72B and the Qwen3-8B-finetuned TaTToo~\cite{zou2026tattoo,yang2025qwen3technicalreport} on 4 of 5 evaluation sets (\cref{tab:initial-result}), with TB-NR a near-tie ($71.1$ vs $71.2$, within the Wilson half-width). When supervision can be defined directly on the encoded table state, a deterministic stepwise reward reaches parity with much larger trained scorers without being trained itself. We next ablate how \ourmodel{} contributes across reasoning settings.
\vspace{-1mm}
{
\clubpenalty=0        
\widowpenalty=0       
\brokenpenalty=0
\interlinepenalty=0
\subsection{Ablation: Accuracy and Efficiency} \label{sub:ablation} \vspace{-1mm}
\textbf{Setup.} We narrow benchmarks to WTQ~\cite{pasupat-liang-2015-compositional} (ID), MMQA~\cite{wu2025mmqa}, and MMTU~\cite{xing2025mmtu} (recent benchmarks approximating OOD conditions). To test generalization beyond the GPT-3.5 parity setting, we re-run \ourmodel{} on six backbones (QWEN3-8B~\cite{yang2025qwen3technicalreport}, Ministral-3-8B~\cite{mistralai_ministral3_8b_25-12}, GPT-4.1-nano~\cite{openai_gpt41_2025}, GPT-oss-20B~\cite{openai_gptoss20b_2025}, GPT-5-nano~\cite{openai_gpt5nano_2025}, and GPT-5.4~\cite{openai2026gpt54}). The first five cover small open-weight and lightweight proprietary backbones, and we include GPT-5.4 to verify that the targeted failure modes and gains persist under a frontier model. We ablate along two axes: (i) single-trajectory inference comparing no-reward against \ourmetric{}, and (ii) multi-trajectory TTS under full \ourmodel{}. For TTS, we benchmark state-based \ourmetric{} against standard action-confidence frameworks~\cite{anonymous2025deep,razghandi-etal-2025-cer} (\cref{fig:main-workflow}), including chain- and step-level confidence and equal-weighting majority voting as controls. All ablation numbers are from our own runs. Implementation details, confidence derivations, and BLEU~\cite{10.3115/1073083.1073135} / ROUGE-L~\cite{rouge} results are in \cref{sec:detail-operation,sec:implementation-details,sub:blue-rouge}.

\textbf{Accuracy.} \cref{tab:main_results} reports that \ourmetric{} benefits both intermediate-state reasoning and TTS trajectory search, as described in \cref{sub:two-phase}. Reward-based supervision over intermediate reasoning consistently improves over no-reward baselines across models, datasets, and reasoning paradigms, with an average accuracy gain of 26.7~pp. Under TTS, using \ourmetric{} as a trajectory voting weight further improves performance. The state-based \ourmetric{} generally outperforms action confidence and uniform rewards on the OOD MMQA and MMTU sets, while on in-distribution WTQ the two reward families are comparable, plausibly because pretraining bias makes internal action confidence already aligned with the answer in this setting. The same pattern holds for the frontier GPT-5.4 backbone, which already starts from a much stronger no-reward baseline yet still benefits from \ourmetric{} on both single-trajectory and TTS settings, indicating that the intermediate-state failure modes \ourmodel{} addresses are not eliminated by larger pretrained reasoners. Similar trends in BLEU and ROUGE-L appear in \cref{tab:blue-rouge}, with volatility analysis in \cref{sub:supporting-main-results}.

\textbf{Efficiency.}
Beyond accuracy, we measure how many trajectories the agent needs before its prediction stabilizes, sampling 100 trajectories per WTQ query and defining convergence as 95\% accuracy with under 5\% volatility. As shown in \cref{fig:reward-analysis-combined}(a, b), no-reward search needs 22 to 27 trajectories, action-confidence reward reduces this to 15 to 18, and state-based \ourmetric{} reaches the same band with 12, a 20-33\% drop relative to action confidence. The figure stays inside the rolling-variance early-stopping budget (\cref{sec:implementation-details}). Per-step \ourmetric{} injection is complementary to TTS voting. Removing it degrades \ourmetric{}-weighted performance, while the confidence-based variant is unaffected because action-token confidence does not depend on injected feedback. End-to-end latency, token budgets, and the per-step LCS bound $O(|Q|\,|\mathrm{Enc}(T_t)|)$ are reported in \cref{tab:cost-accounting,sub:efficiency-protocol}, with LLM drifting analysis in \cref{sub:llm-drifting}.}
\begin{table}[!t]
\centering
\vspace{-5mm}
\caption{Each method block separates \emph{single-trajectory} inference (no reward, \ourmetric{}) from \emph{multi-trajectory} TTS (EQ, EQ+\ourmetric{} filtering, AC, AS, RG); accuracy (\%). \ourmetric{} improves no-reward baselines, and RG outperforms action-confidence voting on OOD datasets. \inlinehl{Second}{yellow} marks \ourmetric{} single-turn inference and \inlinehl{Best}{green} marks full \ourmodel{} (ST+RG). \textbf{Bold}/\underline{underline} denote best/second per row.}
\label{tab:main_results}

\setlength{\tabcolsep}{3pt}
\renewcommand{\arraystretch}{1.15}
\definecolor{stcol}{HTML}{F2F2F7}

\resizebox{0.9\textwidth}{!}{%
\begin{tabular}{l|c|>{\columncolor{stcol}}c>{\columncolor{stcol}}c|ccccc|>{\columncolor{stcol}}c>{\columncolor{stcol}}c|ccccc}
\toprule
\multirow{2}{*}{Model}
& \multirow{2}{*}{\textbf{Code}}
& \multicolumn{7}{c|}{\textbf{Chain-of-Table (Tools)}}
& \multicolumn{7}{c}{\textbf{Tree-of-Table (Tools)}} \\
\cmidrule(lr){3-9}\cmidrule(lr){10-16}
&
& \multicolumn{2}{c|}{\textit{Single Trajectory}} & \multicolumn{5}{c|}{\textit{Multi Trajectory (TTS)}}
& \multicolumn{2}{c|}{\textit{Single Trajectory}} & \multicolumn{5}{c}{\textit{Multi Trajectory (TTS)}} \\
\cmidrule(lr){3-4}\cmidrule(lr){5-9}\cmidrule(lr){10-11}\cmidrule(lr){12-16}

\ourmetric{} & \cross & \cross & \hilight{\tick} & \cross & \tick & \tick & \tick & \hilite{\tick}
 & \cross & \hilight{\tick} & \cross & \tick & \tick & \tick & \hilite{\tick}\\
 TTS Voting & \cross & \cross & \hilight{\cross} & EQ & EQ & AC & AS & \hilite{RG} & \cross & \hilight{\cross} & EQ & EQ & AC & AS & \hilite{RG}\\
\midrule
\multicolumn{16}{c}{WTQ}\\
\midrule
 QWEN3-8B
& 32.65 & 32.65 & \hilight{46.94} & 34.69 & 46.94 & 48.98 & \textbf{58.16} & \hilite{55.61}
& 35.71 & \hilight{44.90} & 44.90 & 42.86 & 53.06 & \textbf{57.14} & \hilite{55.10} \\
 Ministral-3-8B
& 35.71 & 45.92 & \hilight{51.02} & 42.86 & 57.14 & \textbf{67.35} & \textbf{67.35} & \hilite{\textbf{67.35}}
& 42.86 & \hilight{49.49} & 53.06 & 55.10 & \textbf{63.27} & 61.22 & \hilite{\textbf{63.27}} \\
 GPT-4.1-nano
& 44.90 & 59.18 & \hilight{67.35} & 63.27 & 69.39 & 75.51 & 77.55 & \hilite{\textbf{81.63}}
& 55.10 & \hilight{65.31} & 58.16 & 61.22 & 67.35 & 69.39 & \hilite{\textbf{71.43}} \\
 GPT-oss-20B
& 56.12 & 62.24 & \hilight{69.39} & 65.31 & 71.43 & 81.63 & \textbf{87.76} & \hilite{85.71}
& 57.14 & \hilight{66.33} & 69.39 & 83.67 & 83.67 & \textbf{85.71} & \hilite{\textbf{85.71}} \\
 GPT-5-nano
& 48.98 & 79.59 & \hilight{\textbf{89.80}} & 81.63 & 85.71 & \textbf{89.80} & \textbf{89.80} & \hilite{\textbf{89.80}}
& 79.59 & \hilight{83.67} & 85.71 & 85.71 & \textbf{91.84} & \textbf{91.84} & \hilite{\textbf{91.84}} \\
 GPT-5.4
& 78.43 & 88.24 & \hilight{{90.20}} & 88.24 & 90.20 & \textbf{92.16} & \textbf{92.16} & \hilite{\textbf{92.16}} & 82.35 & \hilight{88.24} & 86.27 & 88.24 & \textbf{94.12} & \textbf{94.12} & \hilite{\textbf{94.12}} \\
\midrule
\multicolumn{16}{c}{MMQA}\\
\midrule
 QWEN3-8B
& 24.05 & 25.32 & \hilight{40.51} & 26.58 & 34.18 & 37.97 & 41.77 & \hilite{\textbf{44.30}}
& 36.71 & \hilight{43.04} & 37.97 & 41.77 & 39.24 & 39.24 & \hilite{\textbf{44.94}}\\
 Ministral-3-8B
& 22.78 & 27.85 & \hilight{43.04} & 27.85 & 40.51 & 43.04 & 37.97 & \hilite{\textbf{49.37}}
& 35.44 & \hilight{46.84} & 37.97 & 50.63 & 49.37 & 55.70 & \hilite{\textbf{56.96}} \\
 GPT-4.1-nano
& 22.78 & 32.91 & \hilight{60.76} & 49.37 & 64.56 & 63.29 & \textbf{68.35} & \hilite{67.09}
& 31.65 & \hilight{58.23} & 51.90 & 56.96 & 67.09 & 68.35 & \hilite{\textbf{73.42}} \\
 GPT-oss-20B
& 34.18 & 55.70 & \hilight{72.15} & 56.96 & 73.42 & 74.68 & 78.48 & \hilite{\textbf{83.54}}
& 60.76 & \hilight{74.68} & 63.29 & 75.95 & 77.22 & \textbf{83.54} & \hilite{79.75} \\
 GPT-5-nano
& 41.77 & 49.37 & \hilight{63.29} & 60.76 & 70.89 & 79.75 & 82.28 & \hilite{\textbf{83.54}}
& 56.96 & \hilight{74.68} & 60.76 & 83.54 & 84.81 & \textbf{86.08} & \hilite{\textbf{86.08}} \\
 GPT-5.4
& 67.09 & 84.81 & \hilight{89.87} & 83.54 & 89.87 & 91.14 & \textbf{93.67} & \hilite{\textbf{93.67}}
& 82.28 & \hilight{88.61} & 79.75 & 87.34 & 87.34 & {87.34} & \hilite{\textbf{89.87}} \\
\midrule 
\multicolumn{16}{c}{MMTU}\\
\midrule
QWEN3-8B
& 30.43 & 36.96 & \hilight{45.65} & 35.87 & 44.57 & 46.20 & 48.91 & \hilite{\textbf{52.17}}
& 25.27 & \hilight{29.35} & 28.26 & 31.52 & 30.43 & 31.52 & \hilite{\textbf{33.70}} \\
 Ministral-3-8B
& 28.80 & 29.35 & \hilight{36.96} & 27.17 & 32.61 & 48.91 & 52.17 & \hilite{\textbf{54.35}}
& 27.17 & \hilight{29.89} & 28.80 & 32.34 & 33.70 & 34.78 & \hilite{34.78}\\
 GPT-4.1-nano
& 50.00 & 53.26 & \hilight{69.57} & 63.04 & 72.83 & 76.09 & \textbf{79.35} & \hilite{\textbf{79.35}}
& 52.17 & \hilight{65.22} & 57.61 & 70.65 & 75.00 & 76.09 & \hilite{\textbf{78.26}} \\
 GPT-oss-20B
& 54.35 & 48.91 & \hilight{69.57} & 51.09 & 61.96 & 72.83 & 75.00 & \hilite{\textbf{77.17}}
& 44.57 & \hilight{60.87} & 51.09 & 69.57 & 77.17 & 75.00 & \hilite{\textbf{81.52}} \\
 GPT-5-nano
& 67.39 & 80.43 & \hilight{85.87} & 80.43 & 88.04 & 89.13 & \textbf{90.22} & \hilite{\textbf{90.22}}
& 45.65 & \hilight{80.43} & 47.83 & 81.52 & 82.61 & 82.61 & \hilite{\textbf{85.87}} \\
 GPT-5.4
& 81.72 & 83.87 & \hilight{89.25} & 84.95 & 88.17 & 89.25 & {88.17} & \hilite{\textbf{91.40}}
& 78.49 & \hilight{89.25} & 83.87 & 89.24 & 88.17 & 89.25 & \hilite{\textbf{90.32}} \\
\bottomrule
\end{tabular}}
\vspace{-3mm}
\end{table}
\begin{figure}[t]
    \centering
    \includegraphics[width=\textwidth]{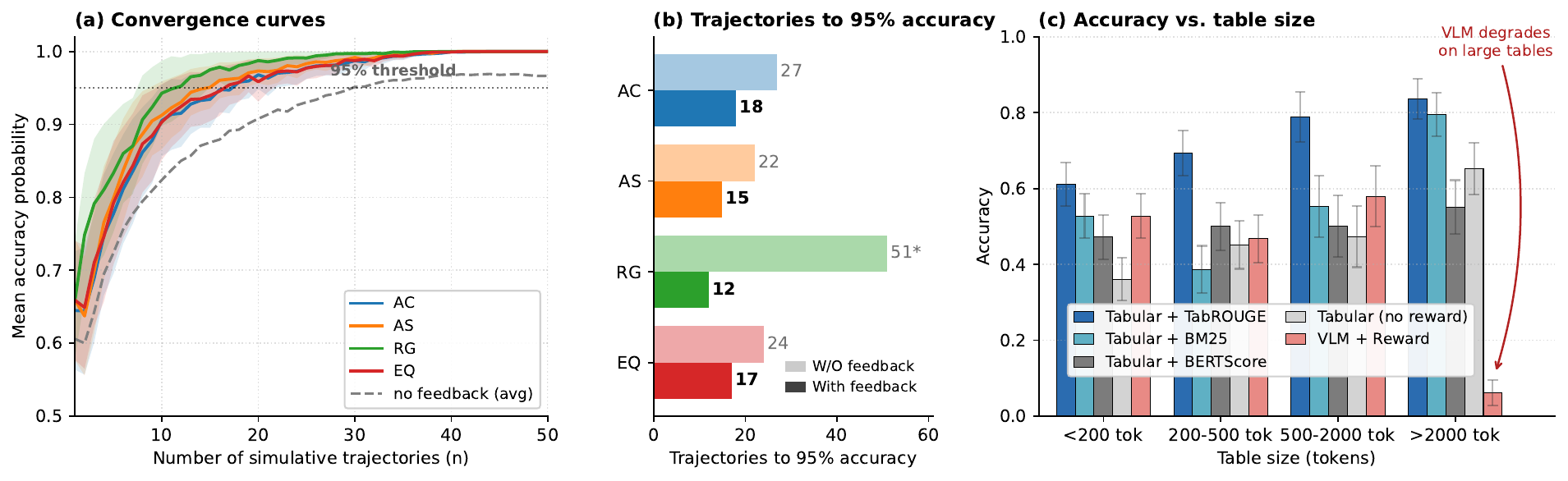}
    \caption{Reward design and table-size effects. \textbf{(a)~\ourmetric{} reaches the 95\% accuracy band with fewest sampled trajectories.} Mean accuracy probability vs number of sampled trajectories under four stepwise rewards with $\pm 1$~SE, along with no-feedback average in grey dashed curve. \textbf{(b)~Per-step \ourmetric{} injection is complementary to TTS voting.} Trajectories required to reach 95\% accuracy with \textbf{dark}/\textbf{light} bars are with/without stepwise feedback. \textbf{(c)~\ourmetric{} is the only deterministic reward that stays smooth across table sizes.} On GPT-4.1-nano, competing rewards dip in the medium-table bin and recover at large sizes, whereas \ourmetric{} tracks accuracy monotonically across small, medium, and large tables.}
    \label{fig:reward-analysis-combined}
    \vspace{-4mm}
\end{figure}

\begin{wraptable}{r}{0.4\linewidth}
\centering
\small
\vspace{-8mm}
\caption{GRPO post-training with \ourmetric{} on QWEN3-8B as the verifiable reward improves over the training-free TTS variant on WTQ and MMQA. }
\label{tab:grpo-preliminary}
\setlength{\tabcolsep}{8pt}
\begin{tabular}{lccc}
\toprule
\textbf{Dataset} & \textbf{Base} & \textbf{TTS} & \textbf{GRPO} \\
\midrule
WTQ & 32.65 & 55.61 & \textbf{77.55} \\
MMQA   & 25.32 & 44.30 & \textbf{58.23} \\
MMTU   & 36.96 & \textbf{52.17} & 41.30 \\
\midrule
Average & 31.64 & 50.69 & \textbf{59.03} \\
\bottomrule
\end{tabular}
\end{wraptable}
\subsection{GRPO Post-Training} \label{sub:grpo-ablation}\vspace{-0.5mm}
We further test whether \ourmetric{} can serve as a deterministic reward for policy-level post-training, beyond its training-free use as a search controller. Under a $4\times$ A100 budget we run an initial GRPO~\cite{shao2024deepseekmathpushinglimitsmathematical} study on QWEN3-8B, training WTQ and MMQA for 300 steps and MMTU for 150 steps (long-valued MMTU tables cause out-of-memory (OOM) errors at the longer schedule). \cref{tab:grpo-preliminary} shows an average 8.34~pp improvement over the training-free TTS baseline. On MMTU, GRPO improves over the base model but does not reach the training-free variant, which we attribute to the OOM-induced halved schedule rather than to a reward limitation. A full multi-model study under uniform compute is left to future work. These results stand as a proof-of-concept that \ourmetric{} can serve as a reward for policy-level post-training, not only a heuristic search add-on. \vspace{-1mm}
\subsection{Failure Modes, Sensitivity, and Why Headline Gains Survive} \label{subsec:failure-modes-main} \vspace{-0.5mm}
We quantify the three failure regimes from \cref{subsec:state-reward} on real trajectories using the \emph{ranking-inversion rate}, the fraction of paired (correct, wrong) intermediate-state pairs whose ordering flips after a semantics-preserving perturbation, since only an inversion can mislead Best-of-$K$ selection. Across $1{,}063$ paired perturbations spanning paraphrase, unit conversion, and derived-evidence categories (F1), \ourmetric{} produces zero inversions against a mean correct-vs-wrong reward gap of $0.20$ (\cref{tab:stress_test}). For F2, the denominator $|\operatorname{Enc}(T)|$ in \cref{eq:tabrouge} absorbs the boost in all but one audited trajectory ($+0.013$, $\sim7\%$ of the reward gap, \cref{tab:echo-column-audit}). F3 produces a $0.8\%$ inversion rate across $N{=}275$ pairs, well below the Wilson half-width of \cref{tab:main_results} (6.86~pp). Per-question-type analysis (\cref{fig:question-type-analysis}) shows the largest gains on arithmetic queries (Sum, Difference, Ratio: $+7$ to $+18$~pp), confirming the lexical formulation does not penalize numeric reasoning. The 26.7~pp within-pipeline gain therefore cannot be explained by these regimes. Detailed cases, VLM-agent comparison, $\beta$-sweep, and per-condition reward-injection results are in \cref{subsec:failure-cases,subsec:echo-column-audit,subsec:stress-tests,subsec:case-study-vlm-appendix,subsec:beta-sensitivity,subsec:question-type-analysis,subsec:reward-injection-detail}.
\section{Discussion and Limitations}\label{sec:discussion}
In this work, we position deterministic state supervision to improve table reasoning quality through dynamic, state-based table understanding. Process-level supervision alone yields a 26.7~pp average accuracy gain across our backbones, with the GRPO results in \cref{sub:grpo-ablation} showing \ourmetric{} also functions as a viable RL reward. We position \ourmetric{} as a deterministic control signal scoped to serialized table-state supervision, not as a substitute for learned or executor-grounded verifiers. The failure regimes in \cref{subsec:failure-modes-main} motivate explicit auditing during deployment but do not undermine the headline gains. Three further limitations bound the current approach.
\begin{itemize}[noitemsep,topsep=2pt,parsep=0pt]
    \item \textbf{Robustness and generalization.} \ourmetric{} depends on a fixed serialization, so formatting or delimiter changes can affect LCS scores without changing table semantics (\cref{sub:delimiter-sensitivity}). Our POMDP analysis is qualitative and does not guarantee reward sufficiency or adversarial robustness. The absence of reproducible execution-grounded and trained-verifier baselines leaves OOD robustness under extreme schema shifts open.
    \item \textbf{Global reward calibration.} \ourmetric{} compares states within the same table-query trajectory but is not calibrated across different tables, calling for a globally calibrated reward in multi-table settings.
    \item \textbf{Precision bias.} In its default precision-leaning configuration, \ourmetric{} favors compact states and may over-penalize auxiliary evidence temporarily useful for later reasoning. A recall-aware hybrid for retrieval-heavy settings is reported in \cref{sec:tabrouge-design,subsec:beta-sensitivity}.
\end{itemize}

Two near-term directions address these gaps. First, composing \ourmodel{} with structural normalization preprocessing mitigates echo-column risk. Second, combining early-state \ourmetric{} pruning with an executor-grounded final-stage verifier inherits the portability of lexical rewards and the precision of execution. More broadly, \ourmetric{} is a natural reward for RL post-training beyond GRPO, and extending tabular state supervision to prediction, imputation, and causal inference~\cite{qu2025tabicl,hollmann2025tabular} is a key future direction.




\bibliographystyle{plainnat}
\bibliography{example_paper}

\newpage
\appendix
\crefalias{section}{appendix}
\crefalias{subsection}{appendix}
\crefalias{subsubsection}{appendix}
\onecolumn
\section{Additional Related Work} \label{sec:additional-related-work}
\paragraph{Concurrent and closely related work.}
Two lines of concurrent work are especially closely related and worth distinguishing explicitly. \textbf{TabTracer}-style systems~\cite{luo2026tabtracermontecarlotree} combine MCTS-style tree search with typed operators and deterministic state hashing to enable rollback. They provide stronger execution-grounded correctness checking but require an orchestration layer that enforces typed constraints on each operation. \ourmodel{}, by contrast, operates on generic serializable states with no typed execution guards, making it applicable to any tool-using agent without infrastructure changes, at the cost of relying on lexical rather than semantic verification. \textbf{Operation-R1}-style approaches~\cite{ding2025orr1automatingmodelingsolving} leverage verifiable operation-wise rewards with RL distillation to produce single-pass pipelines. Their contribution is complementary to ours: they distill a policy through training, whereas \ourmodel{} guides an unmodified policy at test time. A direct head-to-head comparison under shared datasets would quantify the accuracy-versus-training-cost trade-off between these two paradigms and is an important direction for future work. As a practical guide: \ourmetric{} is preferable when (i) no runnable execution environment or domain-specific verifier rules are available, (ii) portability across heterogeneous agent architectures is more important than maximising semantic coverage, or (iii) training data for a learned verifier is scarce. Execution-grounded or trained verifiers (VPRM-style, TableGPT-R1)~\cite{pronesti2026outcomeverificationverifiableprocess,yang2025tablegptr1advancingtabularreasoning} are preferable when an executor and labelled intermediate states are available and the performance ceiling matters more than deployment flexibility.

\section{Theoretical Properties}
\subsection{How unoptimized LLM agent drifts harm performance?} \label{sub:llm-drift-harm}
\begin{proposition}[Variance growth with interaction length]
\label{prop:accuracy-degradation}

Let
\[
Y_S := \ell(\hat{a}(\tau_S), A)
\]
denote a nonnegative loss after $S$ interaction steps, where
$Y_S = 0$ if and only if $\hat{a}(\tau_S)=A$.
Assume that each additional step introduces stochastic variability without
systematic corrective bias, i.e.,
\[
\mathbb{E}[Y_S \mid Y_{S-1}] = Y_{S-1}.
\]

Then the variance of the loss is non-decreasing:
\[
\mathrm{Var}(Y_S) \ge \mathrm{Var}(Y_{S-1}).
\]
\end{proposition}

\begin{proof}
By the law of total variance,
\begin{align*}
\mathrm{Var}(Y_S)
&= \mathrm{Var}\!\big(\mathbb{E}[Y_S \mid Y_{S-1}]\big)
  + \mathbb{E}\!\big[\mathrm{Var}(Y_S \mid Y_{S-1})\big] \\
&= \mathrm{Var}(Y_{S-1})
  + \mathbb{E}\!\big[\mathrm{Var}(Y_S \mid Y_{S-1})\big].
\end{align*}
Since conditional variances are nonnegative,
\[
\mathbb{E}[\mathrm{Var}(Y_S \mid Y_{S-1})] \ge 0,
\]
which implies
\[
\mathrm{Var}(Y_S) \ge \mathrm{Var}(Y_{S-1}).
\]
\end{proof}

This is a direct application of the law of total variance under the martingale assumption above. We include it to formalize the intuition that, absent corrective feedback, agent loss variance accumulates with interaction length, motivating the need for stepwise reward grounding.

\subsection{How does Explicit Reward improve multi-turn LLM conversation performance}
\paragraph{1. Global State Distillation and Information Compression}
In tool-based TableQA, the observation $o_t$ is spatially limited to a small window of rows and columns, creating a significant information gap where the agent lacks visibility into the ``Global Table State'' $T_t$. Formally, the state entropy remains high given only the observation, $H(T_t \mid o_t) \gg 0$. 
The scalar reward $r_t$ functions as a \textbf{compressed representation} of the entire hidden state. It maps the high-dimensional, unobserved data in $T_t \setminus o_t$ into a single-dimensional metric relative to the query $Q$. By providing $r_t$, the environment effectively ``telegraphs'' the global utility of an operation to the agent. This allows the LLM to make informed decisions about the next action $a_{t+1}$ based on global progress rather than being misled by the local, potentially unrepresentative samples found in the snapshot $o_t$.

In our implementation, this partial observation is not an abstract assumption but a concrete interface constraint. At each step, the agent is shown only a serialized snapshot of the current dataframe rather than the full table, typically capped to the first 20 visible rows together with the currently retained columns. Only actions that explicitly expose the table state, namely the initial load, \texttt{f\_select\_column}, \texttt{f\_select\_row}, and \texttt{f\_print\_table}, refresh this observation for the model. As a result, the agent must reason from a local view that may omit globally relevant evidence, which is exactly why an external reward is useful: it summarizes whether the updated hidden table state is moving toward a more answer-sufficient configuration even when the visible snapshot remains incomplete. Appendix~\ref{subsec:observation-model} provides the full observation protocol and sensitivity discussion.

\paragraph{2. Mitigation of Stochastic Drift ($\epsilon$-Accumulation)}
The tool-based transition is modeled as $T_{t+1} = \mathcal{P}(T_t, a_t, \theta_{a_t}) + \epsilon_t$, where $\epsilon_t$ represents the ``imagination gap'' between the LLM's predicted outcome and the actual tabular transformation. In code-based generation or unsupervised tool-calling, these discrepancies are unobserved, leading to a compounding error $\sum_{l=0}^{S} \epsilon_l$. 
Without the turn-wise scalar $r_t$ to ground the agent, the reasoning trajectory $\tau$ becomes \textbf{brittle}, eventually decoupling from the ground truth. The reward signal acts as a corrective feedback loop, ensuring that the generated parameters $\theta_{a_t}$ are grounded in the physical state of the dataset.

\paragraph{3. Convergence via Density of Search Signals}
Code-based generation often suffers from sparse terminal rewards, whereas our framework utilizes $r_t$ to create a \textbf{dense reward environment}. The trajectory reward $R(\tau) = \sum \gamma_l r_l$ transforms the search into a directed optimization problem. 
The scalar $r_t$ provides a ``gradient'' for the agent. If an action $a_t$ fails to maximize the reward, it provides an immediate signal to prune that reasoning path. This significantly reduces the variance in the terminal prediction $\hat{\alpha}$ by constraining the trajectory space $\mathcal{T}$ to paths with measurable progress, shifting the process from a stochastic linguistic walk to a directed heuristic search.

\paragraph{4. Grounding of Latent State Utility}
As noted in~\cite{wang2025needleinatable}, LLMs fail to attend to all tokens, and many atomic operations (e.g., casting types or re-indexing) result in latent changes to $T_{t+1}$ that are invisible in the snapshot $o_{t+1}$. 
In an unsupervised setting, the agent suffers from \textbf{perceptual aliasing}, where it cannot distinguish between a productive transformation and a redundant one. The scalar reward $r_t$ resolves this by capturing the latent utility of the operation, validating that the table is ``closer'' to the goal state even when the visual observation remains static. This calibration ensures that the reasoning tokens $z_t$ stay aligned with the underlying data structure.

\subsection{Detailed theoretical support for preference under \ourmetric{}} \label{sub:theoretical-support-rouge}
The discussion in this subsection formalizes the same assumptions used in the main text. The goal is not to claim universal semantic optimality, but to characterize what the metric prefers when the encoding is deterministic and lexical overlap is treated as a proxy for retained useful evidence.

\begin{proposition}[Preference for Compact Query-Aligned States]
\label{prop:rouge_optimality}
Given a query $Q$, a candidate set of intermediate table states $\mathcal{T}$ and \ourmetric{} (Equation~\ref{eq:tabrouge}), any maximizer
$$
T^\star \in \arg\max_{T\in\mathcal{T}} s(T;Q)
$$
prefers a ``just-right'' table state which preserves available query-aligned information while excluding removable redundancy under the assumptions above. A maximizer of this score cannot contain removable redundancy: if $\exists T' \subset T^\star$ such that $\Omega(T', Q) = \Omega(T^\star, Q)$, then
\begin{equation*}
    \mathcal{L}(T') < \mathcal{L}(T^\star) \implies \frac{\Omega(T', Q)}{\mathcal{L}(T')} > \frac{\Omega(T^\star, Q)}{\mathcal{L}(T^\star)}
\end{equation*}
which yields a higher score $s(T'; Q) > s(T^\star; Q)$, contradicting the maximality of $T^\star$. At the same time, it cannot omit available query-relevant information. If additional content $\Delta T$ can be added to form $T^+ = T^\star \cup \Delta T$ with overlap gain $\Delta \Omega$ and encoding increase $\Delta \mathcal{L}$, the ratio improves whenever the marginal information density exceeds the current state's density:
\begin{equation*}
    \frac{\Delta \Omega}{\Delta \mathcal{L}} > s(T^\star; Q) \implies s(T^+; Q) > s(T^\star; Q)
\end{equation*}
\end{proposition}

If there exists a substring $\Delta \subset \mathrm{Enc}(T^\star)$ such that
\[
\mathrm{LCS}(Q,\mathrm{Enc}(T^\star))
= \mathrm{LCS}\!\bigl(Q,\mathrm{Enc}(T^\star)\setminus \Delta\bigr),
\]
then removing $\Delta$ yields a new table $\tilde T$ satisfying
\begin{align*}
    r_{\ourmetric{}}(\tilde T;Q)
&= \frac{\mathrm{LCS}(Q,\mathrm{Enc}(T^\star))}
{|\mathrm{Enc}(T^\star)| - |\Delta|}\\
&> r_{\ourmetric{}}(T^\star;Q),    
\end{align*}
contradicting the optimality of $T^\star$. Hence $T^\star$ cannot contain removable redundancy. Conversely, if there exists additional content $\Gamma$ such that appending it to $T^\star$ yields a table $T^+$ with
\[
\mathrm{LCS}(Q,\mathrm{Enc}(T^+))
= \mathrm{LCS}(Q,\mathrm{Enc}(T^\star)) + \delta,
\quad \delta \ge |\Gamma|,
\]
then
\begin{align*}
    r_{\ourmetric{}}(T^+;Q)
& = \frac{\mathrm{LCS}(Q,\mathrm{Enc}(T^\star)) + \delta}
{|\mathrm{Enc}(T^\star)| + |\Gamma|}\\
& \ge r_{\ourmetric{}}(T^\star;Q),    
\end{align*}
with strict inequality when $\delta > |\Gamma|$, again contradicting the maximality of $T^\star$ under this metric. Therefore, \ourmetric{} $r_{\ourmetric{}}(T;Q)$ favors intermediate table states that preserve useful query-aligned information while excluding unnecessary content, under the lexical-overlap assumptions stated above.

\subsection{Conditional monotonicity of ROUGE under query-aligned tool steps}
\begin{proof}\label{proof:rouge-monotonic}
Let $S_k = Enc(T_k)$ denote the textual encoding of the table at step $k$ (following the specification in Section~\ref{sec:textual-encoding}), and let $Q$ denote the token sequence of the query. We model each tool step as an operator that may (i) remove query-irrelevant content or (ii) derive new query-relevant information under correct and successful execution:
\[
S_{k+1} = \Pi_k(S_k),
\]
where $\Pi_k$ either removes tokens independent of $Q$ or appends derived tokens that are functionally determined by existing table entries and relevant to $Q$. Now consider the representation length $L_k := |S_k|$ and overlap $c_k := \mathrm{LCS}(Q,S_k)$. Define the stepwise increments
\[
\Delta L_k := L_{k+1}-L_k,
\qquad
\Delta c_k := c_{k+1}-c_k.
\]
Under pruning steps, $\Delta L_k<0$ and typically $\Delta c_k= 0$ since query-relevant tokens are preserved. 
Under derivation steps, both $\Delta L_k,\Delta c_{k}>0$ due to appended tokens and overlapping content with query tokens.

To formalize the likelihood that length increase is accompanied by sufficient overlap gain, let $U_k=\{u_1,\dots,u_{\Delta L_k}\}$ denote the multiset of tokens newly appended by $\Pi_k$ when $\Delta L_k>0$. Define indicator variables
\[
Z_i := I\{u_i \text{ contributes to an LCS match with } Q\},
\]
so that the overlap gain satisfies the lower bound
\begin{equation} \label{eq:one-word-multiple-overlap}
 \Delta c_k \;\ge\; \sum_{i=1}^{\Delta L_k} Z_i   
\end{equation}
Since derivation is query-conditioned under correct execution, we assume each appended token is query-relevant with probability at least $p$:
\[
\mathbb{E}[Z_i \mid \Delta L_{k}, \mathcal{F}_k] \ge p,
\]
where $\mathcal{F}_k$ denotes the interaction history up to step $k$ (including the current table state and query) and $\Delta L_{k}$ is $\mathcal{F}_{k}$-measurable. Now consider Eq. \ref{eq:one-word-multiple-overlap}, by taking conditional expectations on $\Delta L_{k}$ and $\mathcal{F}_{k}$, we have 
\begin{align*}
    \mathbb{E}[\Delta c_{k}|\Delta L_{k}, \mathcal{F}_{k}]&\geq \mathbb{E}\left(\sum_{i=1}^{\Delta L_{k}}Z_{i}\middle|\Delta L_{k}, \mathcal{F}_{k}\right)\\
    \text{By linearity of expectation, } \\\mathbb{E}\left(\sum_{i=1}^{\Delta L_{k}}Z_{i}\middle|\Delta L_{k}, \mathcal{F}_{k}\right)&=\sum_{i=1}^{\Delta L_{k}}\mathbb{E}(Z_{i}|\Delta L_{k}, \mathcal{F}_{k})\\
    & \geq \sum_{i=1}^{\Delta L_{k}}p=p\Delta L_{k}\\
    \mathbb{E}[\Delta c_{k}|\Delta L_{k},\mathcal{F}_{k}]&\geq p\Delta L_{k} \\
    \text{By law of total expectation,}\\
    \mathbb{E}[\Delta c_{k}|\Delta L_{k}]&=\mathbb{E}[\mathbb{E}[\Delta c_{k}|\Delta L_{k},\mathcal{F}_{k}]]\geq p\Delta L_{k}
\end{align*}
Now, under the assumption of conditional independence between $\{Z_{i}\}$ given $\mathcal{F}_{k}$, we have the Hoeffding's inequality such that for any $\epsilon \in (0,p)$, 
\begin{align*}
    \mathbb{E}(\Delta c_{k}|\Delta L_{k})&\geq p\Delta L_{k} \\
    \Pr\left[\sum_{i=1}^{n}Z_{i}-E\left(\sum_{i=1}^{n}Z_{i}\right)\leq -t\right]&\leq \exp \left(-\frac{2t^{2}}{\sum_{i=1}^{n}(1)^{2}}\right)\\
    \text{By choosing }t=\epsilon n, \\
    \Pr\left[\sum_{i=1}^{n}Z_{i}\leq (p-\epsilon)n\right] &\leq \exp \left(-2\epsilon^{2}n\right) \\
    \Pr\left[\sum_{i=1}^{n}Z_{i}\leq (p-\epsilon)\Delta L_{k}\middle|\Delta L_{k}\right] &\leq \exp \left(-2\epsilon^{2}\Delta L_{k}\right)
\end{align*}
Consider the extreme events between $\{S_{n} \leq (p-\epsilon)n\}$ and $\{\Delta c_{k} \leq (p-\epsilon)n\}$, for $S_{n}\leq \Delta c_{k}$, $\{\Delta c_{k} \leq (p-\epsilon)n\}\subseteq\{S_{n} \leq (p-\epsilon)n\}$ given a fixed $(p-\epsilon)n$. Therefore, we have 
\begin{align*}
    \Pr\left[\Delta c_{k}\leq (p-\epsilon)\Delta L_{k}\middle|\Delta L_{k}\right] &\leq \Pr\left[\sum_{i=1}^{n}Z_{i}\leq (p-\epsilon)\Delta L_{k}\middle|\Delta L_{k}\right] \\
    \Pr\left[\Delta c_{k}\leq (p-\epsilon)\Delta L_{k}\middle|\Delta L_{k}\right] & \leq \exp \left(-2\epsilon^{2}\Delta L_{k}\right)
\end{align*}
Therefore, as long as $p$ is not too small, the probability that a derivation step substantially increases length without adding commensurate query-overlap decays exponentially in the number of appended tokens. Letting $\widetilde{\mathrm{ROUGE\text{-}L}}(Q,S_k) := \mathrm{LCS}(Q,S_k)/|S_k|$ denote the precision-form score at step $k$, and combining this with the condition for monotone improvement of this surrogate,
\[
\Delta c_k \ge \widetilde{\mathrm{ROUGE\text{-}L}}(Q,S_k)\cdot \Delta L_k
\Longleftrightarrow \widetilde{\mathrm{ROUGE\text{-}L}}(Q,S_{k+1}) \ge \widetilde{\mathrm{ROUGE\text{-}L}}(Q,S_k),
\]
we conclude that, under query-conditioned derivation and typical navigation pruning, the expectation of $\widetilde{\mathrm{ROUGE\text{-}L}}$ is monotonically non-decreasing as $k$ increases. We emphasize that this is a result under the stated modeling abstraction: legitimate transformations that compute new answer values not lexically present in $Q$ (e.g., a derived numeric column whose surface tokens do not appear in the query) fall outside this regime and may temporarily reduce overlap before downstream lookup steps recover it.
\end{proof}

\section{Textual Encoding Details} \label{sec:textual-encoding}
In our experiment, we utilize the exact format as reference to~\cite{borisov2023language}. Given a tabular data set of $m$ columns with feature names $f_1, f_2, \dots, f_m$ and $n$ rows of samples $\mathbf{s}_1, \dots, \mathbf{s}_n$, we let the entry $v_{i,j}, i \in \{1, \dots, n\}, j \in \{1, \dots, m\}$ represent the value of the $j$-th feature of the $i$-th data point. Taking the feature name and value into account, each sample $\mathbf{s}_i$ of the table is transformed into a textual representation $\mathbf{t}_i$ using the following subject-predicate-object transformation:

\begin{align}
    t_{i,j} &= [f_j, \text{``is''}, v_{i,j}, \text{``,''}] & \forall i \in \{1, \dots, n\}, j \in \{1, \dots, m\}, \label{eq:1} \\
    \mathbf{t}_i &= [t_{i,1}, t_{i,2}, \dots, t_{i,m}] & \forall i \in \{1, \dots, n\}, \label{eq:2}
\end{align}
where $t_{i,j}$, the textually encoded feature, is a clause with information about a single value and its corresponding feature name, and $[\cdot]$ denotes the concatenation operator. We define $Enc(T)$ in Section \ref{subsec:state-reward} as $\mathbf{t}_{i}$. We conduct sensitivity analysis on different delimiters in Section \ref{sub:delimiter-sensitivity}.

\section{Design Rationale for Lexical \ourmetric{}} \label{sec:tabrouge-design}

This section expands on the design choices summarized in the main text. We describe (i) why the reward is intentionally lexical rather than embedding-based, (ii) the lexical-vs-semantic tradeoff, (iii) the optional recall-aware hybrid, and (iv) the scope of the metric across tables.

\paragraph{Why a lexical reward.} For intermediate table control, lexical matching is preferable to embedding-based similarity for three practical reasons. First, it preserves exact numeric and symbolic strings that are often decisive in table reasoning, whereas dense semantic similarity can smooth over denotationally important differences. Second, it avoids truncating long serialized tables to the context limit of an embedding model, which would destroy state information before scoring. Third, it is deterministic and cheap enough to evaluate after every transformation, which makes it suitable for trajectory-time search control rather than only post-hoc answer checking.

\paragraph{Lexical-vs-semantic tradeoff.} \ourmetric{} is a lexical, deterministic reward rather than a semantic equivalence model. This is an intentional design tradeoff. Lexical matching is lightweight, reproducible, and training-free, making it practical for repeated reward evaluation inside multi-step trajectories. At the same time, it may under-reward semantically useful but lexically mismatched states, including paraphrases, schema renamings, implicit computations, derived evidence, and unit-converted evidence. This is especially relevant for legitimate table reasoning transformations such as computing a new column, normalizing a unit, or renaming a header into a more operational form: these steps can be necessary for solving the problem even when they temporarily reduce direct lexical overlap with $Q$. In our framework, however, \ourmetric{} is used primarily to evaluate \emph{relative improvement across steps} rather than to make an absolute yes-or-no judgment about whether a single table state is fully correct. This relative-progress role makes the reward more robust in practice: even when the signal is imperfect in isolation, it remains useful for identifying whether a transformation moves the trajectory toward a more query-aligned state. We therefore position \ourmetric{} as a compact, deterministic progress signal for pruning-centric search, not as a complete solution to semantic adequacy in table reasoning. The encoding penalizes bloat through the denominator, so duplicated text can still increase $|Enc(T)|$ and lower the score.

\paragraph{Recall-aware hybrid.} The default formulation is deliberately precision-oriented, which is useful in pruning-centric trajectories where the agent progressively narrows the table. For retrieval-heavy settings we evaluate a recall-aware hybrid
\[
r_{\text{hybrid}}=\beta\cdot r_{\ourmetric{}}+(1{-}\beta)\cdot r_{\text{rec}},\qquad r_{\text{rec}}=\frac{LCS(Q,Enc(T))}{|Q|}.
\]
We retain $\beta{=}1$ (pure precision) in all main experiments and recommend $\beta{<}1$ when recall over candidate evidence dominates the task. The full $\beta$ sweep is reported in Appendix~\ref{subsec:beta-sensitivity}.

\paragraph{Trajectory-local scope.} \ourmetric{} is conditioned on both the query and the current table state and is not intended for cross-table comparison. Instead, it serves as a trajectory-local progress signal, where relative improvements indicate that a transformation moves the table toward a more sufficient and minimal representation for answering the query.

\section{Atomic Operation Tools} \label{sec:detail-operation}
\begin{table}[H]
\centering
\caption{Summary of the list for atomic operations. Most tools are referenced from Chain-of-Table~\cite{wang2024chainoftable}, except for two additional format-related functions for Python, including string and datetime operations.}
\label{tab:tool_summary}
\resizebox{0.95\textwidth}{!}{
\begin{tabular}{l p{4.2cm} p{8.5cm}}
\toprule
\textbf{Tool Name} & \textbf{Key Arguments} & \textbf{Functionality} \\
\midrule

\texttt{print\_table} 
& N/A 
& Displays the current table snapshot for inspection without modifying the environment. \\

\texttt{f\_get\_data\_info} 
& show\_missing 
& Reports dataset shape, column names, data types, and missing-value statistics. \\

\texttt{f\_select\_column} 
& column\_names, reasoning 
& Projects the table onto a subset of relevant columns while validating column existence and logging agent reasoning. \\

\texttt{f\_select\_row} 
& indices, condition, lookup, reasoning 
& Selects rows using explicit indices or symbolic conditions; supports lookup-only mode and robust expression parsing. \\

\texttt{f\_sort\_by} 
& columns, order 
& Sorts the table lexicographically by one or more columns in ascending or descending order. \\

\texttt{f\_aggregate} 
& column/op or columns\_ops, group\_by, show\_stats\_only 
& Performs statistical aggregation (count, sum, avg, min, max, diff), supporting multi-column and grouped computation. \\

\texttt{f\_compute\_column} 
& new\_col, col\_A, op, col\_B, element\_wise, special\_op 
& Constructs new features via scalar or element-wise arithmetic, including specialized operations such as grouped time differences. \\

\texttt{f\_string\_operation} 
& columns, operation, start/end, replace, concat, new\_columns 
& Performs symbolic string transformations including casing, substring extraction, concatenation, replacement, and type conversion. \\

\texttt{f\_process\_datetime} 
& columns, operation, parse\_format, timezone, new\_columns 
& Parses, transforms, and extracts temporal attributes from datetime columns. \\

\texttt{f\_join} 
& join\_type, left\_on, right\_on, cross\_join, suffix, reasoning 
& Combines multiple tables via relational join or Cartesian product with column conflict resolution. \\

\texttt{f\_final\_answer} 
& answer 
& Submits the final answer and terminates the agent’s reasoning trajectory. \\

\bottomrule
\end{tabular}
}
\end{table}

\section{Additional Experimental Results}
\subsection{Supporting Experiments for Main Results Table~\ref{tab:main_results}} \label{sub:supporting-main-results}
\subsubsection{Empirical Validation for Using Pre-defined Tools} We provide empirical support for using pre-defined tools instead of generating code directly. By quantifying generation variance through the entropy of outcome distributions under repeated sampling, we assign each generation one of three discrete outcomes: \textit{Correct \& Executable}, \textit{Executable but Incorrect}, or \textit{Compilation Failure}. For each task, we compute the entropy of the resulting multinomial distribution across multiple independent runs, where a more positive entropy corresponds to more consistent and stable generation behavior. As shown in Table~\ref{tab:lower-variance}, both Chain-of-Table and Tree-of-Table achieve substantially more positive generation entropy compared to unconstrained code generation across all datasets. This indicates that restricting the model’s action space to structured tool operations significantly stabilizes the generation process, leading to lower outcome variability while preserving task performance.
\begin{table}[H]
\centering
\caption{We estimate generation variance using the entropy of generation outcomes. More positive entropy indicates more consistent and stable generation behavior. Tool-constrained reasoning (Chain-of-Table and Tree-of-Table) yields substantially more positive entropy across all datasets.  Values are reported as mean $\pm$ standard deviation over tasks.}
\label{tab:lower-variance}
\resizebox{0.75\columnwidth}{!}{
\begin{tabular}{c|ccc}
\toprule
\multirow{2}{*}{Dataset} & \multicolumn{3}{c}{Generation Entropy} \\
\cmidrule(lr){2-4}
& Code Generation & {Chain-of-Table (Tools)} & {Tree-of-Table (Tools)} \\
WikiTQ & -0.9316 $\pm$ 0.0472 & -0.6068 $\pm$ 0.0954 & -0.6090 $\pm$ 0.1044\\
MMQA & -0.9156 $\pm$ 0.1175 & -0.6430 $\pm$ 0.0336 & -0.6344 $\pm$ 0.0659 \\
MMTU & -0.9078 $\pm$ 0.1003 & -0.6024 $\pm$ 0.1179 & -0.6195 $\pm$ 0.0191\\
\bottomrule
\end{tabular}
}
\end{table}

\subsubsection{Correlation and Calibration of \ourmetric{}}\label{subsec:reward-validation}
To validate that \ourmetric{} tracks meaningful progress rather than superficial lexical overlap, we measure the association between terminal reward values and final answer correctness over $n{=}1{,}839$ sampled trajectory steps from WikiTQ, MMQA, and MMTU. We find a Spearman correlation of $\rho{=}0.287$ ($p{<}10^{-35}$) and a Kendall correlation of $\tau{=}0.236$ ($p{<}10^{-34}$), with both tests indicating statistical significance. This shows that higher-reward states are consistently associated with better final answers even though the reward is computed without a learned verifier. The moderate rather than strong correlation reflects the intended design as a compact progress signal over all intermediate states (including low-reward early steps). Restricting to terminal states only yields a stronger association.

We further summarize the empirical calibration of \ourmetric{} by binning all $n{=}1{,}690$ trajectory steps from WikiTQ, MMQA, and MMTU according to terminal reward. As shown in Table~\ref{tab:reward-calibration}, the probability of correctness increases broadly with higher reward ranges. Trajectories with zero reward are substantially less reliable ($P(\text{correct}){=}0.22$), whereas trajectories in the $(0.15, 0.25]$ range reach $P(\text{correct}){=}0.69$. A small non-monotone dip appears at $(0.10, 0.15]$ due to the limited sample size in that bin ($n{=}53$), consistent with wide confidence intervals $[0.27, 0.53]$ that overlap neighboring bins. This broadly positive calibration supports using \ourmetric{} for trajectory ranking and early stopping, even though the signal is moderate rather than perfectly monotone.

\begin{table}[H]
\centering
\caption{Calibration of terminal \ourmetric{} against final answer correctness over $n{=}1{,}690$ trajectory steps from WikiTQ, MMQA, and MMTU (the $n{=}1{,}839$ figure in the main text includes all intermediate steps. Here we report terminal steps with non-empty bins). Higher reward ranges correspond to broadly increasing empirical probability of correctness, supporting \ourmetric{} as a meaningful trajectory-selection signal. The non-monotone dip at $(0.10,0.15]$ reflects the small sample size ($n{=}53$) in that bin.}
\label{tab:reward-calibration}
\resizebox{0.82\columnwidth}{!}{
\begin{tabular}{lccccc}
\toprule
\ourmetric{} range & Total & Correct & $P(\mathrm{correct})$ & 95\% CI \\
\midrule
$0$             & 822 & 177 & 0.215 & [0.188, 0.244] \\
$(0, 0.02]$     & 491 & 187 & 0.381 & [0.339, 0.424] \\
$(0.02, 0.05]$  & 236 &  94 & 0.398 & [0.337, 0.462] \\
$(0.05, 0.10]$  &  75 &  40 & 0.533 & [0.421, 0.643] \\
$(0.10, 0.15]$  &  53 &  21 & 0.396 & [0.273, 0.531] \\
$(0.15, 0.25]$  &  13 &   9 & 0.692 & [0.423, 0.886] \\
\bottomrule
\end{tabular}
}
\end{table}

\subsubsection{Standard Deviation of Main Result (Table \ref{tab:main_results})}
\begin{table}[H]
\centering
\caption{Performance Standard Deviation(\%) across datasets, models, workflows, and inference settings. Column in \textcolor{brown}{yellow} indicates performance after incorporating reward feedback in \textcolor{brown}{single-turn inference}. Column in \textcolor{teal}{green} indicates performance with \textcolor{teal}{full \ourmodel{} incorporation}.}
\label{tab:main_results_sd}

\setlength{\tabcolsep}{3pt}
\renewcommand{\arraystretch}{1.15}

\resizebox{0.9\textwidth}{!}{%
\begin{tabular}{l|c|ccccccc|ccccccc}
\toprule
\multirow{1}{*}{Model}
& \textbf{Code} & \multicolumn{7}{c|}{\textbf{Chain-of-Table (Tools)}} 
& \multicolumn{7}{c}{\textbf{Tree-of-Table (Tools)}} \\
\cmidrule(lr){1-16}
 
ST Reward & \cross & \cross & \hilight{\tick} & \cross & \tick & \tick & \tick & \hilite{\tick}
 & \cross & \hilight{\tick} & \cross & \tick & \tick & \tick & \hilite{\tick}\\
 SR Reward & \cross & \cross & \hilight{\cross} & EQ & EQ & AC & AS & \hilite{RG} & \cross & \hilight{\cross} & EQ & EQ & AC & AS & \hilite{RG}\\
\midrule
\multicolumn{16}{c}{WikiTQ}\\
\midrule
 QWEN3-8B & \textbf{3.32} & \textbf{3.32} & \hilight{3.53} & 3.37 & 3.53 & 3.53 & 3.49 & \hilite{3.51} & \textbf{3.39} & \hilight{3.52} & 3.52 & 3.50 & 3.53 & 3.50 & \hilite{3.52} \\
 Ministral-3-8B & 3.39 & 3.52 & \hilight{3.53} & 3.50 & 3.50 & \textbf{3.32} & \textbf{3.32} & \hilite{\textbf{3.32}} & 3.50 & \hilight{3.54} & 3.53 & 3.52 & \textbf{3.40} & \textbf{3.40} & \hilite{3.41} \\
 GPT-4.1-nano & 3.52 & 3.48 & \hilight{3.32} & 3.40 & 3.26 & 3.04 & 2.95 & \hilite{\textbf{2.74}} & 3.52 & \hilight{3.37} & 3.49 & 3.45 & 3.30 & 3.26 & \hilite{\textbf{3.19}} \\
 GPT-oss-20B & 3.51 & 3.43 & \hilight{3.26} & 3.37 & 3.19 & 2.74 & \textbf{2.32} & \hilite{2.47} & 3.50 & \hilight{3.34} & 3.26 & 2.61 & 2.61 & \textbf{2.47} & \hilite{\textbf{2.47}} \\
 GPT-5-nano & 3.53 & 2.85 & \hilight{\textbf{2.14}} & 2.74 & 2.47 & \textbf{2.14} & \textbf{2.14} & \hilite{\textbf{2.14}} & 2.85 & \hilight{2.61} & 2.47 & 2.47 & \textbf{1.94} & \textbf{1.94} & \hilite{\textbf{1.94}} \\

\midrule
\multicolumn{16}{c}{MMQA}\\
\midrule
 QWEN3-8B
& \textbf{3.02} & \textbf{3.07} & \hilight{3.47} & 3.12 & 3.35 & 3.43 & 3.49 & \hilite{3.51} & \textbf{3.37} & \hilight{3.50} & 3.43 & 3.49 & 3.45 & 3.45 & \hilite{3.52} \\
 Ministral-3-8B
& \textbf{2.97} & \textbf{3.17} & \hilight{3.50} & \textbf{3.17} & 3.47 & 3.50 & 3.43 & \hilite{3.54} & \textbf{3.38} & \hilight{3.53} & 3.43 & 3.54 & 3.54 & 3.51 & \hilite{3.50} \\
 GPT-4.1-nano
& \textbf{2.97} & 3.32 & \hilight{3.45} & 3.54 & 3.38 & 3.41 & 3.29 & \hilite{3.32} & 3.29 & \hilight{3.49} & 3.53 & 3.50 & 3.32 & \textbf{3.29} & \hilite{\textbf{3.12}} \\
 GPT-oss-20B
& 3.35 & 3.51 & \hilight{3.17} & 3.50 & 3.12 & 3.07 & 2.91 & \hilite{\textbf{2.62}} & 3.45 & \hilight{3.07} & 3.41 & 3.02 & 2.97 & \textbf{2.62} & \hilite{2.84} \\
 GPT-5-nano
& 3.49 & 3.54 & \hilight{3.41} & 3.45 & 3.21 & 2.84 & 2.70 & \hilite{\textbf{2.62}} & 3.50 & \hilight{3.07} & 3.45 & 2.62 & 2.54 & \textbf{2.45} & \hilite{\textbf{2.45}} \\

\midrule 
\multicolumn{16}{c}{MMTU}\\
\midrule
QWEN3-8B
& \textbf{3.25} & 3.41 & \hilight{3.52} & \textbf{3.39} & 3.51 & 3.53 & 3.53 & \hilite{3.53} & \textbf{3.07} & \hilight{3.22} & 3.18 & 3.29 & 3.25 & 3.29 & \hilite{3.34} \\
 Ministral-3-8B
& 3.20 & 3.22 & \hilight{3.41} & \textbf{3.15} & 3.31 & 3.53 & 3.53 & \hilite{3.52} & \textbf{3.15} & \hilight{3.24} & 3.20 & 3.31 & 3.34 & 3.37 & \hilite{3.37} \\
 GPT-4.1-nano
& 3.54 & 3.53 & \hilight{3.25} & 3.41 & 3.15 & 3.02 & \textbf{2.86} & \hilite{\textbf{2.86}} & 3.53 & \hilight{3.37} & 3.49 & 3.22 & 3.06 & \textbf{2.88} & \hilite{2.92} \\
 GPT-oss-20B
& 3.52 & 3.53 & \hilight{3.25} & 3.53 & 3.43 & 3.14 & 3.06 & \hilite{\textbf{2.97}} & 3.51 & \hilight{3.45} & 3.53 & 3.25 & 2.97 & 3.06 & \hilite{\textbf{2.74}} \\
 GPT-5-nano
& 3.31 & 2.81 & \hilight{2.46} & 2.81 & 2.29 & 2.20 & \textbf{2.10} & \hilite{\textbf{2.10}} & 3.52 & \hilight{2.81} & 3.53 & 2.74 & 2.68 & 2.68 & \hilite{\textbf{2.47}} \\

\bottomrule
\end{tabular}}
\end{table}

\subsubsection{BLEU and ROUGE Score of Main Result (Table \ref{tab:main_results})} \label{sub:blue-rouge}
\begin{table}[H]
\centering
\caption{BLEU and ROUGE-L performance (\%) across datasets, models, workflows, and inference settings. Although \ourmodel{} consistently improves reasoning quality, the lack of strict output formatting control sometimes causes the LLM to produce full sentences rather than exact answer phrases, which artificially lowers both metrics despite the answers being semantically correct.}
\label{tab:blue-rouge}

\setlength{\tabcolsep}{3pt}
\renewcommand{\arraystretch}{1.15}

\resizebox{0.9\textwidth}{!}{%
\begin{tabular}{l|c|ccccccc|ccccccc}
\toprule
\multirow{1}{*}{Model}
& \textbf{Code} & \multicolumn{7}{c|}{\textbf{Chain-of-Table (Tools)}} 
& \multicolumn{7}{c}{\textbf{Tree-of-Table (Tools)}} \\
\cmidrule(lr){1-16}
 
ST Reward & \cross & \cross & \hilight{\tick} & \cross & \tick & \tick & \tick & \hilite{\tick}
 & \cross & \hilight{\tick} & \cross & \tick & \tick & \tick & \hilite{\tick}\\
 SR Reward & \cross & \cross & \hilight{\cross} & EQ & EQ & AC & AS & \hilite{RG} & \cross & \hilight{\cross} & EQ & EQ & AC & AS & \hilite{RG}\\
\midrule
\multicolumn{16}{c}{Metric: \textbf{BLEU}}\\
\midrule
\multicolumn{16}{c}{WikiTQ}\\
\midrule
 QWEN3-8B & 28.74 & 21.98 & \hilight{29.70} & 14.14 & 22.44 & 21.98 & \textbf{30.00} & \hilite{27.77} & \textbf{33.55} & \hilight{31.43} & 20.91 & 23.07 & 25.57 & 25.47 & \hilite{26.29}\\
 Ministral-3-8B & 24.81 & 15.10 & \hilight{21.52} & 9.07 & 17.15 & 15.10 & 18.42 & \hilite{\textbf{27.98}} & 25.61 & \hilight{\textbf{24.85}} & 14.07 & 19.89 & 18.88 & 22.29 & \hilite{21.82} \\
 GPT-4.1-nano & 36.74 & 13.38 & \hilight{4.65} & 14.53 & 15.26 & 46.26 & 50.26 & \hilite{\textbf{59.44}} & 56.33 & \hilight{\textbf{60.49}} & 10.39 & 29.12 & 40.99 & 51.45 & \hilite{42.60} \\
 GPT-oss-20B & \textbf{48.74} & 39.84 & \hilight{26.79} & 35.93 & 40.30 & 37.60 & 38.24 & \hilite{43.57} & 45.15 & \hilight{44.90} & 26.79 & 39.95 & 39.95 & 40.27 & \hilite{40.27} \\
 GPT-5-nano & 36.74 & 23.71 & \hilight{18.83} & 57.23 & \textbf{59.98} & 55.00 & 57.44 & \hilite{58.44} & {58.44} & \hilight{58.44} & 60.23 & 62.21 & \textbf{63.19} & \textbf{63.19} & \hilite{\textbf{63.19}}\\
 GPT-5.4 & 60.32 & 66.35 & \hilight{66.35} & 66.35 & 69.05 & \textbf{70.02} & 69.23 & \hilite{\textbf{70.02}} & 61.58 & \hilight{66.99} & 67.35 & 67.99 & \textbf{73.87} & 72.76 & \hilite{72.58} \\

\midrule
\multicolumn{16}{c}{MMQA}\\
\midrule
 QWEN3-8B
& 22.78 & 15.52 & \hilight{19.67} & 23.45 & 22.21 & 15.92 & 20.97 & \hilite{\textbf{24.12}} & 29.11 & \hilight{\textbf{29.95}} & 24.19 & 26.63 & 24.30 & 26.34 & \hilite{22.44} \\
 Ministral-3-8B
& 15.54 & 6.33 & \hilight{12.54} & \textbf{23.45} & 20.86 & 6.33 & 10.83 & \hilite{12.82} & 9.52 & \hilight{12.68} & 24.19 & \textbf{26.35} & 25.11 & 24.65 & \hilite{20.78} \\
 GPT-4.1-nano
& 16.46 & 26.91 & \hilight{\textbf{45.16}} & 23.93 & 26.82 & 33.76 & 45.16 & \hilite{43.90} & 46.43 & \hilight{45.16} & 41.20 & 45.90 & 47.80 & \textbf{50.96} & \hilite{48.43} \\
 GPT-oss-20B
& 30.96 & 31.11 & \hilight{\textbf{49.80}} & 34.65 & 32.86 & 31.11 & 48.54 & \hilite{49.58} & 52.66 & \hilight{\textbf{53.97}} & 34.65 & 40.18 & 36.70 & 37.97 & \hilite{40.10} \\
 GPT-5-nano
& 35.44 & 37.46 & \hilight{45.41} & 14.30 & 49.20 & 42.21 & 48.65 & \hilite{\textbf{54.61}} & 56.47 & \hilight{56.40} & 33.50 & 58.78 & 61.25 & 61.48 & \hilite{\textbf{61.67}} \\
 GPT-5.4
& 54.71 & 63.01 & \hilight{65.45} & 63.01 & 65.45 & 66.67 & \textbf{69.11} & \hilite{\textbf{69.11}} & 58.91 & \hilight{61.35} & 56.06 & 61.75 & 61.75 & 61.75 & \hilite{\textbf{62.97}} \\

\midrule
\multicolumn{16}{c}{MMTU}\\
\midrule
QWEN3-8B
& 11.96 & 5.43 & \hilight{8.27} & \textbf{13.41} & 12.06 & 3.94 & 6.50 & \hilite{9.41} & 8.67 & \hilight{10.45} & \textbf{17.21} & 13.32 & 14.41 & 13.32 & \hilite{15.50}\\
 Ministral-3-8B
& \textbf{14.13} & 5.92 & \hilight{8.82} & 4.57 & 7.56 & 7.01 & 5.78 & \hilite{8.27} & 11.68 & \hilight{8.82} & 8.70 & 8.64 & 11.06 & 9.33 & \hilite{10.93} \\
 GPT-4.1-nano
& 17.39 & 13.99 & \hilight{16.70} & 13.04 & 4.27 & 16.70 & 16.70 & \hilite{\textbf{19.42}} & 19.97 & \hilight{19.97} & 8.67 & 17.39 & 19.02 & 17.93 & \hilite{\textbf{20.65}} \\
 GPT-oss-20B
& \textbf{21.74} & 17.25 & \hilight{18.72} & 15.39 & 16.79 & 17.88 & 16.49 & \hilite{17.10} & 17.63 & \hilight{17.57} & 16.30 & 16.26 & 18.66 & 17.57 & \hilite{17.75}\\
 GPT-5-nano
& 17.61 & 20.42 & \hilight{20.39} & 17.26 & 18.20 & 19.88 & 21.47 & \hilite{\textbf{22.83}} & 21.47 & \hilight{21.47} & 18.06 & 19.56 & \textbf{22.06} & 19.83 & \hilite{21.19}\\
 GPT-5.4
& 25.53 & 26.02 & \hilight{\textbf{28.42}} & 26.02 & 27.89 & \textbf{28.42} & 27.89 & \hilite{\textbf{28.42}} & 23.76 & \hilight{25.71} & 24.82 & 26.77 & 26.77 & \textbf{27.83} & \hilite{\textbf{27.83}}\\
\midrule
\multicolumn{16}{c}{Metric: \textbf{ROUGE-L}}\\
\midrule
\multicolumn{16}{c}{WikiTQ}\\
\midrule
 QWEN3-8B & 30.03 & 26.33 & \hilight{34.73} & 21.02 & 25.13 & 26.33 & {33.50} & \hilite{\textbf{35.83}} & \textbf{38.73} & \hilight{37.39} & 26.13 & 28.41 & 32.07 & 30.94 & \hilite{35.67}\\
 Ministral-3-8B & 26.18 & 18.96 & \hilight{25.57} & 16.60 & 24.57 & 18.52 & 21.49 & \hilite{\textbf{31.81}} & \textbf{31.39} & \hilight{29.97} & 22.10 & 26.91 & 28.08 & 30.24 & \hilite{29.59} \\
 GPT-4.1-nano & 37.94 & 17.25 & \hilight{8.67} & 18.08 & 23.65 & 51.60 & 55.60 & \hilite{\textbf{63.83}} & 61.49 & \hilight{\textbf{65.82}} & 18.45 & 37.88 & 46.21 & 56.50 & \hilite{46.22} \\
 GPT-oss-20B & \textbf{50.08} & 42.43 & \hilight{30.55} & 41.55 & 45.92 & 41.26 & 43.37 & \hilite{49.16} & \textbf{51.54} & \hilight{51.10} & 30.55 & 43.82 & 43.82 & 44.93 & \hilite{44.93} \\
 GPT-5-nano & 38.75 & 30.36 & \hilight{27.56} & 63.62 & 66.36 & 60.13 & 63.43 & \hilite{\textbf{64.76}} & 64.76 & \hilight{64.76} & 66.95 & 68.72 & 70.63 & 70.63 & \hilite{\textbf{70.63}} \\
 GPT-5.4 & 64.08 & 73.33 & \hilight{73.33} & 73.33 & 75.99 & \textbf{77.27} & 76.63 & \hilite{\textbf{77.27}} & 69.36 & \hilight{73.01} & 74.71 & 75.67 & \textbf{80.16} & 79.00 & \hilite{79.26} \\

\midrule
\multicolumn{16}{c}{MMQA}\\
\midrule
 QWEN3-8B
& 23.10 & 17.03 & \hilight{25.13} & \textbf{29.85} & {27.66} & 17.60 & 23.23 & \hilite{26.10} & 30.98 & \hilight{\textbf{32.78}} & 31.40 & 31.02 & 28.64  & 25.36 & \hilite{30.58} \\
 Ministral-3-8B
& 16.69 & 8.62 & \hilight{17.88} & \textbf{29.85} & 26.32 & 9.04 & 16.45 & \hilite{18.51} & 13.43 & \hilight{18.64} & 31.40 & \textbf{31.43} & 31.26 & 27.53 & \hilite{30.78}\\
 GPT-4.1-nano
& 17.50 & 29.79 & \hilight{\textbf{49.52}} & 28.92 & 34.90 & 36.89 & 49.52 & \hilite{48.54} & 50.79 & \hilight{49.52} & 45.29 & 50.17 & 52.91 & 53.97 & \hilite{\textbf{56.50}} \\
 GPT-oss-20B
& 32.46 & 34.20 & \hilight{53.46} & 36.78 & 39.63 & 34.20 & 52.19 & \hilite{\textbf{53.68}} & 56.26 & \hilight{\textbf{58.00}} & 36.78 & 43.89 & 40.83 & 43.27 & \hilite{41.86} \\
 GPT-5-nano
& 37.34 & 40.87 & \hilight{49.45} & 25.19 & 56.44 & 45.86 & 52.90 & \hilite{\textbf{59.33}} & 61.12 & \hilight{61.12} & 39.15 & 64.77 & 66.56 & 66.70 & \hilite{\textbf{67.01}}\\
 GPT-5.4
& 56.68 & 68.68 & \hilight{70.81} & 68.68 & 70.81 & 72.03 & 74.11 & \hilite{\textbf{74.43}} & 65.05 & \hilight{67.32} & 61.56 & 68.65 & 68.42 & 68.27 & \hilite{\textbf{69.87}}\\

\midrule
\multicolumn{16}{c}{MMTU}\\
\midrule
QWEN3-8B
& \textbf{15.66} & 9.71 & \hilight{12.42} & 14.33 & 12.88 & 8.04 & 10.97 & \hilite{13.60} & 12.78 & \hilight{14.59} & 18.31 & 14.51 & 15.60 & 14.52 & \hilite{\textbf{16.68}}\\
 Ministral-3-8B
& \textbf{16.02} & 8.14 & \hilight{11.72} & 7.16 & 9.69 & 9.23 & 8.14 & \hilite{12.98} & 15.53 & \hilight{13.44} & 10.70 & 10.23 & 12.9 & 11.46 & \hilite{13.33} \\
 GPT-4.1-nano
& 20.11 & 18.49 & \hilight{21.85} & 15.73 & 9.78 & 21.52 & 21.85 & \hilite{\textbf{24.75}} & 25.11 & \hilight{25.11} & 12.27 & 21.96 & 24.11 & 23.23 & \hilite{\textbf{25.56}} \\
 GPT-oss-20B
& \textbf{25.44} & 20.17 & \hilight{23.14} & 17.40 & 20.63 & 21.04 & 21.06 & \hilite{21.62} & 22.65 & \hilight{22.54} & 18.79 & 21.08 & 24.20 & 22.95 & \hilite{22.91}\\
 GPT-5-nano
& 21.65 & 25.51 & \hilight{26.75} & 19.58 & 24.07 & 25.15 & 27.84 & \hilite{\textbf{29.92}} & \textbf{28.04} & \hilight{\textbf{28.04}} & 20.46 & 25.09 & 26.90 & 25.72 & \hilite{26.18}\\
 GPT-5.4
& 29.54 & 30.95 & \hilight{\textbf{33.89}} & 30.95 & 33.18 & \textbf{33.89} & 33.18 & \hilite{\textbf{33.89}} & 28.44 & \hilight{31.34} & 29.89 & 32.40 & 32.21 & 33.27 & \hilite{\textbf{33.47}}\\
\bottomrule
\end{tabular}}
\end{table}

\subsection{LLM Drifting} \label{sub:llm-drifting}
\begin{figure}[H]
    \centering
    \includegraphics[width=0.9\linewidth]{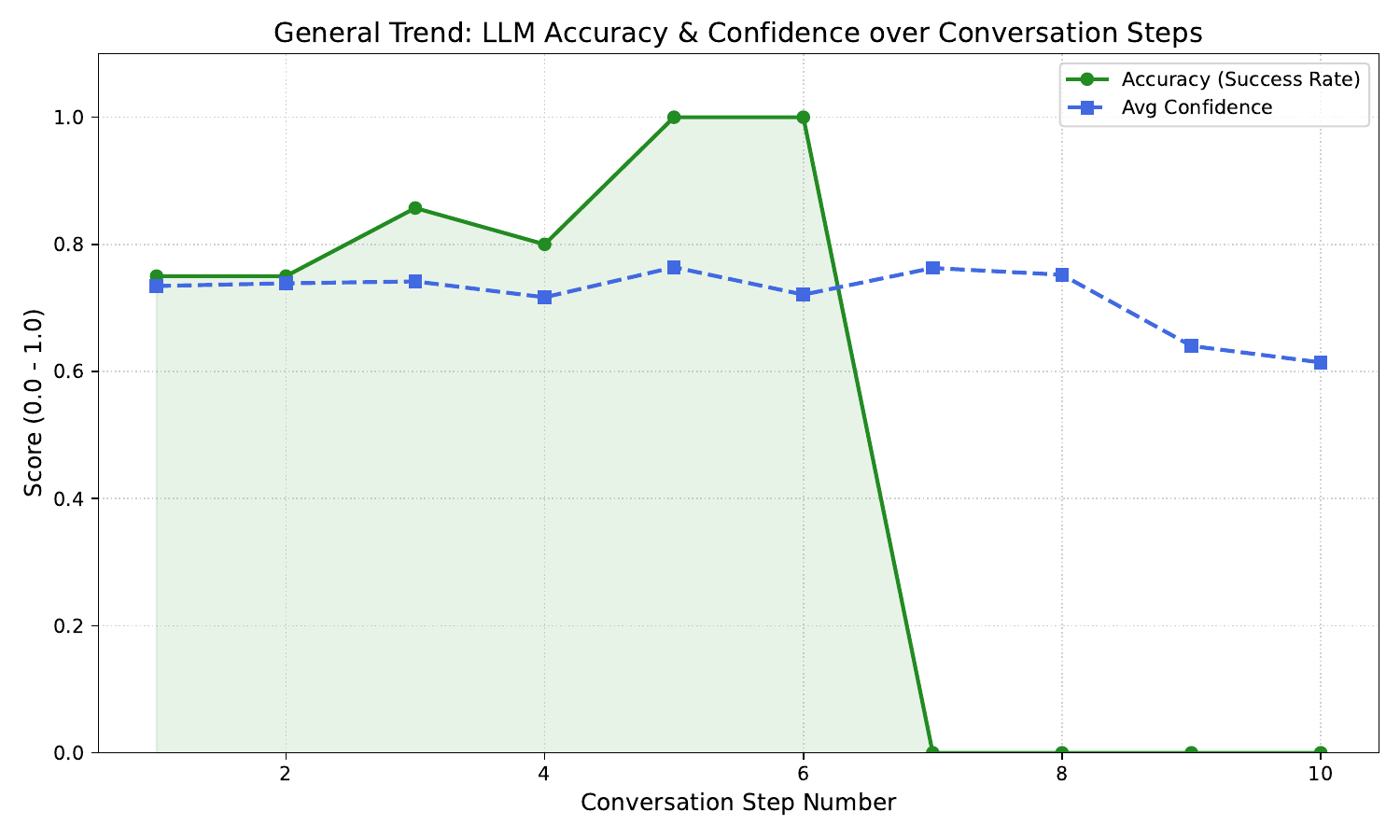}
    \caption{Unoptimized tabular LLM agent shows tendency to drift despite having already discovered the correct answer within the intermediate steps, and such drift would lead to the generation of an incorrect answer at final step. }
    \label{fig:drift_analysis}
\end{figure}

We measure LLM drift via the drifting on the \ourmetric{}, as we align all trajectories to the ending point to measure how each step performs before reaching conclusion. For $x=0$ representing the moment the model stops generating, while both lines show an upward trend as they approach termination, we observe that incorporating early stopping leads to correct selection of higher ROUGE stopping points. Moreover, in the second plot, we can see that the model is actually approaching the conclusion in a fair early stage with about 4 steps to achieve over 80\% of the work needed to be done, while the rest is less meaningful drifting which does not contribute to the answer accuracy. By eliminating these drifting rounds, we can both improve accuracy and remove the noisy extra tokens that cost additional computational burden. 

\begin{figure}[H]
    \centering
    \includegraphics[width=0.9\linewidth]{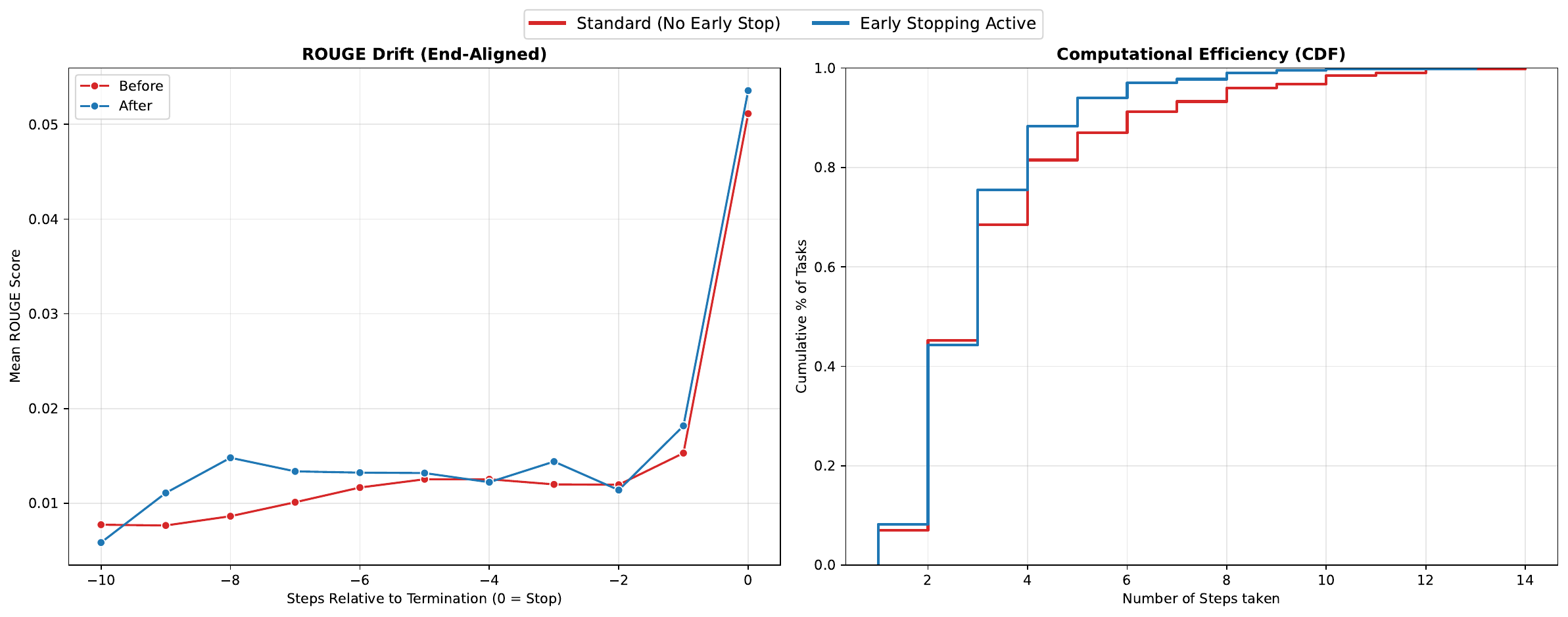}
    \caption{The first plot depicts LLM's drifting before reaching final conclusion, while the second plot illustrates the LLM reaches most of the work in the first four steps, while the subsequent drifting does not contribute to the performance significantly. }
    \label{fig:llm-drifting}
\end{figure}

\subsection{Case Study: VLM-Agent Baseline and \ourmetric{} Failure Modes} \label{subsec:case-study-vlm-appendix}
We unpack two results that complement the reward-family comparison in Section~\ref{subsec:motivating-study} (Table~\ref{tab:vs-vagen}): the rewarded VLM-agent baseline and the failure modes of \ourmetric{} itself.

\paragraph{VLM-agent baseline.} The rewarded VLM agent operates on rendered table images rather than serialized text. \ourmodel{} integrated into the tabular agent consistently outperforms both the rewarded VLM agent and the original tabular agent across all three datasets under a Best-of-5 single-shot setting with GPT-4.1-nano~\cite{openai_gpt41_2025}. Interestingly, although VLM agent lacks explicit table operations, it is slightly above the baseline on the reasoning-oriented benchmarks WTQ and MMQA, suggesting that world model reasoning can improve multi-step inference when table perception is not the primary bottleneck. However, on MMTU, which emphasizes exact value lookup over large and irregular tables, VLM agent degrades substantially and underperforms even the baseline. This contrast indicates that while world model reasoning benefits high-level reasoning, it struggles with precise tabular perception, for which tool-based operations provide a clear advantage.

\paragraph{\ourmetric{} failure modes.} \ourmetric{} is not failure-free. Its lexical nature can be misled in boundary cases where query tokens overlap with semantically irrelevant fields, such as a renamed column that happens to share a salient query word, or a distractor column whose header matches the query better than the truly useful column. In practice these collisions are relatively rare because the encoding couples headers with their cell values and because trajectory selection depends on relative reward changes rather than a single isolated score. We therefore supplement the qualitative discussion with a quantitative stress test (Table~\ref{tab:stress_test}) covering paraphrases (0.0\% inversion), unit conversions (0.0\%), renamed columns (0.8\%), and derived-evidence columns (0.0\%), documenting exactly where lexical matching succeeds and where it may introduce rare ranking errors. Likewise, because the reward depends on deterministic serialization, stronger main-text sensitivity reporting for ordering and tokenization choices remains an important direction beyond the delimiter analysis reported here. More broadly, the reward-family comparison in Section~\ref{subsec:motivating-study} should be read as evidence that \ourmetric{} is a useful control signal inside the present pipeline, not as proof that lexical rewards are universally superior to learned or execution-grounded verifiers. One should also be cautious about reward-hacking behaviors that manipulate the surface form of the serialization (for example, by echoing query tokens in low-information columns), even if such behaviors are partially discouraged by the denominator penalty in Equation~\ref{eq:tabrouge}.

\paragraph{Reward-Injection Ablation.} Comparing majority vote (no reward), ranking-only (\ourmetric{} for trajectory selection only), and full \ourmodel{} (\ourmetric{} both injected per-step and used for ranking) shows that ranking-only alone provides inconsistent gains over majority vote (negative on 6 of 8 conditions), while full \ourmodel{} matches or exceeds rank-only in 6 of 8 cases (and matches or exceeds majority vote in 4 of 8), indicating that per-step prompt injection contributes beyond pure trajectory selection. Full per-condition results are in Table~\ref{tab:reward-injection-detail}.

\subsection{Reward-Injection Ablation Detail}
\label{subsec:reward-injection-detail}

\begin{table}[H]
\centering
\caption{Full per-condition results for the reward-injection ablation (Section~\ref{sub:initial-results}). ``Majority'' = no reward, equal-weight selection. ``Rank-only'' = \ourmetric{} used only for trajectory ranking, not injected per step. ``Full \ourmodel{}'' = \ourmetric{} injected per step and used for ranking. Gain columns show absolute accuracy change (\%). Negative gains indicate conditions where reward injection hurts relative to the baseline.}
\label{tab:reward-injection-detail}
\resizebox{0.95\columnwidth}{!}{
\begin{tabular}{ll|ccc|cc}
\toprule
Dataset & Model & Majority & Rank-only & Full \ourmodel{} & $\Delta$(Rank vs.~Maj.) & $\Delta$(Full vs.~Rank) \\
\midrule
WikiTQ & QWEN3-8B        & 30.0 & 28.0 & 30.0 & $-2.0$ & $+2.0$ \\
WikiTQ & Ministral-3-8B  & 44.0 & 40.0 & 36.0 & $-4.0$ & $-4.0$ \\
MMQA   & QWEN3-8B        & 35.4 & 31.6 & 32.9 & $-3.8$ & $+1.3$ \\
MMQA   & Ministral-3-8B  & 43.0 & 41.8 & 44.3 & $-1.3$ & $+2.5$ \\
MMQA   & GPT-oss-20B     & 59.5 & 54.4 & 51.9 & $-5.1$ & $-2.5$ \\
MMTU   & QWEN3-8B        & 14.1 & 12.0 & 13.0 & $-2.2$ & $+1.1$ \\
MMTU   & Ministral-3-8B  & 18.5 & 20.6 & 22.8 & $+2.2$ & $+2.2$ \\
MMTU   & GPT-oss-20B     & 19.6 & 20.6 & 20.6 & $+1.1$ & $0.0$ \\
\bottomrule
\end{tabular}
}
\end{table}

\subsection{Echo-Column Audit}
\label{subsec:echo-column-audit}

\begin{table}[H]
\centering
\caption{Echo-column audit across datasets and backbone models. An echo column is defined as a \texttt{f\_compute\_column} call whose new column name shares at least one non-stopword token with the query. ``Echo calls'' counts such calls across all trajectories. ``Echo \%'' is the fraction of \texttt{f\_compute\_column} calls that are echo columns. ``Trajectories w/ echo'' is the fraction of questions (trajectories) that contain at least one echo call. GPT-oss-20B made zero \texttt{f\_compute\_column} calls across all datasets, so echo rates are 0\% by construction.}
\label{tab:echo-column-audit}
\resizebox{0.85\columnwidth}{!}{
\begin{tabular}{ll|rrr|r}
\toprule
Dataset & Model & Compute calls & Echo calls & Echo \% (of compute calls) & Trajectories w/ echo \\
\midrule
WikiTQ  & QWEN3-8B        &  6 &  2 & 33.3\% & 4.0\% \\
WikiTQ  & Ministral-3-8B  & 36 & 14 & 38.9\% & 28.0\% \\
WikiTQ  & GPT-oss-20B     &  0 &  0 &  0.0\% &  0.0\% \\
\midrule
MMQA    & QWEN3-8B        &  0 &  0 &  0.0\% &  0.0\% \\
MMQA    & Ministral-3-8B  &  3 &  3 & 100.0\% & 3.8\% \\
MMQA    & GPT-oss-20B     &  0 &  0 &  0.0\% &  0.0\% \\
\midrule
MMTU    & QWEN3-8B        & 17 &  3 & 17.6\% &  3.3\% \\
MMTU    & Ministral-3-8B  & 40 & 12 & 30.0\% & 13.0\% \\
MMTU    & GPT-oss-20B     &  0 &  0 &  0.0\% &  0.0\% \\
\bottomrule
\end{tabular}
}
\end{table}

\subsection{Failure Cases and Boundary Conditions}
\label{subsec:failure-cases}
Although \ourmetric{} works well as a lightweight trajectory-level reward, it remains a lexical matching signal and therefore admits edge cases where token overlap is not perfectly aligned with semantic usefulness. The first two rows below concern reward-level failure modes of \ourmetric{}, whereas the last row isolates a downstream evaluation-normalization discrepancy that is related in spirit but conceptually distinct.

\begin{table}[t]
\centering
\small
\caption{Quantitative breakdown of failure modes. The first two rows summarize \ourmetric{}-specific failure modes: semantically correct but lexically mismatched transformations, and query-matching column names that can occasionally make an incorrect trajectory appear competitive. The final row is not a reward failure of \ourmetric{} itself, but an evaluation-normalization discrepancy caused by added numerical units in otherwise correct answers.}
\label{tab:error-breakdown}
\resizebox{0.98\columnwidth}{!}{%
\begin{tabular}{
  >{\raggedright\arraybackslash}p{2.5cm}
  >{\raggedright\arraybackslash}p{7cm}
  >{\raggedright\arraybackslash}p{3.6cm}
}
\toprule
\textbf{Category} & \textbf{Quantitative evidence} & \textbf{Interpretation / Scope} \\
\midrule
Paraphrase / semantic mismatch & In the ``margin'' example, \ourmetric{} remains $0.00$ across all correct steps, including \texttt{f\_compute\_column("margin",\ldots)}, \texttt{f\_select\_column([margin])}, and \texttt{f\_select\_row(\ldots, "margin > 20")}. & A semantically correct trajectory can receive no lexical reward when the derived representation does not share query tokens. \\
\midrule
Renamed / partial-match columns & For paired correct vs. wrong trajectories with the same generated column name ``At least 20'', final F1 is $0.103$ (correct) vs. $0.113$ (wrong). In aggregate, only $2.5\%$ of wrong trajectories exceed terminal \ourmetric{} $> 0.05$, while $4.3\%$ of correct ones have terminal \ourmetric{} $\le 0$. & Query-matching column names can occasionally favor the wrong trajectory, but such collisions appear rare overall. \\
\midrule
Units / evaluation normalization & On WikiTQ evaluation, LLM-as-a-judge plus human verification reports $81.63\%$ versus $80.10\%$ for automated denotation EM when extra units are present in otherwise correct answers. & This is an answer-evaluation discrepancy rather than a trajectory-reward failure, but it shows that unit-bearing outputs can still shift reported accuracy. \\
\bottomrule
\end{tabular}}
\vspace{-2mm}
\end{table}

\paragraph{Case 1: Misleading lexical overlap from distractor columns.}
Consider a query asking for a team with the best \emph{score differential}, while the table contains both a \texttt{team} column and a \texttt{score} column. Because the query explicitly contains the words ``team'' and ``score'', an intermediate table that preserves the \texttt{team} column together with a superficially relevant but incomplete score field can receive a non-trivial \ourmetric{} value even if the actually decisive evidence lies in another derived column or a filtered subset of rows. In such cases, lexical overlap can temporarily overvalue a locally plausible but semantically incomplete table state.

In practice, this effect is moderated by the serialization scheme in Appendix~\ref{sec:textual-encoding}, which binds each cell value to its header rather than scoring isolated tokens. As a result, the reward usually improves only when the retained state preserves both the relevant field and the matching row context. Still, if several columns share overlapping query words, \ourmetric{} may rank two semantically different states too closely.

\paragraph{Case 2: Renamed or query-biased columns.}
Another boundary case arises when a column name is changed to more closely resemble the query while the underlying values remain less useful for solving the task. For example, a generic result column renamed from \texttt{Result} to \texttt{score} may obtain additional lexical overlap with the question even though it does not by itself identify the correct answer. This type of collision is most likely when multiple columns carry partially related evidence and the reward is evaluated on a single intermediate state in isolation.

Several aspects of the method reduce the practical impact of this failure mode. First, repeated rows or duplicated columns do not help because \ourmetric{} is based on a single longest common subsequence and therefore does not reward naive duplication. Second, the serialization is dominated by cell values rather than header tokens alone, so renaming one column contributes only a limited change to the overall score. Third, our agent ranks complete trajectories, meaning a misleading intermediate reward must persist through later transitions to alter the final choice. Empirically, such cases are uncommon, but they motivate future work on hybrid rewards that combine training-free reproducibility with stronger schema-aware semantics.

\paragraph{Case 3: Computed-column mismatch.}
Some useful intermediate states contain evidence that is semantically decisive but not explicitly lexicalized in the query. For example, a trajectory may construct a derived column for score differential, normalized rank, or filtered counts, even though the question only refers to the final concept and never names the exact derived representation. In such a case, the transformed table may be more useful for solving the problem but receives only a modest reward increase because the newly introduced evidence is implicit rather than lexically aligned with the query tokens.

This failure mode reflects the fact that \ourmetric{} rewards query-aligned state refinement rather than full semantic equivalence. In practice, the reward still remains helpful when the surrounding tool trajectory progressively aligns the schema and values with the question, but it may underestimate intermediate states whose value comes primarily from latent computation.

\paragraph{Case 4: Unit-conversion mismatch.}
Another edge case occurs when the table state is semantically correct but expressed in a different unit or tokenization from the query. For instance, the useful state may convert miles to kilometers, dollars to millions of dollars, or percentages to decimal ratios before the final answer is produced. Such a state can improve reasoning quality while reducing direct lexical overlap with the original query, causing \ourmetric{} to under-reward a genuinely helpful transformation.

This limitation is again moderated by the fact that \ourmetric{} is used for relative trajectory ranking over multiple steps rather than as a binary accept-or-reject criterion for a single state. When tool trajectories gradually align schema, values, and representations, the reward often remains informative enough to guide search. Nevertheless, these examples show that the method optimizes for reproducibility and efficiency rather than full semantic invariance, and thus works best when combined with transformations that progressively expose query-aligned evidence.

\paragraph{Case 5: Echo-column transient boost (concrete example).}
The following example from our MMTU trajectories directly answers the question of whether \ourmetric{} can improve while end-task accuracy decreases.

\textbf{Question}: ``As of December 31 2013, what was the ratio of equity compensation plans approved by security holders remaining to be issued to the amount to be issued upon exercise of outstanding \ldots{}'' (gold answer: 14.1).

\textbf{What happened}: QWEN3-8B (trajectory 2) used \texttt{f\_compute\_column} to create a new column named \texttt{ratio} (dividing two long-named columns). The token \texttt{ratio} appears in the query, producing a transient \ourmetric{} increase from $0.137$ to $0.150$ ($+0.013$). Immediately afterwards, the model mistakenly called \texttt{f\_retrieve\_original\_df}, resetting the dataframe to its original state, so the computed column was lost and the final answer was empty.

\textbf{Interpretation}: This is a confirmed case where an echo column briefly raised \ourmetric{} while the trajectory ultimately failed. Two observations limit its severity: (1) the boost magnitude ($+0.013$) is roughly $7\%$ of the mean correct-wrong reward gap ($\approx 0.19$, Table~\ref{tab:stress_test}), meaning it is far too small to flip a ranking in a typical question, (2) the trajectory failure was caused by the subsequent \texttt{f\_retrieve\_original\_df} reset, not by the echo column itself. A second example (WikiTQ/Ministral-3-8B: ``which country has produced the most drivers?'', gold: France, model answer: Italy) shows an echo column (\texttt{Country}) being created without any \ourmetric{} boost at all ($\Delta{=}0$), because the denominator growth absorbed the LCS gain. Across all trajectories audited in Table~\ref{tab:echo-column-audit}, only one question yielded a measurable transient boost from an echo column, consistent with the stress-test inversion rates remaining at or near $0\%$.

\subsection{Stress Tests on Reward Robustness and Encoding Sensitivity}
\label{subsec:stress-tests}

%
%

\begin{table}[t]
\centering
\caption{%
  \textbf{(a) Failure-mode stress-test.}
  Quantitative breakdown across four semantics-preserving perturbation
  categories on 463 real trajectories (WikiTQ, MMQA, MMTU; models
  gpt-4.1-nano, gpt-oss-20b, gpt-5.4).
  For each example a \emph{correct} state (maximal query–table token overlap)
  and a \emph{wrong} state (minimal overlap) are constructed automatically;
  all perturbations preserve the gold answer.
  \textbf{Rank.\ acc.} = fraction of examples where \textsc{TabRouge}
  correctly ranks the correct state above the wrong one,
  before and after perturbation.
  \textbf{Inv.} = inversion rate: fraction of originally-correct rankings
  that are flipped (the reward now \emph{misleads} the agent).
  \textbf{Mean gap} = mean of (correct-state reward $-$ wrong-state reward)
  before perturbation; this is the baseline margin the perturbation must
  overcome to cause an inversion.
  \textbf{Mean~$\Delta$} / \textbf{Std~$\Delta$} = mean and standard
  deviation of the reward change on the correct state after perturbation.
  \textbf{\% drop} = fraction where the correct-state reward falls by more
  than 0.02 / 0.05 absolute points.%
}
\label{tab:stress_test}
\small
\setlength{\tabcolsep}{4pt}
\resizebox{0.98\columnwidth}{!}{%
\begin{tabular}{lr rr r r rr rr}
\toprule
 & & \multicolumn{2}{c}{Rank.\ acc.\ (\%)} & & & \multicolumn{2}{c}{Reward $\Delta$ (correct)} & \multicolumn{2}{c}{\% reward drop} \\
\cmidrule(lr){3-4} \cmidrule(lr){7-8} \cmidrule(lr){9-10}
Category & $N$ & Before & After & Inv.\ (\%) & Mean gap & Mean & Std & ${>}0.02$ & ${>}0.05$ \\
\midrule
Paraphrase / semantic mismatch  & 410 & 89.3 & 89.5 & 0.0 & 0.195 & $+$0.004 & 0.025 & 10.5 &  3.2 \\
Renamed / partial-match columns & 275 & 93.8 & 94.2 & 0.8 & 0.188 & $-$0.018 & 0.067 & 21.8 & 14.2 \\
Unit conversion                 & 395 & 88.6 & 88.6 & 0.0 & 0.195 & $-$0.015 & 0.022 & 28.4 &  6.1 \\
Derived-evidence columns        & 258 & 88.0 & 88.0 & 0.0 & 0.211 & $+$0.028 & 0.026 &  0.0 &  0.0 \\
\bottomrule
\end{tabular}
}
\vspace{8pt}

\caption*{%
  \textbf{(b) Encoding sensitivity} across four axes (463 examples,
  20 shuffles / 2 variants each).
  All perturbations are semantics-preserving.
  \textbf{Row shuffle}: random permutation of row order.
  \textbf{Col shuffle}: random permutation of column order.
  \textbf{Delimiter}: swap \texttt{``\textless col\textgreater\ is \textless val\textgreater,''}
    to \texttt{``\textless col\textgreater:\textless val\textgreater;''} and
    \texttt{``\textless col\textgreater=\textless val\textgreater|''}.
  \textbf{Header case}: lower-case / upper-case all column headers.
  Mean $|\Delta|$ and max $|\Delta|$ are averaged across examples.%
}
\small
\setlength{\tabcolsep}{5pt}
\resizebox{0.98\columnwidth}{!}{%
\begin{tabular}{l rr r rr r rr r rr}
\toprule
 & \multicolumn{2}{c}{Row shuffle} & \phantom{x} & \multicolumn{2}{c}{Col shuffle} & \phantom{x} & \multicolumn{2}{c}{Delimiter} & \phantom{x} & \multicolumn{2}{c}{Header case} \\
\cmidrule(lr){2-3} \cmidrule(lr){5-6} \cmidrule(lr){8-9} \cmidrule(lr){11-12}
$N$ & Mean $|\Delta|$ & Max $|\Delta|$ && Mean $|\Delta|$ & Max $|\Delta|$ && Mean $|\Delta|$ & Max $|\Delta|$ && Mean $|\Delta|$ & Max $|\Delta|$ \\
\midrule
463 & 0.0074 & 0.0144 && 0.0049 & 0.0095 && 0.0198 & 0.0219 && 0.000 & 0.000 \\
\bottomrule
\end{tabular}
}
\end{table}

\paragraph{Analysis.}
Tables~\ref{tab:stress_test}a--b provide a quantitative characterisation of
the four failure modes and four encoding-sensitivity axes discussed in
Cases~1--4 above.

\textbf{Perturbation protocol.}
For each example we construct a \emph{correct} intermediate state---the
sub-table whose column headers and cell values have maximal token overlap
with the query---and a \emph{wrong} state with minimal overlap.
All perturbations are applied automatically and are semantics-preserving:
synonym substitutions do not alter the question's denotation; column
renames, unit suffixes, and derived columns leave the set of gold answers
unchanged.
The key metric is the \emph{inversion rate}: the fraction of
originally-correct rankings that the perturbation flips, i.e.\ cases
where the reward actively misleads the agent.

\textbf{Inversion rates remain negligible.}
Across all four categories, inversion rates are at most 0.8\%
(column rename), and zero for the remaining three categories.
This is because the \emph{mean reward gap} between correct and wrong states
before perturbation is large ($0.19$--$0.21$, Table~\ref{tab:stress_test}a),
so a perturbation must overcome a substantial margin to flip the ranking.
The largest mean perturbation effect is $-0.018$ (column rename), roughly
one tenth of the baseline gap---an order of magnitude too small to
systematically cause inversions.

\textbf{Ranking accuracy is stable before and after perturbation.}
Across all four categories, ranking accuracy changes by at most
$0.4$ percentage points after perturbation.
Although the reward \emph{magnitude} on the correct state shifts
(unit conversion: $-0.015$; column rename: $-0.018$), the wrong state
is affected similarly, preserving the ordering.

\textbf{Column rename is the most variable failure mode.}
The standard deviation of reward change is highest for column rename
($\sigma{=}0.067$), and 14.2\% of examples suffer a drop exceeding 0.05
absolute points.
This is consistent with the lexical nature of \textsc{TabRouge}: when a
header is replaced by a synonym absent from the query, token overlap drops
discontinuously.
Nevertheless, because the wrong state is equally affected, the ranking
signal is preserved.

\textbf{Encoding sensitivity is small across all four axes
(Table~\ref{tab:stress_test}b).}
Row-order shuffles produce a mean $|\Delta|{=}0.0074$ and a mean
worst-case of $0.0144$.
Column-order shuffles are even smaller ($0.0049$ / $0.0095$).
Delimiter changes (the largest perturbation axis) produce a mean
$|\Delta|{=}0.020$, still roughly one tenth of the mean correct--wrong
reward gap of $0.20$.
Header case changes produce \emph{exactly zero} perturbation because
\textsc{TabRouge} lowercases both the query and the encoding before
computing the LCS, making it invariant to capitalisation by design.
Across all axes, the encoding fluctuations are an order of magnitude
smaller than the typical reward gap, confirming they cannot systematically
affect ranking decisions during search.

\subsection{$\beta$ Sensitivity Sweep} \label{subsec:beta-sensitivity}
We sweep the recall-aware hybrid weight $\beta\in\{0, 0.25, 0.5, 0.75, 1\}$ across three datasets and three backbones (Table~\ref{tab:beta-sensitivity}). Recall the convention from Appendix~\ref{sec:tabrouge-design}: $r_{\text{hybrid}}=\beta\cdot r_{\ourmetric{}}+(1-\beta)\cdot r_{\text{rec}}$, so $\beta{=}1$ recovers pure precision (the default \ourmetric{}) and $\beta{=}0$ is pure recall.

\begin{table}[H]
\centering
\caption{$\beta$ sensitivity sweep. Bold marks the best $\beta$ in each row. ``Spread'' is $\max-\min$ across the five $\beta$ values. MMTU uniformly favors precision ($\beta{=}1$). The larger gpt-oss-20B is nearly $\beta$-insensitive (plateau at $\beta{\geq}0.25$). Smaller models on WikiTQ/MMQA show dataset-specific preferences with modest spread on most cells.}
\label{tab:beta-sensitivity}
\small
\setlength{\tabcolsep}{4pt}
\begin{tabular}{ll|rrrrr|cr}
\toprule
Dataset & Model & $\beta{=}0$ & $0.25$ & $0.5$ & $0.75$ & $1.0$ & Best $\beta$ & Spread \\
\midrule
WikiTQ & qwen3-8B          & 26.0 & \textbf{28.0} & \textbf{28.0} & \textbf{28.0} & \textbf{28.0} & ${\geq}0.25$ & 2.0 \\
WikiTQ & ministral-3-8B    & \textbf{44.0} & 42.0 & 42.0 & 42.0 & 42.0 & $0$ & 2.0 \\
\midrule
MMQA   & qwen3-8B          & \textbf{31.7} & \textbf{31.7} & \textbf{31.7} & 30.4 & 27.9 & ${\leq}0.5$ & 3.8 \\
MMQA   & ministral-3-8B    & 38.0 & 38.0 & 38.0 & 39.2 & \textbf{41.8} & $1$ & 3.8 \\
MMQA   & gpt-oss-20B       & 53.2 & \textbf{54.4} & \textbf{54.4} & \textbf{54.4} & \textbf{54.4} & ${\geq}0.25$ & 1.3 \\
\midrule
MMTU   & qwen3-8B          & 12.0 & 12.0 & 12.0 & 12.0 & \textbf{13.0} & $1$ & 1.1 \\
MMTU   & ministral-3-8B    & 14.1 & 14.1 & 14.1 & 17.4 & \textbf{21.7} & $1$ & 7.6 \\
MMTU   & gpt-oss-20B       & 18.5 & \textbf{20.7} & \textbf{20.7} & 19.6 & \textbf{20.7} & ${\geq}0.25$ & 2.2 \\
\bottomrule
\end{tabular}
\end{table}

\paragraph{Reading the sweep.} Three observations stand out. First, MMTU consistently rewards $\beta{=}1$ across all three backbones, with the single largest gain on the whole sweep ($+7.6$ pts on ministral-3-8B), consistent with MMTU's lookup-heavy character where retaining auxiliary content dilutes the precision signal. Second, the largest model gpt-oss-20B is essentially $\beta$-insensitive, plateauing at $\beta{\geq}0.25$ across all datasets, suggesting that strong calibration absorbs differences between the precision and recall components. Third, smaller models on WikiTQ/MMQA show dataset-specific rather than monotone preferences: qwen3-8B prefers $\beta{=}1$ on WikiTQ and MMTU but lower $\beta$ on MMQA, while ministral-3-8B prefers $\beta{=}0$ on WikiTQ but $\beta{=}1$ on MMQA and MMTU. Spreads are modest ($\leq 4$ pts on 7 of 8 cells), so we caution against strong cross-condition claims and retain $\beta{=}1$ as the default precision-leaning operating point, with the recall-aware hybrid available where retrieval cells benefit from it.

\subsection{Delimiter Sensitivity Analysis} \label{sub:delimiter-sensitivity}
We further analyze the sensitivity of the textual encoding to delimiter choices by measuring paired per-question differences relative to the baseline. Concretely, we vary the relation token (\texttt{is} → \texttt{:}, \texttt{=}), the separator token (\texttt{,} → \texttt{;} , \texttt{and}), and the row delimiter (\texttt{\textbackslash n} → \texttt{|}, \texttt{.}). As shown in Figure~\ref{fig:delimiter-sensitivity}, these substitutions yield only marginal changes in performance. All effects remain small within about 2 - 3 questions among the sampled 50 test cases and their confidence intervals overlap zero, indicating no statistically significant impact. In particular, the relation token exhibits near-zero difference, while separator and row delimiters show slight but inconsistent fluctuations. These results suggest that our method is largely robust to reasonable delimiter variations, and that performance gains are not attributable to specific token engineering but instead arise from the reward design itself.

\begin{figure}[H]
    \centering
    \includegraphics[width=0.8\linewidth]{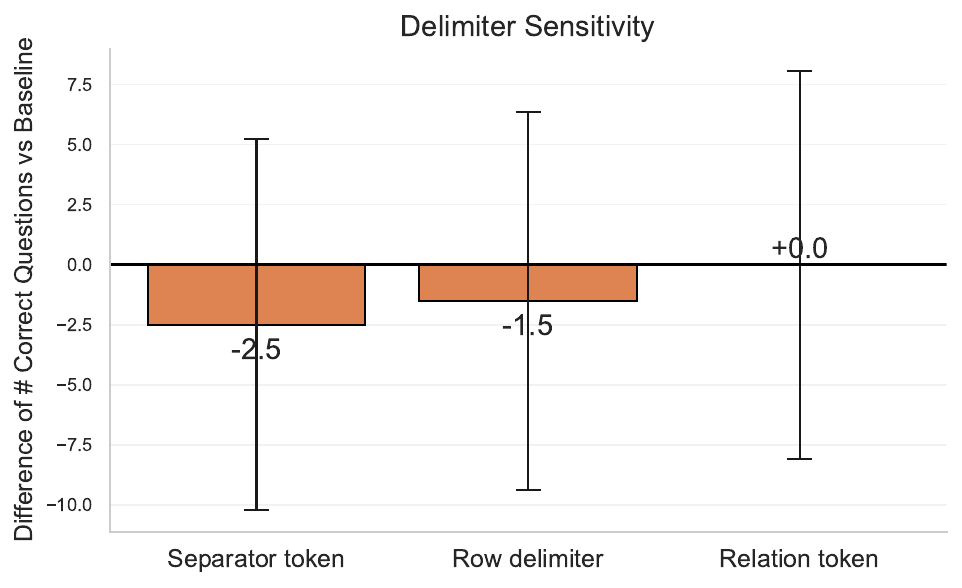}
    \caption{Delimiter sensitivity measured via paired per-question differences relative to the baseline encoding. Changes to relation, separator, or row delimiters produce only small fluctuations, with confidence intervals overlapping zero, indicating limited impact on performance.}
    \label{fig:delimiter-sensitivity}
\end{figure}

\subsection{Sensitivity Analysis on Tabular Size} \label{sec:analysis-size}
Large tables pose significant challenges for LLM-based reasoning, as models struggle to interpret multi-dimensional structures, particularly when tables are serialized as text or images~\cite{wang2025needleinatable}. We categorize tables into three groups based on token count: \emph{small} ($<$300 tokens), \emph{medium} (300 to 2,000 tokens), and \emph{large} ($>$2,000 tokens).

By enforcing partial observability, we expect overall performance to remain stable across table sizes, as the agent's decision process depends on identifying the correct operations rather than fully parsing the entire table. Consistent with this expectation, Table \ref{tab:different-size-performance} shows stable relative advantages of different reward models across table sizes.

However, a counter-intuitive phenomenon emerges: Performance of the Chain-of-Table framework \emph{improves} as table size increases, whereas Tree-of-Table exhibits the opposite trend. Because both frameworks operate under identical partial observability, this behavior cannot be directly attributed to table size alone. To understand this discrepancy, we further analyze how each framework interacts with the underlying task semantics across datasets.

\begin{table}
\centering
\caption{As tabular size increases, performance trend moves in opposite direction between Chain-of-Table and Tree-of-Table, which cannot be directly attributed to table size due to the partially observable table setting, motivating additional analysis on Table \ref{tab:chain-vs-tree}.}
\label{tab:different-size-performance}
\resizebox{\columnwidth}{!}{
\small
\setlength{\tabcolsep}{4pt}
\begin{tabular}{c c cccc cccc}
\toprule
& & \multicolumn{4}{c}{Chain} & \multicolumn{4}{c}{Tree} \\
Size & & EQ & AC & AS & RG & EQ & AC & AS & RG \\
\midrule
S & ($<$300)
& 57.58 & 62.22 & 63.23 & \textbf{65.66}
& 69.29 & 68.89 & 71.11 & \textbf{72.32} \\

M & (300 to 2000)
& 62.47 & 67.67 & 72.60 & \textbf{73.97}
& 63.84 & 69.04 & 70.14 & \textbf{70.96} \\

L & ($>$2000)
& 66.12 & 75.51 & 76.33 & \textbf{78.37}
& 50.61 & 54.69 & 55.10 & \textbf{57.14} \\
\bottomrule
\end{tabular}}
\end{table}

\subsection{Performance Analysis between Chain-of-Table and Tree-of-Table} \label{sub:cot-vs-tot}
\begin{table}
\caption{Deviated performance in tabular size arises from the two frameworks' strength in different semantic tasks, whereas Chain-of-Table and Tree-of-Table perform best on lookup and computation tasks respectively, which dominate in different datasets. Results are reported as accuracy (average $\pm$ standard deviation).}
\label{tab:chain-vs-tree}
\centering
\resizebox{0.98\columnwidth}{!}{
\begin{tabular}{cc|cc|cc}
\toprule
\multirow{2}{*}{Dataset} & \multirow{2}{*}{\makecell[c]{Average \\ Token}} & \multicolumn{2}{c}{Chain-of-Table (Tools)} & \multicolumn{2}{c}{Tree-of-Table (Tools)} \\
& & Accuracy & Acc. (+Reward) & Accuracy & Acc. (+Reward) \\
WikiTQ & 661.22 & 53.37 $\pm$ 3.53 & 73.55 $\pm$ 3.12 & 62.86 $\pm$ 3.42 & 73.96 $\pm$ 3.10\\
MMQA & 339.46 & 41.01 $\pm$ 3.48 & 61.97 $\pm$ 3.43 & 50.63 $\pm$ 3.54 & 67.14 $\pm$ 3.32 \\
MMTU & 9,499.04 & 47.44 $\pm$ 3.53 & 65.47 $\pm$ 3.36 & 40.67 $\pm$ 3.47 & 58.18 $\pm$ 3.49 \\
\bottomrule
\end{tabular}
}
\end{table}
Further analysis on dataset-wise performance (Table~\ref{tab:chain-vs-tree}) shows that Tree-of-Table performs best on WikiTQ and MMQA, which require semantic inference to align latent attributes with the query. These datasets consist largely of small to medium tables, where hierarchical decomposition enables independent resolution of conceptual components. Such decomposition is particularly beneficial for MMQA, where relevant signals are less explicit. In contrast, MMTU contains large tables focused on exact lookups over long and irregular expressions. Such patterns are difficult to identify directly from prompts but can be precisely retrieved via atomic operations (e.g., \texttt{f\_select\_row}, \texttt{f\_select\_column}), favoring the progressive reasoning style of Chain-of-Table. For Tree-of-Table, early semantic decomposition conversely introduces biases, such as normalizing intentionally irregular values by misinterpreting the term as a typographical error, hence degrading exact retrieval. This is further supported by Figure~\ref{fig:question-type-analysis}, where Chain-of-Table performs best on \emph{needle-in-a-table} queries, while Tree-of-Table is stronger on computation-intensive tasks. Incorporating reward feedback via \ourmodel{} mitigates these complementary weaknesses, producing more consistent performance across datasets and task types.
\subsection{Question Type Analysis} \label{subsec:question-type-analysis}
\begin{figure}[H]
    \centering
    \includegraphics[width=0.95\linewidth]{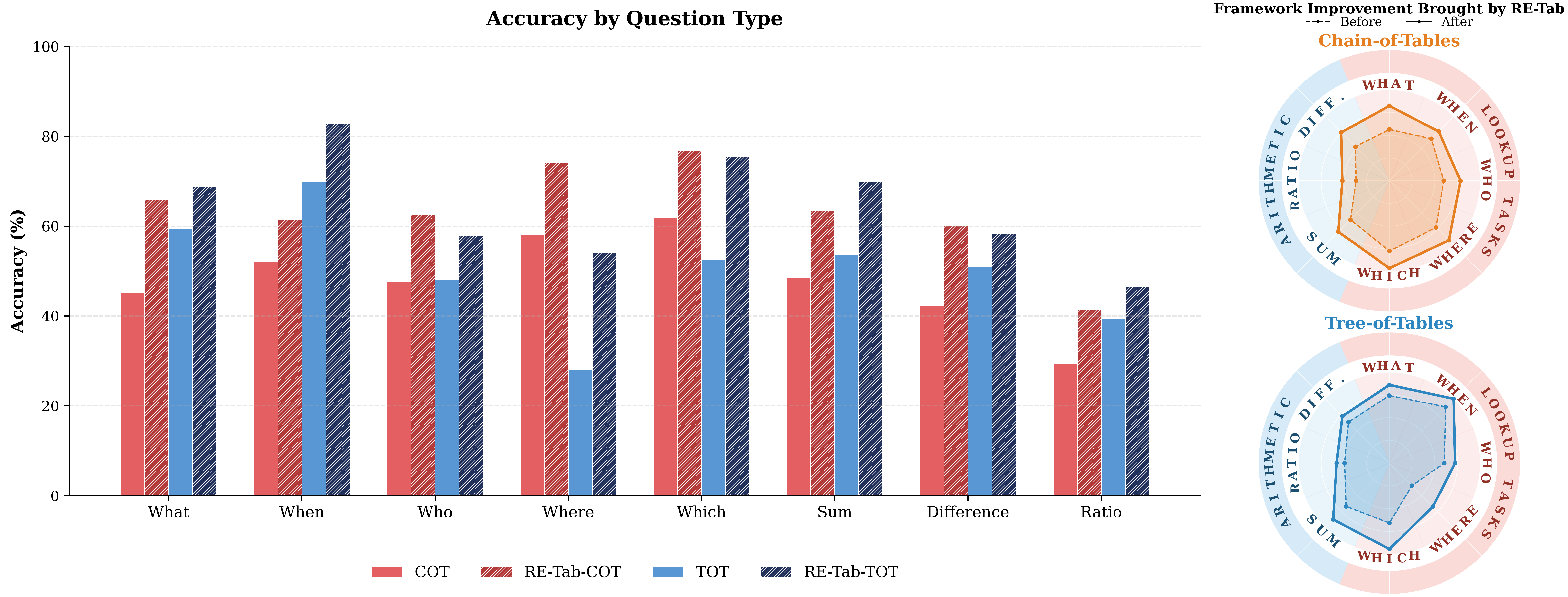}
    \caption{\textbf{\ourmetric{} delivers its largest gains on arithmetic queries.} Per-question-type accuracy on WikiTQ across eight categories: five lookup intents (\texttt{What}, \texttt{When}, \texttt{Who}, \texttt{Where}, \texttt{Which}) and three arithmetic intents (\texttt{Sum}, \texttt{Difference}, \texttt{Ratio}). Incorporating \ourmodel{} improves \texttt{Sum} by +15-16~pp and \texttt{Difference} by +7-18~pp over both Chain- and Tree-of-Table, showing that the lexical reward does not penalize aggregation or numeric reasoning despite operating purely on serialized table strings. Within lookup intents we further distinguish \texttt{which} queries (options given in the question) from \texttt{what} queries (require table lookup).}
    \label{fig:question-type-analysis}
\end{figure}

\section{Implementation Details} \label{sec:implementation-details}
In our implementation, for GPT-4.1-nano, GPT-5-nano, and GPT-5.4~\cite{openai2026gpt54}, we use the API provided by OpenAI for inference. For the three open-source models (QWEN3-8B, Ministral-3-8B, and GPT-oss-20B), we use Ollama~\cite{ollama2025} for efficient inference on an NVIDIA GeForce RTX 4090 GPU. Except for GPT-5-nano and GPT-5.4, all models expose top-$k$ log-probabilities. For these two backbones, we use an additional GPT-4.1-nano agent to rewrite reasoning text at each step, indirectly estimating the action-level confidence.

To focus on evaluating quality of \ourmetric{}, we set all discount factors to be 1. Regarding halting criteria, we continuously measure \ourmetric{} in each step, and we declare a signal convergence if the rolling variance of the signal is less than a threshold, which indicates that the agent is unable to further improve table precision. Further hyperparameter details are included in the following detailed algorithm:

\paragraph{Evaluation Datasets.}
For the Tier-1 main results (Table~\ref{tab:initial-result}), we evaluate on $n{=}200$ questions per dataset across WTQ~\cite{pasupat-liang-2015-compositional}, TableBench~\cite{10.1609/aaai.v39i24.34739}, TabFact~\cite{Chen2020TabFact}, and MMQA~\cite{wu2025mmqa}, sampled under the unified agentic protocol described in Appendix~\ref{subsec:evaluation-protocol}. For the Tier-2 controlled ablations (Table~\ref{tab:main_results}), we evaluate on three benchmarks at $n=200$ per datasets. The smaller subsets reflect the cross-product cost of sweeping six backbones and four reward conditions per dataset. Given these subset sizes, the binomial 95\% Wilson resolution is approximately $\pm$6.86\,pp, so cross-condition gaps within these margins should be read as ties rather than ranked differences. Outputs are reduced to concise answer spans using the same backbone model family as the generator and compared against ground truth via exact-match with NFKD normalization and canonicalization of common symbols and units (full extraction and normalization details in Appendix~\ref{subsec:evaluation-protocol}).

\subsection{Hyperparameter Configuration}
Table~\ref{tab:all_hyperparameters} summarizes the key rendering, generation, and agent-loop hyperparameters used throughout our experiments. We keep these settings fixed across datasets unless a model API imposes additional constraints, so that performance differences are attributable to the reward design rather than ad hoc tuning.

\begin{table*}[ht]
\centering
\caption{Summary of image rendering, model generation, and agent-loop hyperparameters used in our experiments.}
\label{tab:all_hyperparameters}
\renewcommand{\arraystretch}{1.12}
\setlength{\tabcolsep}{8pt}
\resizebox{0.95\textwidth}{!}{%
\begin{tabular}{lll}
\toprule
\textbf{Parameter} & \textbf{Value / Default} & \textbf{Description} \\
\midrule
\multicolumn{3}{c}{\textbf{Image Rendering Hyperparameters}} \\
\midrule
Max rows displayed & 20 & Truncate overly long table views during rendering \\
Max cols displayed & 10 & Truncate overly wide table views during rendering \\
Font size & 8pt & Standardize table readability across renderings \\
Row scale & 1.2$\times$ & Standardize vertical spacing for table rows \\
DPI & 200 & Standardize figure resolution \\
Format & \texttt{.jpg} & Standardize rendered image format \\
\midrule
\multicolumn{3}{c}{\textbf{Model / Generation Hyperparameters}} \\
\midrule
Max tokens & 8192 & Maximum output tokens per model call \\
Temperature & 0.7 & Sampling temperature (set to 1 when required by GPT-5 APIs) \\
Top-$k$ logprobs & 20 & Used for action-based confidence rewards when available \\
Seed & 42 & Default seed for reproducible runs when supported \\
\midrule
\multicolumn{3}{c}{\textbf{Experiment / Agentic Loop Hyperparameters}} \\
\midrule
Max steps & 12 & Maximum tool-use steps per question attempt \\
Runtime retries & 2 & Retries for malformed JSON/tool outputs \\
Chain attempts $I$ & 10 & Independent attempts for Chain-of-Table aggregation \\
Tree attempts $K$ & 20 & Independent attempts for Tree-of-Table aggregation \\
ROUGE window $w$ & 5 & Window size for rolling reward convergence \\
Convergence threshold $\varepsilon$ & 0.005 & Variance threshold for early stopping \\
Discount factor $\gamma$ & 1.0 & Fixed to isolate reward quality from temporal discounting \\
\bottomrule
\end{tabular}}
\end{table*}

\subsection{Computational Configuration}
Inference is conducted using OpenAI-hosted APIs for GPT-4.1-nano, GPT-5-nano, and GPT-5.4~\cite{openai2026gpt54}, and with Ollama~\cite{ollama2025} for the open-weight models QWEN3-8B, Ministral-3-8B, and GPT-oss-20B on NVIDIA GeForce RTX 4090 GPUs. Unless the serving API requires otherwise, we use the same generation settings across models to reduce confounding from infrastructure differences. For GPT-5-nano and GPT-5.4, which do not expose token-level confidence in the same way as the other backbones, we employ an auxiliary GPT-4.1-nano rewriting step to approximate action-level confidence for the action-based reward baselines.

\subsection{GRPO Post-Training Setup} \label{sub:grpo-setup}
We provide setup details for the preliminary GRPO experiments reported in Section~\ref{sub:grpo-ablation}. All runs train Qwen3-8B with the GRPO~\cite{shao2024deepseekmathpushinglimitsmathematical} advantage estimator on a single node of $4{\times}$~A100 GPUs.

\paragraph{Algorithm and policy regularization.}
We use group-relative advantage estimation with group size $n{=}8$ rollouts per prompt. The KL-divergence regularizer is applied to the policy loss with coefficient $\beta_{\mathrm{KL}}{=}0.01$ using the low-variance estimator (rather than added to the reward). Entropy regularization is included with coefficient $0.001$. Per-group standardization of advantages within each prompt's eight rollouts means the optimizer never compares raw \ourmetric{} magnitudes across different table-query pairs, sidestepping cross-table reward-calibration concerns.

\paragraph{Optimization.}
Learning rate is $1\!\times\!10^{-6}$ with a cosine warmup over the first $5\%$ of steps and a minimum-LR ratio of $0.1$. Train batch size is $64$, PPO mini-batch size $64$, and PPO micro-batch size per GPU $4$. Maximum prompt length is $6{,}144$ tokens. Maximum response length is $1{,}024$ tokens. Over-long prompts are truncated from the left. Multi-turn rollouts are capped at $12$ assistant turns and $13$ user turns, matching the agent's existing tool-loop budget.

\paragraph{Datasets and steps.}
Training subsets are prepared via \texttt{prepare\_datasets\_parquet.py}: WikiTQ uses $n{=}200$ training examples (random seed $42$). MMQA and MMTU use $n{=}92$ (random seed $1018$). We train WikiTQ and MMQA for $300$ steps and MMTU for $150$ steps. The MMTU schedule is halved due to out-of-memory errors when handling MMTU's long-valued tables on the available hardware. Held-out evaluation uses the same extraction-and-comparison protocol as the controlled ablation in Section~\ref{sub:initial-results}.

\paragraph{Rollout engine and FSDP.}
Rollouts use SGLang as the inference engine with $50\%$ GPU memory utilization and tensor-parallel size $2$, on top of FSDP-based actor optimization with parameter and optimizer offloading to CPU.

\subsection{Efficiency Benchmark Protocol} \label{sub:efficiency-protocol}
Our efficiency analysis is designed to measure not only whether a state-supervision signal improves accuracy, but also how much test-time search is required to obtain that improvement. Following the style of recent efficiency reporting in related table-reasoning work, we make the benchmark setup and metric definitions explicit so that the reported search-cost reductions are reproducible and interpretable.

\paragraph{Benchmark setup and sampling.}
For each method, we evaluate the same benchmark questions under the same backbone model, prompt template, tool inventory, and execution environment, and vary only the state-supervision strategy used for trajectory ranking and stopping. On each question, we sample multiple independent reasoning trajectories and record the full search trace until the method reaches a stable consensus answer or satisfies its stopping rule. This design isolates differences attributable to reward-guided search rather than implementation mismatch.

\paragraph{Reported metrics.}
We report the following efficiency indicators.
\begin{itemize}
    \item \textbf{Sampled trajectories to convergence:} the number of sampled reasoning chains required to reach 95\% of the best observed accuracy for a reward strategy. This is our main test-time scaling metric. Lower values indicate that the method reaches near-peak performance with fewer sampled trajectories.
    \item \textbf{Reasoning steps per trajectory:} the average number of tool-use steps executed within a sampled trajectory. This quantifies step efficiency and indicates whether reward feedback helps the agent halt or prune earlier once the decisive table state has been reached.
    \item \textbf{Late-stage token generation:} the number of token-consuming reasoning rounds that occur after the trajectory has already captured most of the useful evidence. This measures how much redundant deliberation remains in the search process.
    \item \textbf{End-to-end latency:} the total wall-clock time required to process a single sample under a fixed serving setup. We report this as a deployment-facing measure of inference cost, while noting that absolute values depend on provider and hardware stack.
\end{itemize}

The reported cost reduction should therefore be interpreted as the joint effect of \emph{trajectory efficiency} (fewer samples), \emph{step efficiency} (shorter reasoning traces), and \emph{generation efficiency} (less redundant late-stage deliberation). In the main text, we emphasize relative comparisons within a fixed pipeline rather than universal wall-clock speedups across providers.

From a computational perspective, the per-step cost of \ourmetric{} is dominated by LCS computation between the query and the serialized intermediate table, giving worst-case complexity $O(|Q|\,|Enc(T_t)|)$ at step $t$, and total trajectory cost $\sum_t O(|Q|\,|Enc(T_t)|)$. In practice, this overhead is modest because $|Q|$ is small and intermediate states are already bounded by the agent’s serialization and observation pipeline. As a result, reward computation remains far cheaper than an additional LLM call, so the overall runtime is still dominated by model inference rather than by scoring. We therefore interpret the measured efficiency gains primarily as reductions in search and deliberation burden within a fixed inference stack.

\begin{table}[ht]
\centering
\caption{Token budget and wall-clock latency for \ourmodel{} ($K{=}3$ trajectories per question) across datasets and backbone models. Mean and standard deviation computed over all questions in each split. Missing cells indicate configurations not evaluated.}
\label{tab:cost-accounting}
\small
\setlength{\tabcolsep}{5pt}
\renewcommand{\arraystretch}{1.10}
\begin{tabular}{llcrrrr}
\toprule
\textbf{Dataset} & \textbf{Model} & \textbf{Condition} & \textbf{Latency (s)} & \textbf{$\sigma$ lat.} & \textbf{Tokens} & \textbf{$\sigma$ tok.} \\
\midrule
\multirow{5}{*}{WikiTQ}
 & Qwen3-8B       & no reward & 36.4 & 30.7 & 161.7 & 160.2 \\
 & Qwen3-8B       & \ourmodel{}    & 38.9 & 32.2 & 125.1 & 123.0 \\
 & Ministral-3-8B & no reward & 16.9 &  9.0 & 491.3 & 228.1 \\
 & Ministral-3-8B & \ourmodel{}    & 13.0 &  4.9 & 408.0 & 187.5 \\
 & GPT-oss-20B    & \ourmodel{}    &  9.5 &  3.9 &  61.6 &  52.8 \\
\midrule
\multirow{6}{*}{MMQA}
 & Qwen3-8B       & no reward & 55.9 & 59.0 & 103.6 & 121.9 \\
 & Qwen3-8B       & \ourmodel{}    & 49.0 & 37.5 & 106.4 & 123.8 \\
 & Ministral-3-8B & no reward & 14.7 &  6.8 & 493.8 & 216.0 \\
 & Ministral-3-8B & \ourmodel{}    & 12.4 &  3.8 & 457.8 & 189.2 \\
 & GPT-oss-20B    & no reward & 11.0 &  5.8 &  77.3 & 129.8 \\
 & GPT-oss-20B    & \ourmodel{}    & 13.0 &  4.1 &  30.5 &  33.5 \\
\midrule
\multirow{6}{*}{MMTU}
 & Qwen3-8B       & no reward & 37.7 & 33.3 & 182.6 & 144.2 \\
 & Qwen3-8B       & \ourmodel{}    & 36.3 & 24.2 & 218.1 & 291.8 \\
 & Ministral-3-8B & no reward & 15.5 & 10.7 & 494.9 & 220.8 \\
 & Ministral-3-8B & \ourmodel{}    & 13.8 &  6.0 & 475.8 & 155.6 \\
 & GPT-oss-20B    & no reward & 37.2 & 45.6 &  75.1 & 157.4 \\
 & GPT-oss-20B    & \ourmodel{}    & 29.2 & 24.7 &  75.6 & 263.8 \\
\bottomrule
\end{tabular}
\end{table}

\begin{figure}[H]
\centering

\begin{minipage}[t]{0.48\linewidth}
\begin{algorithm}[H]
\caption{Chain-of-Table with \ourmodel{} Reasoning}
\label{alg:chain_of_tables}
\small
\begin{algorithmic}[1]

\State {\bfseries Input:} Dataset path $D$, number of samples $N$, 
iterations per question $I{=}10$, LLM $\mathcal{M}$, 
\ourmetric{} window $w{=}5$, convergence threshold $\varepsilon{=}0.005$

\State Load dataset samples $\{(Q_i, T_i, \alpha_i)\}_{i=1}^{N}$ from $D$

\For{$i = 1$ {\bfseries to} $N$}
    \State Initialize table state $T_0 \leftarrow T_i$

    \For{$j = 1$ {\bfseries to} $I$}
        \State Initialize \ourmetric{} history $\mathcal{R} \leftarrow [\,]$
        \State $t \leftarrow 0$

        \Repeat
            \State Initialize observation $o_{t} \leftarrow T_{t}$
            \State $t \leftarrow t + 1$

            \State LLM generates reasoning and tool call:
            \State \hspace{1em} $(r_t, \tau_t) \leftarrow \mathcal{M}(Q_i, T_{t-1})$

            \State Execute the single proposed tool call (no branch expansion or rollback):
            \State \hspace{1em} $T_{t} \leftarrow a_t(T_{t-1})$

            \State Compute \ourmetric{} on the resulting state:
            \State \hspace{1em} $r^{\text{\ourmetric{}}}_t \leftarrow \text{\ourmetric{}}(Q_i, T_t)$

            \State Append $r^{\text{\ourmetric{}}}_t$ to $\mathcal{R}$ for stopping and final trajectory ranking

            \If{\ourmetric{}Converged($\mathcal{R}, w, \varepsilon$)}
                \State Instruct LLM to produce final answer
                \State {\bfseries break}
            \EndIf
        \Until{LLM emits {\ttfamily f\_final\_answer}}

        \State Record predicted answer $\hat{\alpha}_{i,j}$
    \EndFor

    \State Select final answer $\hat{\alpha}_i$ by Simulative Reasoning on the state-supervision signal
    \State Output $(Q_i, \hat{\alpha}_i, \alpha_i)$
\EndFor

\end{algorithmic}
\end{algorithm}
\end{minipage}
\hfill
\begin{minipage}[t]{0.48\linewidth}
\begin{algorithm}[H]
\caption{Tree-of-Table with \ourmodel{} Reasoning}
\label{alg:tree_of_tables}
\small
\begin{algorithmic}[1]
\State {\bfseries Input:} Dataset path $D$, number of samples $N$, LLM $\mathcal{M}$,
maximum attempts $K{=}20$, ROUGE window $w{=}5$, convergence threshold $\varepsilon{=}0.005$

\State Load dataset samples $\{(Q_i, T_i, \alpha_i)\}_{i=1}^{N}$ from $D$

\For{$i = 1$ {\bfseries to} $N$}
    \State Initialize table state $T_0 \leftarrow T_i$
    \State Initialize Tree-of-Table structure $\mathcal{G} \leftarrow \{\text{root}\}$
    \State Initialize \ourmetric{} history $\mathcal{R} \leftarrow [\,]$

    \For{$k = 1$ {\bfseries to} $K$}
        \State $t \leftarrow 0$
        \Repeat
            \State Initialize observation $o_t \leftarrow T_t$
            \State $t \leftarrow t + 1$
            \State LLM generates reasoning and selects tool call:
            \State \hspace{1em} $(r_t, \tau_t) \leftarrow \mathcal{M}(Q_i, o_{t})$
            \State Execute the single proposed tool call (no branch expansion or rollback): $T_{t+1} \leftarrow \tau_t(T_{t})$
            \State Append snapshot $o_{t}$ to Tree-of-Table $\mathcal{G}$
            \State Compute \ourmetric{} on the resulting state $r^{\text{\ourmetric{}}}_t \leftarrow \text{\ourmetric{}}(Q_i, T_t)$
            \State Append $r^{\text{\ourmetric{}}}_t$ to $\mathcal{R}$ for stopping and final trajectory ranking

            \If{\ourmetric{}Converged($\mathcal{R}, w, \varepsilon$)}
                \State Produce final answer
                \State {\bfseries break}
            \EndIf
        \Until{LLM emits {\ttfamily f\_final\_answer}}

        \State Record predicted answer $\hat{\alpha}_{i,k}$
    \EndFor

    \State Select final answer $\hat{\alpha}^{*}_i$ via Simulative Reasoning on the state-supervision signal
    \State Write $(Q_i, \hat{\alpha}_i, \alpha_i)$ to output file $F$
\EndFor

\end{algorithmic}
\end{algorithm}
\end{minipage}

\end{figure}

\newpage

\subsection{Observation Model Details}
\label{subsec:observation-model}
Our agent never reasons over the entire table state at every step. Instead, the environment exposes a partial observation $o_t$ derived from the current dataframe $T_t$, and the agent conditions its next action only on that observation together with the question and prior reasoning history.

\paragraph{What is shown to the model.}
Whenever the environment reveals the table, we serialize the current dataframe into a compact textual snapshot. This snapshot contains the currently retained columns and a bounded prefix of rows from the dataframe, rather than the full table. In our default implementation, the visible window contains the first 20 rows of the current dataframe after the latest transformation. This design keeps prompts within a stable context budget and mirrors realistic agentic settings where repeatedly exposing the full table is both expensive and unnecessary.

\paragraph{Which actions expose the table state.}
Not every action refreshes the observation. The agent receives an explicit table snapshot at the initial load and after actions that directly materialize a table view, namely \texttt{f\_select\_column}, \texttt{f\_select\_row}, and \texttt{f\_print\_table}. These actions update the dataframe and then serialize the visible window for the next reasoning step. Other actions may change intermediate reasoning state without expanding the visible table context. Consequently, the agent often acts under partial observability even though the underlying table state has changed.

\paragraph{Why the observation can be insufficient.}
This observation protocol creates a natural mismatch between local evidence and global table state. Relevant rows may fall outside the visible prefix, and a locally plausible snapshot may still omit the column or record needed for the final answer. This is particularly important for ambiguous or compositional questions, where several candidate subspaces appear promising from the local window alone. The purpose of the reward signal is therefore not to replace the observation, but to complement it with a scalar summary of whether the hidden full state has become more aligned with the query.

\paragraph{Sensitivity to windowing and sampling.}
We use a fixed visible-window rule in the main experiments for reproducibility. Empirically, performance is reasonably stable to moderate changes in serialization and window design, provided the agent is consistently exposed to a coherent local view. Nonetheless, the observation function remains a meaningful source of difficulty: if the relevant evidence is absent from the visible rows, the agent may follow a drifted reasoning path despite selecting valid tools. This is one reason why explicit reward feedback is helpful. It provides a trajectory-level correction signal even when the currently displayed rows are unrepresentative of the best downstream action.

\paragraph{Relation to the theoretical discussion.}
The observation model operationalizes the POMDP view introduced in the main paper. The hidden state is the full transformed table $T_t$, while the model acts on a bounded observation $o_t$. In this setting, deterministic state supervision serves as an auxiliary channel for state disambiguation: it compresses information from the full post-action table into a stepwise signal that is inexpensive to compute and more globally informative than the displayed snapshot alone.

\subsection{Evaluation Protocol}
\label{subsec:evaluation-protocol}
We adopt a three-stage evaluation pipeline so that binary accuracy reflects answer equivalence rather than superficial formatting differences.

\paragraph{Stage 1: Answer extraction.}
Most TableQA agents generate a short rationale together with the final answer. To standardize evaluation across models, we first reduce each raw generation to a concise answer span. We use the same backbone model family that produced the response to extract the final answer phrase from its own output, which keeps the extraction rule uniform across methods while avoiding a systematic advantage for any one external extractor. This step is purely post-processing and does not alter the reasoning trajectory used during inference. Importantly, this extraction stage is used only to isolate the model's already-produced final answer span from a verbose response. It is not used to judge whether that answer is correct.

\paragraph{Stage 2: Answer normalization.}
Before comparison, both the predicted answer and the gold answer are normalized by applying Unicode NFKD decomposition, stripping diacritics, lowercasing when case is not semantically meaningful, collapsing repeated whitespace, and standardizing punctuation and symbols. We also normalize common answer variants that frequently appear in tabular data, including numerical formatting differences (e.g., commas, trailing zeros, signed symbols), simple unit variants, and visually different but denotationally equivalent strings. This normalization is designed to preserve answer meaning while removing incidental surface variation.

\paragraph{Stage 3: Standard comparison.}
After extraction and normalization, we compute binary accuracy using the standard task comparison protocol, i.e. exact match or denotation-equivalence evaluation under dataset-appropriate normalization. This handles cases where raw string equality is too brittle, such as reordered phrases, abbreviated entities, or answers with optional units, while remaining aligned with standard TableQA evaluation practice.

\paragraph{Scope and motivation.}
We use this evaluation design because agent outputs in TableQA are often verbose and heterogeneous even when they express the same answer. A raw surface-form exact-match protocol would therefore undercount correct predictions and confound reasoning quality with stylistic variation. Our evaluation framing follows recent table-reasoning work, which typically reports standard benchmark accuracy with limited additional adjudication detail~\cite{zou2026tattoo,yang2025cit,xing2026tabledart}. Compared with these works, we make the pipeline more explicit by separating answer extraction, normalization, and final deterministic comparison. In particular, correctness adjudication in our main results occurs only at the last stage through dataset-standard exact-match or denotation-equivalence comparison against ground truth, rather than through an LLM-as-a-judge decision.

\subsection{Agent Prompt}
The respective agent prompt is as follows:

\begin{tcolorbox}[
  title={Agent Prompt for TableQA with Tools},
  colback=white,
  colframe=black,
  boxrule=0.6pt,
  arc=2pt,
  left=6pt,right=6pt,top=6pt,bottom=6pt
]
\noindent\textbf{Task.} Answer the following question about the loaded table:
\[
\{\texttt{question}\}.
\]

\noindent\textbf{Environment.} The table is loaded as a pandas DataFrame \texttt{df}. There are \texttt{\{num\_tools\}} tools available:
\begin{itemize}[leftmargin=1.2em]
  \item \texttt{\{tools\_context\}}
\end{itemize}

\noindent\textbf{Instruction.} Follow your reasoning steps, then provide a clear final answer.

\vspace{0.5em}
\noindent\textbf{Pre-processing (before reasoning).}
\begin{enumerate}[leftmargin=1.4em]
  \item If a name/entity mentioned in the question is not found in \texttt{df}, make sensible inferences to map the entity to the table. Assume an answer always exists.
  \item If the DataFrame contains unit symbols (e.g., \texttt{\$}) or commas as thousand separators, use \texttt{StringOperationTool()} to clean the data and parse values into numbers before further computation.
\end{enumerate}

\noindent\textbf{Tips (use if necessary).}
\begin{itemize}[leftmargin=1.2em]
  \item Some columns may contain only a single group/category.
  \item The toolset computes time differences in \emph{days}.
  \item The question may contain typos; interpret it so that at least one valid answer exists.
  \item If the question does not specify alphabetical or chronological ordering, preserve the original row order of \texttt{df}.
  \item For NA-indicating symbols (e.g., \texttt{-}, \texttt{``N/A''}, \texttt{``unknown''}, \texttt{``null''}): convert them to \texttt{pd.NA} using \texttt{StringOperationTool()}, then handle them with \texttt{HandleMissingValuesTool()}.
  \item If \texttt{df} seems cleared or unavailable, use \texttt{RetrieveOriginalDFTool()} to restore it from \texttt{df\_og}.
\end{itemize}
\end{tcolorbox}

\section{Compute Reporting}
\label{sec:compute-reporting}
We provide a transparency-oriented summary of the computational resources used to produce the results reported in this paper, following the spirit of recent venue-level compute reporting initiatives. This information is not part of the technical evaluation of the paper and is intended only to make hardware usage and inference cost auditable.

\paragraph{Inference compute.}
Inference for the open-weight backbones (QWEN3-8B, Ministral-3-8B, GPT-oss-20B) was conducted with Ollama on a single NVIDIA GeForce RTX 4090 GPU. Inference for GPT-4.1-nano, GPT-5-nano, and GPT-5.4 was conducted through the OpenAI-hosted API; no local GPU compute was used for these backbones.

\paragraph{Training compute.}
The preliminary GRPO experiments in Section~\ref{sub:grpo-ablation} were trained on a single node of $4{\times}$~NVIDIA A100 GPUs. WikiTQ and MMQA were each trained for $300$ steps and MMTU for $150$ steps with the hyperparameters listed in Section~\ref{sub:grpo-setup}.

\paragraph{Development and failed runs.}
Beyond the runs that appear in the final tables, additional compute was spent on prompt iteration, hyperparameter exploration, debugging of the tool-loop and serialization layer, and discarded training configurations that did not enter the paper, on the same RTX 4090 and $4{\times}$~A100 hardware described above.

\paragraph{External services.}
No proprietary cluster, paid annotation service, or third-party labeler was used. Hosted-model inference relied solely on the OpenAI API as described above.

\newpage

\section*{Impact Statement}
This paper aims to advance machine learning by improving the reliability and efficiency of reasoning over structured data. Tabular data is a cornerstone of information storage in critical sectors such as healthcare, finance, and marketing. By enhancing the accuracy and grounding of TableQA agents, this work provides a more dependable interface for data-driven decision-making.

The proposed \ourmodel{} framework contributes to sustainable AI by reducing inference cost and token generation through more efficient state exploration, which lowers the energy footprint of test-time scaling. By introducing deterministic, query-grounded rewards, the framework also addresses the opaque nature of multi-step reasoning, fostering greater transparency and reducing the risk of compounding errors or hallucinations in automated data analysis. These properties are valuable for deploying trustworthy AI agents in environments where precision and interpretability are paramount.

Several risks accompany this line of work. More capable TableQA agents may accelerate automated decision-making in high-stakes domains such as clinical triage, credit assessment, and hiring, where residual errors can disproportionately harm specific groups. Stronger surface fluency without commensurate gains in calibration may also encourage over-reliance and reduce human oversight of model outputs. \ourmetric{} is a structure-aware lexical reward and inherits the failure modes documented in \cref{subsec:failure-modes-main,sec:discussion}, including paraphrase under-rewarding, echo-column hacking, and serialization sensitivity. If the metric is treated as a sole training or selection signal in domains where these failure modes apply, models may learn to exploit it rather than improve genuine reasoning quality. To mitigate these risks, we recommend pairing \ourmetric{} with answer-checking or human review in deployment, auditing for the reported failure modes when porting the reward to new domains, and combining lexical supervision with executor-grounded verification whenever an executor is available. The benchmarks used in our experiments are public and contain no personally identifying information, so this work does not introduce new dataset-level privacy concerns.
\newpage

\section*{NeurIPS Paper Checklist}

The checklist is designed to encourage best practices for responsible machine learning research, addressing issues of reproducibility, transparency, research ethics, and societal impact. Do not remove the checklist: {\bf The papers not including the checklist will be desk rejected.} The checklist should follow the references and follow the (optional) supplemental material.  The checklist does NOT count towards the page
limit. 

Please read the checklist guidelines carefully for information on how to answer these questions. For each question in the checklist:
\begin{itemize}
    \item You should answer \answerYes{}, \answerNo{}, or \answerNA{}.
    \item \answerNA{} means either that the question is Not Applicable for that particular paper or the relevant information is Not Available.
    \item Please provide a short (1--2 sentence) justification right after your answer (even for \answerNA). 
\end{itemize}

{\bf The checklist answers are an integral part of your paper submission.} They are visible to the reviewers, area chairs, senior area chairs, and ethics reviewers. You will also be asked to include it (after eventual revisions) with the final version of your paper, and its final version will be published with the paper.

The reviewers of your paper will be asked to use the checklist as one of the factors in their evaluation. While \answerYes{} is generally preferable to \answerNo{}, it is perfectly acceptable to answer \answerNo{} provided a proper justification is given (e.g., error bars are not reported because it would be too computationally expensive'' or ``we were unable to find the license for the dataset we used''). In general, answering \answerNo{} or \answerNA{} is not grounds for rejection. While the questions are phrased in a binary way, we acknowledge that the true answer is often more nuanced, so please just use your best judgment and write a justification to elaborate. All supporting evidence can appear either in the main paper or the supplemental material, provided in appendix. If you answer \answerYes{} to a question, in the justification please point to the section(s) where related material for the question can be found.

IMPORTANT, please:
\begin{itemize}
    \item {\bf Delete this instruction block, but keep the section heading ``NeurIPS Paper Checklist"},
    \item  {\bf Keep the checklist subsection headings, questions/answers and guidelines below.}
    \item {\bf Do not modify the questions and only use the provided macros for your answers}.
\end{itemize}


\begin{enumerate}

\item {\bf Claims}
    \item[] Question: Do the main claims made in the abstract and introduction accurately reflect the paper's contributions and scope?
    \item[] Answer: \answerYes{}
    \item[] Justification: The abstract and \cref{sec:experiments} report the contributions stated up front, namely a deterministic LCS-based state reward (\ourmetric{}, \cref{subsec:state-reward}), a training-free framework (\ourmodel{}, \cref{sub:two-phase}), an average $+26.7$~pp accuracy gain and up to $33\%$ reduction in TTS samples across five backbones and three benchmarks, and a preliminary $+8.34$~pp GRPO post-training result. Scope and bounds are stated in \cref{subsec:state-reward,subsec:failure-modes-main,sec:discussion}.
    \item[] Guidelines:
    \begin{itemize}
        \item The answer \answerNA{} means that the abstract and introduction do not include the claims made in the paper.
        \item The abstract and/or introduction should clearly state the claims made, including the contributions made in the paper and important assumptions and limitations. A \answerNo{} or \answerNA{} answer to this question will not be perceived well by the reviewers. 
        \item The claims made should match theoretical and experimental results, and reflect how much the results can be expected to generalize to other settings. 
        \item It is fine to include aspirational goals as motivation as long as it is clear that these goals are not attained by the paper. 
    \end{itemize}

\item {\bf Limitations}
    \item[] Question: Does the paper discuss the limitations of the work performed by the authors?
    \item[] Answer: \answerYes{}
    \item[] Justification: \cref{sec:discussion} dedicates a section to limitations (serialization sensitivity in \cref{sub:delimiter-sensitivity}, lack of global cross-table reward calibration, and a precision bias in the default $\beta$ configuration). \cref{subsec:state-reward,subsec:failure-modes-main} characterize and quantify the three known failure regimes (paraphrase and derived-evidence under-rewarding, echo-column reward hacking, and renamed or distractor-column collisions) on real trajectories using the ranking-inversion rate.
    \item[] Guidelines:
    \begin{itemize}
        \item The answer \answerNA{} means that the paper has no limitation while the answer \answerNo{} means that the paper has limitations, but those are not discussed in the paper. 
        \item The authors are encouraged to create a separate ``Limitations'' section in their paper.
        \item The paper should point out any strong assumptions and how robust the results are to violations of these assumptions (e.g., independence assumptions, noiseless settings, model well-specification, asymptotic approximations only holding locally). The authors should reflect on how these assumptions might be violated in practice and what the implications would be.
        \item The authors should reflect on the scope of the claims made, e.g., if the approach was only tested on a few datasets or with a few runs. In general, empirical results often depend on implicit assumptions, which should be articulated.
        \item The authors should reflect on the factors that influence the performance of the approach. For example, a facial recognition algorithm may perform poorly when image resolution is low or images are taken in low lighting. Or a speech-to-text system might not be used reliably to provide closed captions for online lectures because it fails to handle technical jargon.
        \item The authors should discuss the computational efficiency of the proposed algorithms and how they scale with dataset size.
        \item If applicable, the authors should discuss possible limitations of their approach to address problems of privacy and fairness.
        \item While the authors might fear that complete honesty about limitations might be used by reviewers as grounds for rejection, a worse outcome might be that reviewers discover limitations that aren't acknowledged in the paper. The authors should use their best judgment and recognize that individual actions in favor of transparency play an important role in developing norms that preserve the integrity of the community. Reviewers will be specifically instructed to not penalize honesty concerning limitations.
    \end{itemize}

\item {\bf Theory assumptions and proofs}
    \item[] Question: For each theoretical result, does the paper provide the full set of assumptions and a complete (and correct) proof?
    \item[] Answer: \answerYes{}
    \item[] Justification: The POMDP formulation in \cref{sub:problem-formulation} states the agent, observation, and reward model used throughout. The main theoretical claim, Proposition~\ref{prop:rouge_optimality} on the monotonic preference behavior of \ourmetric{} under query-conditioned derivation, is proved in \cref{sub:theoretical-support-rouge} together with a separate proposition on accuracy degradation under unsupervised drift in \cref{sub:llm-drift-harm}. Each statement lists its assumptions explicitly, and we flag legitimate transformations that fall outside the stated regime.
    \item[] Guidelines:
    \begin{itemize}
        \item The answer \answerNA{} means that the paper does not include theoretical results. 
        \item All the theorems, formulas, and proofs in the paper should be numbered and cross-referenced.
        \item All assumptions should be clearly stated or referenced in the statement of any theorems.
        \item The proofs can either appear in the main paper or the supplemental material, but if they appear in the supplemental material, the authors are encouraged to provide a short proof sketch to provide intuition. 
        \item Inversely, any informal proof provided in the core of the paper should be complemented by formal proofs provided in appendix or supplemental material.
        \item Theorems and Lemmas that the proof relies upon should be properly referenced. 
    \end{itemize}

    \item {\bf Experimental result reproducibility}
    \item[] Question: Does the paper fully disclose all the information needed to reproduce the main experimental results of the paper to the extent that it affects the main claims and/or conclusions of the paper (regardless of whether the code and data are provided or not)?
    \item[] Answer: \answerYes{}
    \item[] Justification: The full agent loops are given as Algorithms~\ref{alg:chain_of_tables} and~\ref{alg:tree_of_tables}. The atomic operations, textual encoding, observation protocol, evaluation protocol, sampling seeds, dataset splits, and Wilson 95\% half-widths are documented in \cref{sec:detail-operation,sec:textual-encoding,subsec:observation-model,subsec:evaluation-protocol,sec:experiments}. All generation, agent-loop, and reward hyperparameters are listed in Table~\ref{tab:all_hyperparameters}, and the GRPO post-training configuration is in \cref{sub:grpo-setup}.
    \item[] Guidelines:
    \begin{itemize}
        \item The answer \answerNA{} means that the paper does not include experiments.
        \item If the paper includes experiments, a \answerNo{} answer to this question will not be perceived well by the reviewers: Making the paper reproducible is important, regardless of whether the code and data are provided or not.
        \item If the contribution is a dataset and\slash or model, the authors should describe the steps taken to make their results reproducible or verifiable. 
        \item Depending on the contribution, reproducibility can be accomplished in various ways. For example, if the contribution is a novel architecture, describing the architecture fully might suffice, or if the contribution is a specific model and empirical evaluation, it may be necessary to either make it possible for others to replicate the model with the same dataset, or provide access to the model. In general. releasing code and data is often one good way to accomplish this, but reproducibility can also be provided via detailed instructions for how to replicate the results, access to a hosted model (e.g., in the case of a large language model), releasing of a model checkpoint, or other means that are appropriate to the research performed.
        \item While NeurIPS does not require releasing code, the conference does require all submissions to provide some reasonable avenue for reproducibility, which may depend on the nature of the contribution. For example
        \begin{enumerate}
            \item If the contribution is primarily a new algorithm, the paper should make it clear how to reproduce that algorithm.
            \item If the contribution is primarily a new model architecture, the paper should describe the architecture clearly and fully.
            \item If the contribution is a new model (e.g., a large language model), then there should either be a way to access this model for reproducing the results or a way to reproduce the model (e.g., with an open-source dataset or instructions for how to construct the dataset).
            \item We recognize that reproducibility may be tricky in some cases, in which case authors are welcome to describe the particular way they provide for reproducibility. In the case of closed-source models, it may be that access to the model is limited in some way (e.g., to registered users), but it should be possible for other researchers to have some path to reproducing or verifying the results.
        \end{enumerate}
    \end{itemize}

\item {\bf Open access to data and code}
    \item[] Question: Does the paper provide open access to the data and code, with sufficient instructions to faithfully reproduce the main experimental results, as described in supplemental material?
    \item[] Answer: \answerYes{}
    \item[] Justification: All datasets used (WTQ, TableBench, TabFact, MMQA, MMTU) are publicly available and cited at their original sources. Anonymized code, prompts, agent loops, dataset preparation utilities, and GRPO training scripts are provided in the supplementary material with instructions to reproduce the main benchmark results, ablations, efficiency analysis, and the GRPO post-training experiments. Algorithmic and hyperparameter detail sufficient to re-implement the system from scratch is given in \cref{sec:implementation-details,sec:detail-operation,sub:grpo-setup} and Table~\ref{tab:all_hyperparameters}.
    \item[] Guidelines:
    \begin{itemize}
        \item The answer \answerNA{} means that paper does not include experiments requiring code.
        \item Please see the NeurIPS code and data submission guidelines (\url{https://neurips.cc/public/guides/CodeSubmissionPolicy}) for more details.
        \item While we encourage the release of code and data, we understand that this might not be possible, so \answerNo{} is an acceptable answer. Papers cannot be rejected simply for not including code, unless this is central to the contribution (e.g., for a new open-source benchmark).
        \item The instructions should contain the exact command and environment needed to run to reproduce the results. See the NeurIPS code and data submission guidelines (\url{https://neurips.cc/public/guides/CodeSubmissionPolicy}) for more details.
        \item The authors should provide instructions on data access and preparation, including how to access the raw data, preprocessed data, intermediate data, and generated data, etc.
        \item The authors should provide scripts to reproduce all experimental results for the new proposed method and baselines. If only a subset of experiments are reproducible, they should state which ones are omitted from the script and why.
        \item At submission time, to preserve anonymity, the authors should release anonymized versions (if applicable).
        \item Providing as much information as possible in supplemental material (appended to the paper) is recommended, but including URLs to data and code is permitted.
    \end{itemize}

\item {\bf Experimental setting/details}
    \item[] Question: Does the paper specify all the training and test details (e.g., data splits, hyperparameters, how they were chosen, type of optimizer) necessary to understand the results?
    \item[] Answer: \answerYes{}
    \item[] Justification: \cref{sec:experiments} reports backbones, baselines, datasets, sample sizes, and seeds used for each evaluation. \cref{sec:implementation-details} together with Table~\ref{tab:all_hyperparameters} enumerates rendering, generation, and agent-loop hyperparameters, and \cref{sub:grpo-setup} reports the full GRPO post-training configuration including optimizer, learning rate schedule, batch sizes, multi-turn budget, prompt and response lengths, and dataset preparation.
    \item[] Guidelines:
    \begin{itemize}
        \item The answer \answerNA{} means that the paper does not include experiments.
        \item The experimental setting should be presented in the core of the paper to a level of detail that is necessary to appreciate the results and make sense of them.
        \item The full details can be provided either with the code, in appendix, or as supplemental material.
    \end{itemize}

\item {\bf Experiment statistical significance}
    \item[] Question: Does the paper report error bars suitably and correctly defined or other appropriate information about the statistical significance of the experiments?
    \item[] Answer: \answerYes{}
    \item[] Justification: We report binomial 95\% Wilson half-widths for each dataset sample size used in the main results (\cref{sec:experiments}, $\pm 6.86$~pp at $n{=}200$ and $\pm 9.81$~pp at $n{=}96$), and explicitly read cross-condition gaps below this resolution as ties. Per-seed standard deviations for the main accuracy table appear in \cref{tab:main_results_sd}, and Spearman and Kendall correlations between \ourmetric{} and final correctness are reported with $p$-values in \cref{subsec:reward-validation}.
    \item[] Guidelines:
    \begin{itemize}
        \item The answer \answerNA{} means that the paper does not include experiments.
        \item The authors should answer \answerYes{} if the results are accompanied by error bars, confidence intervals, or statistical significance tests, at least for the experiments that support the main claims of the paper.
        \item The factors of variability that the error bars are capturing should be clearly stated (for example, train/test split, initialization, random drawing of some parameter, or overall run with given experimental conditions).
        \item The method for calculating the error bars should be explained (closed form formula, call to a library function, bootstrap, etc.)
        \item The assumptions made should be given (e.g., Normally distributed errors).
        \item It should be clear whether the error bar is the standard deviation or the standard error of the mean.
        \item It is OK to report 1-sigma error bars, but one should state it. The authors should preferably report a 2-sigma error bar than state that they have a 96\% CI, if the hypothesis of Normality of errors is not verified.
        \item For asymmetric distributions, the authors should be careful not to show in tables or figures symmetric error bars that would yield results that are out of range (e.g., negative error rates).
        \item If error bars are reported in tables or plots, the authors should explain in the text how they were calculated and reference the corresponding figures or tables in the text.
    \end{itemize}

\item {\bf Experiments compute resources}
    \item[] Question: For each experiment, does the paper provide sufficient information on the computer resources (type of compute workers, memory, time of execution) needed to reproduce the experiments?
    \item[] Answer: \answerYes{}
    \item[] Justification: Compute is documented in \cref{sec:implementation-details,sub:grpo-setup}. Open-weight inference (QWEN3-8B, Ministral-3-8B, GPT-oss-20B) runs on a single NVIDIA GeForce RTX 4090 through Ollama. OpenAI-hosted models (GPT-3.5, GPT-4.1-nano, GPT-5-nano) use the public API. The GRPO post-training experiment runs on one node of $4{\times}$ A100 GPUs with FSDP and SGLang inference. The per-step LCS bound $O(|Q|\,|\mathrm{Enc}(T_t)|)$ is derived in \cref{sub:efficiency-protocol}, and end-to-end token budgets and latencies are reported in Table~\ref{tab:cost-accounting}.
    \item[] Guidelines:
    \begin{itemize}
        \item The answer \answerNA{} means that the paper does not include experiments.
        \item The paper should indicate the type of compute workers CPU or GPU, internal cluster, or cloud provider, including relevant memory and storage.
        \item The paper should provide the amount of compute required for each of the individual experimental runs as well as estimate the total compute. 
        \item The paper should disclose whether the full research project required more compute than the experiments reported in the paper (e.g., preliminary or failed experiments that didn't make it into the paper). 
    \end{itemize}
    
\item {\bf Code of ethics}
    \item[] Question: Does the research conducted in the paper conform, in every respect, with the NeurIPS Code of Ethics \url{https://neurips.cc/public/EthicsGuidelines}?
    \item[] Answer: \answerYes{}
    \item[] Justification: The work uses only public benchmarks containing no personally identifying information, involves no human subjects or crowdsourced data collection, does not release new generative models with high misuse risk, and conforms in all respects with the NeurIPS Code of Ethics.
    \item[] Guidelines:
    \begin{itemize}
        \item The answer \answerNA{} means that the authors have not reviewed the NeurIPS Code of Ethics.
        \item If the authors answer \answerNo, they should explain the special circumstances that require a deviation from the Code of Ethics.
        \item The authors should make sure to preserve anonymity (e.g., if there is a special consideration due to laws or regulations in their jurisdiction).
    \end{itemize}

\item {\bf Broader impacts}
    \item[] Question: Does the paper discuss both potential positive societal impacts and negative societal impacts of the work performed?
    \item[] Answer: \answerYes{}
    \item[] Justification: The Impact Statement discusses both positive impacts (more reliable and energy-efficient table reasoning, greater transparency of intermediate steps) and negative impacts (over-reliance in high-stakes automated decision-making, fairness concerns, and exploitable lexical failure modes of \ourmetric{} such as paraphrase under-rewarding and echo-column hacking), together with recommended mitigations including answer-checking, human review, and pairing \ourmetric{} with executor-grounded verification.
    \item[] Guidelines:
    \begin{itemize}
        \item The answer \answerNA{} means that there is no societal impact of the work performed.
        \item If the authors answer \answerNA{} or \answerNo, they should explain why their work has no societal impact or why the paper does not address societal impact.
        \item Examples of negative societal impacts include potential malicious or unintended uses (e.g., disinformation, generating fake profiles, surveillance), fairness considerations (e.g., deployment of technologies that could make decisions that unfairly impact specific groups), privacy considerations, and security considerations.
        \item The conference expects that many papers will be foundational research and not tied to particular applications, let alone deployments. However, if there is a direct path to any negative applications, the authors should point it out. For example, it is legitimate to point out that an improvement in the quality of generative models could be used to generate Deepfakes for disinformation. On the other hand, it is not needed to point out that a generic algorithm for optimizing neural networks could enable people to train models that generate Deepfakes faster.
        \item The authors should consider possible harms that could arise when the technology is being used as intended and functioning correctly, harms that could arise when the technology is being used as intended but gives incorrect results, and harms following from (intentional or unintentional) misuse of the technology.
        \item If there are negative societal impacts, the authors could also discuss possible mitigation strategies (e.g., gated release of models, providing defenses in addition to attacks, mechanisms for monitoring misuse, mechanisms to monitor how a system learns from feedback over time, improving the efficiency and accessibility of ML).
    \end{itemize}
    
\item {\bf Safeguards}
    \item[] Question: Does the paper describe safeguards that have been put in place for responsible release of data or models that have a high risk for misuse (e.g., pre-trained language models, image generators, or scraped datasets)?
    \item[] Answer: \answerNA{}
    \item[] Justification: The paper does not release pre-trained generative models, scraped datasets, or other artifacts with high misuse risk. \ourmodel{} is a training-free wrapper around existing publicly available LLMs, \ourmetric{} is a transparent lexical reward whose failure modes are documented in \cref{subsec:failure-modes-main}, and the GRPO experiment is a small proof-of-concept fine-tune of an already public 8B model.
    \item[] Guidelines:
    \begin{itemize}
        \item The answer \answerNA{} means that the paper poses no such risks.
        \item Released models that have a high risk for misuse or dual-use should be released with necessary safeguards to allow for controlled use of the model, for example by requiring that users adhere to usage guidelines or restrictions to access the model or implementing safety filters. 
        \item Datasets that have been scraped from the Internet could pose safety risks. The authors should describe how they avoided releasing unsafe images.
        \item We recognize that providing effective safeguards is challenging, and many papers do not require this, but we encourage authors to take this into account and make a best faith effort.
    \end{itemize}

\item {\bf Licenses for existing assets}
    \item[] Question: Are the creators or original owners of assets (e.g., code, data, models), used in the paper, properly credited and are the license and terms of use explicitly mentioned and properly respected?
    \item[] Answer: \answerYes{}
    \item[] Justification: All datasets (WTQ~\cite{pasupat-liang-2015-compositional}, TableBench~\cite{10.1609/aaai.v39i24.34739}, TabFact~\cite{Chen2020TabFact}, MMQA~\cite{wu2025mmqa}, MMTU~\cite{xing2025mmtu}), base models (QWEN3-8B~\cite{yang2025qwen3technicalreport}, Ministral-3-8B~\cite{mistralai_ministral3_8b_25-12}, GPT-oss-20B~\cite{openai_gptoss20b_2025}, GPT-4.1-nano~\cite{openai_gpt41_2025}, GPT-5-nano~\cite{openai_gpt5nano_2025}, DeepSeek-R1-Distill-Qwen-14B~\cite{deepseek2025r1}), and baseline pipelines (Chain-of-Table~\cite{wang2024chainoftable}, Tree-of-Table~\cite{ji2025treeoftable}, TaTToo~\cite{zou2026tattoo}) are cited and used under their original public licenses and stated terms of use.
    \item[] Guidelines:
    \begin{itemize}
        \item The answer \answerNA{} means that the paper does not use existing assets.
        \item The authors should cite the original paper that produced the code package or dataset.
        \item The authors should state which version of the asset is used and, if possible, include a URL.
        \item The name of the license (e.g., CC-BY 4.0) should be included for each asset.
        \item For scraped data from a particular source (e.g., website), the copyright and terms of service of that source should be provided.
        \item If assets are released, the license, copyright information, and terms of use in the package should be provided. For popular datasets, \url{paperswithcode.com/datasets} has curated licenses for some datasets. Their licensing guide can help determine the license of a dataset.
        \item For existing datasets that are re-packaged, both the original license and the license of the derived asset (if it has changed) should be provided.
        \item If this information is not available online, the authors are encouraged to reach out to the asset's creators.
    \end{itemize}

\item {\bf New assets}
    \item[] Question: Are new assets introduced in the paper well documented and is the documentation provided alongside the assets?
    \item[] Answer: \answerYes{}
    \item[] Justification: The paper introduces \ourmetric{} (a deterministic state reward) and \ourmodel{} (a training-free framework). Both are documented end-to-end via \cref{eq:tabrouge}, \cref{sec:textual-encoding,sec:tabrouge-design,sec:detail-operation,sec:implementation-details}, Algorithms~\ref{alg:chain_of_tables} and~\ref{alg:tree_of_tables}, and Table~\ref{tab:all_hyperparameters}. The accompanying anonymized supplementary material includes the implementation, prompts, and reproducibility scripts, and the public release will include a README and an open-source license.
    \item[] Guidelines:
    \begin{itemize}
        \item The answer \answerNA{} means that the paper does not release new assets.
        \item Researchers should communicate the details of the dataset\slash code\slash model as part of their submissions via structured templates. This includes details about training, license, limitations, etc. 
        \item The paper should discuss whether and how consent was obtained from people whose asset is used.
        \item At submission time, remember to anonymize your assets (if applicable). You can either create an anonymized URL or include an anonymized zip file.
    \end{itemize}

\item {\bf Crowdsourcing and research with human subjects}
    \item[] Question: For crowdsourcing experiments and research with human subjects, does the paper include the full text of instructions given to participants and screenshots, if applicable, as well as details about compensation (if any)?
    \item[] Answer: \answerNA{}
    \item[] Justification: The paper does not involve crowdsourcing or research with human subjects. All experiments are conducted on public table-reasoning benchmarks.
    \item[] Guidelines:
    \begin{itemize}
        \item The answer \answerNA{} means that the paper does not involve crowdsourcing nor research with human subjects.
        \item Including this information in the supplemental material is fine, but if the main contribution of the paper involves human subjects, then as much detail as possible should be included in the main paper. 
        \item According to the NeurIPS Code of Ethics, workers involved in data collection, curation, or other labor should be paid at least the minimum wage in the country of the data collector. 
    \end{itemize}

\item {\bf Institutional review board (IRB) approvals or equivalent for research with human subjects}
    \item[] Question: Does the paper describe potential risks incurred by study participants, whether such risks were disclosed to the subjects, and whether Institutional Review Board (IRB) approvals (or an equivalent approval/review based on the requirements of your country or institution) were obtained?
    \item[] Answer: \answerNA{}
    \item[] Justification: The paper does not involve research with human subjects, so IRB review or equivalent approval was not required.
    \item[] Guidelines:
    \begin{itemize}
        \item The answer \answerNA{} means that the paper does not involve crowdsourcing nor research with human subjects.
        \item Depending on the country in which research is conducted, IRB approval (or equivalent) may be required for any human subjects research. If you obtained IRB approval, you should clearly state this in the paper. 
        \item We recognize that the procedures for this may vary significantly between institutions and locations, and we expect authors to adhere to the NeurIPS Code of Ethics and the guidelines for their institution. 
        \item For initial submissions, do not include any information that would break anonymity (if applicable), such as the institution conducting the review.
    \end{itemize}

\item {\bf Declaration of LLM usage}
    \item[] Question: Does the paper describe the usage of LLMs if it is an important, original, or non-standard component of the core methods in this research? Note that if the LLM is used only for writing, editing, or formatting purposes and does \emph{not} impact the core methodology, scientific rigor, or originality of the research, declaration is not required.
    \item[] Answer: \answerYes{}
    \item[] Justification: LLMs are core to the studied system. \cref{sec:experiments,sec:implementation-details,sub:grpo-setup} list every backbone (GPT-3.5, GPT-4.1-nano, GPT-5-nano, QWEN3-8B, Ministral-3-8B, GPT-oss-20B, DeepSeek-R1-Distill-Qwen-14B) and the role each plays in the agent loop, including the auxiliary GPT-4.1-nano rewriting step used to estimate action-level confidence for GPT-5-nano. LLMs were not used as ghostwriters for the methodology or scientific contributions of this paper.
    \item[] Guidelines:
    \begin{itemize}
        \item The answer \answerNA{} means that the core method development in this research does not involve LLMs as any important, original, or non-standard components.
        \item Please refer to our LLM policy in the NeurIPS handbook for what should or should not be described.
    \end{itemize}

\end{enumerate}

\end{document}